\pdfoutput=1
\documentclass{article}

\usepackage{microtype}
\usepackage{graphicx}
\usepackage{amsthm}
\usepackage{booktabs}
\usepackage{physics}
\PassOptionsToPackage{numbers, compress}{natbib}
\usepackage[main,final]{neurips_2026}

\usepackage[inline]{enumitem}
\usepackage{dsfont}
\usepackage{braket}
\usepackage{bbm}
\usepackage{xfrac}
\usepackage{multirow}
\usepackage{tikz}
\usetikzlibrary{arrows.meta}
\usepackage{wrapfig}
\usepackage[unicode=true,
 bookmarks=false,
 breaklinks=false,pdfborder={0 0 1},colorlinks=false]
 {hyperref}
\usepackage{thm-restate}
\usepackage{empheq}
\usepackage{color}
\definecolor{cool-blue}{HTML}{1D6996}
\definecolor{cool-purple}{HTML}{5F4690}
\definecolor{cool-teal}{HTML}{38A6A5}
\definecolor{cool-green}{HTML}{0F8554}
\definecolor{dark-green}{rgb}{0.0,0.5,0.0}

\newcommand{\E}[1]{\mathbb{E}\left[#1\right]}

\newcommand{\diff}{\mathrm{d}}

\DeclareMathOperator*{\argmax}{arg\,max}
\DeclareMathOperator*{\argmin}{arg\,min}
\DeclareMathOperator*{\KL}{KL}

\usepackage[utf8]{inputenc}
\usepackage[T1]{fontenc}
\usepackage{url}
\usepackage{amsfonts}
\usepackage{nicefrac}
\usepackage{xcolor}

\usepackage{algorithm}
\usepackage{algorithmic}
\usepackage{amsmath}
\usepackage{amssymb}
\usepackage{mathtools}
\usepackage{mdframed}
\mdfdefinestyle{MyBoxStyle}{
    linecolor=green!60!black,
    linewidth=0pt,
    backgroundcolor=blue!10,
    roundcorner=5pt,innerleftmargin=4pt,innerrightmargin=4pt
}

\usepackage[capitalize,noabbrev]{cleveref}
\usepackage{titlesec}
\titlespacing*{\paragraph}{0pt}{0.2ex plus .1ex minus .1ex}{1em}

\theoremstyle{plain}
\newmdtheoremenv[style=MyBoxStyle]{theorem}{Theorem}[section]
\newtheorem{proposition}[theorem]{Proposition}
\newtheorem{lemma}[theorem]{Lemma}

\theoremstyle{definition}
\newmdtheoremenv[style=MyBoxStyle]{definition}[theorem]{Definition}

\newtheorem{example}[theorem]{Example}
\theoremstyle{remark}

\crefname{assumption}{Assumption}{Assumptions}
\Crefname{assumption}{Assumption}{Assumptions}

\usepackage{subcaption}

\title{Training Generalizable Collaborative Agents via Strategic Risk Aversion}

\author{%
  Chengrui Qu\\
  California Institute of Technology \\
  \texttt{cqu@caltech.edu} \\
  \And
  Yizhou Zhang \\
  California Institute of Technology\\
  \texttt{yzhang8@caltech.edu} \\
  \AND
  Nicolas Lanzetti \\
  California Institute of Technology \\
  \texttt{lnicolas@caltech.edu} \\
  \And
  Eric Mazumdar \\
  California Institute of Technology \\
  \texttt{mazumdar@caltech.edu}
}

\begin{document}

\maketitle

\begin{abstract}
  Many emerging agentic paradigms require agents to collaborate with one another (or people) to achieve shared goals. However, existing approaches to learning policies for such collaborative problems produce brittle solutions that fail when paired with new partners.
  We attribute these failures to a combination of \emph{free-riding} during training and a lack of \emph{strategic robustness}. To address these problems, we study the concept of strategic risk aversion and interpret it as a principled inductive bias for generalizable cooperation with unseen partners. While strategically risk-averse players are robust to deviations in their partner's behavior by design, we show that, in collaborative games,  they also (1) can have better equilibrium outcomes than those at classical game-theoretic concepts like Nash, and (2) exhibit less or no free-riding. Inspired by these insights, we develop a multi-agent reinforcement learning algorithm that integrates strategic risk aversion into standard policy optimization methods. Our empirical results across collaborative benchmarks (including an LLM collaboration task) validate our theory and demonstrate that our approach consistently achieves reliable collaboration with heterogeneous and previously unseen partners across collaborative tasks.
\end{abstract}

\section{Introduction}

AI systems increasingly operate in multi-agent environments, where success depends on effective interaction with other agents.  An increasingly important subset of these problems involves agents solving \emph{collaborative tasks}---i.e., tasks where agents must work together towards a shared goal. Examples of such problems range from robots coordinating with people in shared physical spaces like warehouses \citep{kruse2013}, to emerging ``agentic AI'' problems where multiple models collaborate  to write code~\citep{wu2024autogen} or solve math problems~\citep{park-etal-2025-maporl}.

A common approach to such problems is to view them as collaborative \emph{games} and use concepts from game theory to design algorithms to learn collaborative policies~\citep{dafoeCooperativeAIMachines2021}. Viewed through this lens, a central challenge in collaborative games is \emph{partner generalization}---where one seeks to maintain effective interaction across new sets of partners~\citep{carroll2019,hu20a,strouse2022collaboratinghumanshumandata,ruhdorfer2025overcooke}. Indeed, in many real-world settings, agents must navigate situations in which they interact with different partners (either algorithmic or human) who may have slightly different goals, heuristics, or levels of competence. Agents should learn strategies that work well across a broad range of partners without sacrificing performance.
Unfortunately, current approaches tend to struggle with this---with learned policies often failing to generalize to new partners~\citep{stone2010}, or \emph{overfitting} to other agents' eccentricities and becoming overly reliant on specific conventions \citep{carroll2019,lerer2019}.

We view the problem of partner generalization as ultimately a question of \emph{robustness} of the learned policy and of \emph{alignment} between agents~\citep{leibo2017}. The need for robustness emerges from the requirement that learned strategies be insensitive to changes in a partner’s strategy and degrade gracefully in performance as the partners become less collaborative.
The need for alignment emerges from the desire to find agents that contribute proportionally to the task.
If agents learn to under-contribute or delegate costly actions \citep{vdn,liu23ac} to their partners---effectively \emph{free-riding} on their partner’s effort---they cannot generalize to new environments and partners.
While free-riding is a widely acknowledged phenomenon in game theory, it is rarely studied as a problem in collaborative multi-agent learning despite recent empirical evidence that it arises even with large AI models~\citep{zhang2025unlockingpowermultiagentllm,liang2025contributesincentivizingstrategiccooperation}.

Prior work---which we review in~\cref{sec:related} and then in depth in Appendix~\ref{app:related}---has studied partner generalization in specific settings that fail to scale to large-scale problems like the post-training of LLM agents~\citep{hu20a} or via largely heuristic approaches~\citep{PPOentropy}.  In contrast, in this paper we advance \emph{strategic} risk aversion as a \emph{principled} and \emph{scalable} approach to addressing this issue in a broad class of collaborative games.
Recently proposed in multi-agent reinforcement learning (MARL) to derive human-like strategies in general games~\citep{mazumdar2024tractableequilibriumcomputationmarkov}, strategic risk aversion mirrors behaviors observed in people in experimental economics~\citep{goeree2003risk}. Unlike many approaches to risk aversion or robustness in MARL which primarily focus on robustness to uncertainty in the underlying environment, \emph{strategic} risk aversion (robustness) posits that agents should be risk-averse (robust) to uncertainty stemming from their opponents' decisions~\citep{mazumdar2024tractableequilibriumcomputationmarkov,lanzetti2025strategically}.
This perspective aligns with the goal of partner generalization, as it forces agents to be robust to deviations in their partner's play.

Formalizing this principle leads to a new equilibrium concept: (Strategically) Risk-Averse Quantal Response Equilibria (RQE)~\citep{mazumdar2024tractableequilibriumcomputationmarkov}. While RQE has been shown to better capture human play in simple experiments and offer better computational tractability~\citep{zhang2025convergent} than classic game theoretic concepts, its potential for collaborative games and partner generalization remains unknown. We argue that strategic risk aversion provides a suitable inductive bias for robust cooperation in collaborative games: at an RQE, an agent is trained against a structured set of \emph{plausible} partner deviations, governed by the agent's degree of risk aversion. This ensures strategies are robust without being overly conservative. Remarkably, although strategic risk aversion is introduced for \textit{robustness}, we prove that it can also lead to better \textit{alignment} between agents.

 \textbf{Contributions:} Through both a theoretical analysis of RQE in structured collaborative games and extensive experiments across a range of diverse multi-agent benchmarks, we validate that RQE are particularly suited for collaborative tasks.  In particular, our contributions are as follows:
\begin{itemize}[leftmargin=10pt, noitemsep]
    \item \textbf{Incentivizing collaboration (\cref{thm:rqe_cooperation}):} We prove that in continuous quadratic aggregative games,  strategic risk aversion can encourage collaboration, with higher risk aversion leading to a greater focus on the shared goal. This implies a counterintuitive result: unlike in classic robust optimization or RL, \emph{robustness does not necessarily require sacrificing performance}.
    \item \textbf{Alleviating free-riding (\cref{thm:free_riding}):} We prove in finite-action collaborative games with private costs that strategic risk aversion can mitigate free-riding at equilibrium. \emph{Thus, risk aversion can reduce the alignment problem in collaborative MARL.}
    \item \textbf{Scalable risk-averse MARL algorithms:} We develop \textit{Strategically Risk-Averse Policy Optimization} (SRPO), a MARL algorithm that optimizes an RQE-derived objective which integrates naturally with policy-optimization algorithms like independent proximal policy optimization~(IPPO).
    \item \textbf{Empirical validation:} Across collaborative MARL benchmarks, SRPO \emph{consistently} achieves more reliable coordination with heterogeneous and unseen partners than IPPO (the current scalable MARL baseline). We observe that IPPO consistently learns free-riding equilibria that limit its generalization. In contrast and in line with our theory, SRPO converges to RQE-like behavior, eliminates free-riding, and improves partner generalization. We also provide preliminary small-scale experiments showing that these findings can even extend to the fine-tuning of collaborative agentic AI teams made up of large language models, demonstrating the scalability of our approach.
\end{itemize}

\section{Related Work}\label{sec:related}
Our work studies strategic risk aversion in collaborative MARL.
We cover the most relevant related work here and defer a more in depth discussion to Appendix~\ref{app:related}.

The most related literature in collaborative MARL studies partner generalization: learning collaborative policies that remain effective when paired with previously unseen partners \citep{Barrett2014,carroll2019,dizdarevic25a}. This problem is also studied as zero-shot coordination \citep{hu20a,hu21c}, or ad-hoc teamwork problem \citep{stone2010}. A major class of approaches is population-based training \citep{vinyalsGrandmasterLevelStarCraft2019,zhao2022maximumentropypopulationbasedtraining,yu2023learning,wang2024zsceval,rahman2023generating,andrei2021,strouse2022collaboratinghumanshumandata,MBRL,cui2023adversarial}, where agents train against a diverse set of learned partners, essentially performing domain randomization for partners, with methods differing mainly in how they construct the population and how they train against it.  Although simple, these methods require nontrivial partner-population design, often rely on task-specific heuristics, and can be costly in high-dimensional settings such as LLM fine-tuning. They also offer limited guarantees of improved generalization. Empirically, we find that fictitious co-play~\citep{strouse2022collaboratinghumanshumandata}, a representative population-based method, requires substantially more compute yet still performs poorly in our setting. To avoid these costs,~\citet{PPOentropy} increase entropy regularization to induce robustness. While scalable, this approach does not eliminate free-riding and therefore does not fully address collaborative generalization.

In contrast, our approach based on strategic risk aversion scales as a simple modification to existing policy-optimization techniques, is grounded in a principled game-theoretic formulation, and empirically consistently delivers more stable cross-play performance across collaborative benchmarks. Rather than approximating robustness through brute-force partner diversification, our approach directly optimizes a risk-averse robustness objective. Although risk aversion has long been studied in experimental economics~\citep{gollier2001economics,goeree2002efficiency,goeree2003risk}, strategic risk aversion has only recently been studied in game theory~\citep{lanzetti2025strategically,mazumdar2024tractableequilibriumcomputationmarkov}, control~\citep{velasevic2026strategically,feik2026strategically,qin2026strategically}, and MARL~\citep{zhang2025convergent,zhang2026provablyconvergentactorcriticriskaverse,slumbers2023game}. Existing work mainly focuses on tractability and Q-learning-based computation of risk-averse equilibria~\citep{zhang2025convergent,zhang2026provablyconvergentactorcriticriskaverse}, whose scalability remains unclear. In contrast, we study the effects of risk qualitatively in \emph{collaborative} problems. We prove two free-lunch theorems (\cref{thm:rqe_cooperation,thm:free_riding}) characterizing the effect of risk on welfare and payoffs and propose a policy optimization-based approach that can be readily incorporated into existing code-bases.

\section{Problem Setup}\label{sec:problem-setup}
We consider a general-sum collaborative game with $N$ players. Each player $i\in [N]$ is endowed with actions $a_i\in\mathcal{A}_i$, where the action spaces can be discrete or continuous, and we use the notation $\mathcal{A}_{-i}=\prod_{j\neq i}\mathcal{A}_j$ for the action space of all players except $i$.
Each player aims to maximize a utility function $u_i:\mathcal{A}_i\times \mathcal{A}_{-i}\rightarrow \mathbb{R}$.
The utility function consists of a shared reward $R$, which coincides for all players, and a private cost $c_i$, so that
\begin{equation*}
    u_i(a_i,a_{-i})
    =
    R(a_1,\ldots, a_N) - c_i(a_i).
\end{equation*}
These games are collaborative since all players seek to maximize the reward, but pay for their own effort. They have been well studied in economics under the guise of public goods games and naturally model many real-world scenarios where players aim to collaboratively accomplish a task but are penalized for their own effort.
The players select a mixed strategy $x_i\in\Delta(\mathcal{A}_i)$, which is a probability over the available actions $\mathcal{A}_i$ (here, $\Delta(\mathcal{A}_i)$ is the set of probability distributions over $\mathcal{A}_i$).
In the standard game-theoretic setting, the utility of each player is the standard risk-neutral expected utility
\begin{equation*}
    U_i(x_i,x_{-i})
    =
    \mathbb{E}_{\substack{a_1\sim x_1, \cdots,a_N\sim x_N}}[u_i(a_i,a_{-i})],
\end{equation*}
where $x_{-i}$ is the collection of the mixed strategies of all other players $j\neq i$.

In this work, we follow~\citet{mazumdar2024tractableequilibriumcomputationmarkov} and introduce risk aversion and entropy regularization.

\paragraph{Risk aversion}
To model risk aversion, we consider players that seek to optimize a risk-adjusted utility based on the  \textit{entropic risk measure} with risk aversion parameter $\tau_i>0$ (e.g., see~\citet{follmer2002convex}), instead of their expected utility. Players' resulting objectives are given by
\begin{equation}\label{eq:risk_adjusted_objective}
    U^{\tau_i}_i(x_i,x_{-i})
    =
    \inf_{p\in\Delta(\mathcal{A}_{-i})}
    U_i(x_i,p) + \frac{1}{\tau_i}\KL(p,x_{-i}).
\end{equation}
In words, the utility results from a worst-case approach, in which a fictitious adversary tries to inflict maximum damage at the (risk-neutral) expected utility, while not deviating ``too much'' from the strategies of all other players.
For the sake of this paper, we will measure this deviation with the KL divergence, defined as $\KL(p,q)=\sum_j p_j\log(\frac{p_j}{q_j})$ (when $\mathcal{A}_{-i}$ is discrete) and $\KL(p,q)=\int_{\mathbb{R}^n}\log(\frac{\dd p}{\dd q}(x))\,\dd p(x)$ (when $\mathcal{A}_{-i}=\mathbb{R}^n$ is continuous), where $\frac{\dd p}{\dd q}$ is the density of $p$ with respect to $q$ (provided it exists, else $\KL(p,q)=+\infty$).
Our choice of the entropic risk measure is rooted in operations research, but also motivated by ubiquitous use of the KL divergence across machine learning and in particular in policy optimization in reinforcement learning. Using the KL divergence, we derive algorithms that can be implemented with minimal changes to existing codebases.

\paragraph{Entropy regularization}
We add entropy to each player's utility, so that their utility is\begin{align}\label{eq:risk_entropy_adjusted_objective}
    U_i^{\tau_i,\epsilon_i}(x_i,x_{-i})
    &=
    U^{\tau_i}_i(x_i,x_{-i})
    -
    \epsilon_i H(x_i),
\end{align}
where $\epsilon_i>0$ and $H(x_i)$ is the (negative) entropy of the mixed strategy $x_i$; i.e., $H(p)=\sum_j p_j\log(p_j)$ (if $\mathcal{A}_i$ is discrete) and $H(p)=\int_{\mathbb{R}^n} \rho_p(x)\log(\rho_p(x))\dd x$ (when $\mathcal{A}_i=\mathbb{R}^n$ is continuous), where $\rho_p$ is again the density of $p$ (provided it exists, else $H(p)=+\infty$).
Entropy regularization of the utility leads to the celebrated quantal response model in behavioral economics~\citep{mckelvey1995quantal} and is widely employed in machine learning for exploration purposes (e.g., see \citet{ahmed2019understanding}).
For more general regularization schemes, we refer to~\citet{mazumdar2024tractableequilibriumcomputationmarkov}.

\paragraph{Risk-averse quantal response equilibrium (RQE).} Given our utility $U_i^{\tau_i,\epsilon_i}(x_i,x_{-i})$, we define an RQE as a set of mixed strategies from which no player has an incentive to unilaterally deviate:

\begin{definition}[\citet{mazumdar2024tractableequilibriumcomputationmarkov}, Definition 5]\label{def:RQE_normal_form}
    A tuple of mixed strategies
    $(x_1^\star, \ldots, x_N^\star)$
    is a risk-averse quantal response equilibrium (RQE) with degrees of risk aversion $\tau_1,\ldots,\tau_N$ and entropy regularization $\epsilon_1,\ldots,\epsilon_N$ if for all players $i$ we have
    \begin{equation}
        U^{\tau_i,\epsilon_i}_i(x_i,x_{-i}^\star)
        \leq
        U^{\tau_i,\epsilon_i}_i(x_i^\star,x_{-i}^\star)
        \: \forall\,x_i\in\Delta(\mathcal{A}_i).
    \end{equation}
\end{definition}

When $\tau_i,\epsilon_i\to 0$ we recover Nash equilibria, when $\tau_i\to 0$ we recover quantal response equilibria, and when $\epsilon_i \rightarrow 0$ and $\tau_i \rightarrow \infty$ we recover security strategies.
In the next section, we investigate the benefits of risk aversion, together with entropy regularization, in two classes of collaborative games: aggregative continuous quadratic games and finite symmetric collaborative games.
These insights prompt us, in the subsequent section, to use RQE to design algorithms to train collaborative agents.

\section{``Free-Lunch'' Theorems for Strategic Risk Aversion}\label{sec:theory}
We first study strategic risk aversion in simple, \emph{structured} classes of collaborative games which are more tractable to analyze than complex MARL problems we ultimately design algorithms for.
These simple games act as a theoretical laboratory, allowing us to rigorously prove two ``free-lunch'' theorems that illustrate how strategic risk aversion, besides ensuring robustness, can (1) increase collaboration and (2) mitigate free-riding---both of which are essential to the task of partner generalization. These results serve as a principled rationale for incorporating strategic risk aversion into MARL. We empirically observe both these properties across our experiments in \cref{sec:experiments}, validating that our insights extend beyond structured games.

\subsection{Strategic Risk Aversion Can Induce Collaboration}\label{subsec:risk_collaboration}

Our first setting lies in continuous ($\mathcal{A}_i=\mathbb{R}^n$) quadratic aggregative games, popular in economics~\citep{corchon1994comparative} and control~\citep{paccagnan2018nash}.
In line with our collaborative setting, we write the utility of each player as the difference between a shared reward $R$ that depends on the aggregate action $a_1+\ldots+a_N$, and a private cost $c_i$ that depends on the player's own action $a_i$:%
\begin{equation}\label{theory:continuous_game:payoff}
    R(a_1,\ldots,a_N)
    =
    \frac{1}{2}
    \left\langle\sum_{i=1}^N a_i, H\sum_{i=1}^N a_i\right\rangle + \left\langle h,\sum_{i=1}^N a_i\right\rangle,\  c_i(a_i)=\frac{\rho_i}{2}\norm{a_i}^2,\ \forall i\in[N]
\end{equation}
for some $H$ negative definite, $h\neq 0$ of appropriate dimensions, and  $\rho_i>0$.
For simplicity of exposition, we assume that all players have the same degree of risk aversion $\tau_i$, bounded rationality $\epsilon_i$, and $\rho_i$. As we show in \cref{sec:proof:thm:rqe_cooperation}, all results extend to the case where these are player-dependent.

Being a continuous game, the computation of an RQE is infinite dimensional, as it involves searching over mixed strategies over continuous action spaces.
In~\cref{proposition:continuous_game:rqe} in the appendix, however, we bypass this complexity by showing that there is a unique RQE (at which players have finite risk-adjusted utility) and that the equilibrium mixed strategies are Gaussian distributions, whose mean and variance can be computed efficiently.
This effectively allows us to study the effect of risk aversion on the expected shared reward $J(\tau)\coloneqq \mathbb{E}_{a_i\sim x_i^\star(\tau)} \left[R(a_1,\ldots,a_N)\right].$

Remarkably, as we show in the next theorem (proven in \cref{sec:proof:thm:rqe_cooperation}), this function is strictly increasing. That is, risk aversion \textit{monotonically} increases the shared reward and thus induces collaboration.

\begin{theorem}[risk induces collaboration]\label{thm:rqe_cooperation}
    Let $x_i^\star(\tau)$ be the mixed strategy of player $i$ at the unique RQE of the game, as a function of the degree of risk aversion $\tau$.
    Then, the expected shared reward
    $\tau\mapsto J(\tau)$ is strictly increasing.\footnotemark
    That is, players contribute more to the shared reward as they become more risk-averse.
\end{theorem}
\footnotetext{Here, $\tau\in(0,\bar\tau)$ with $\bar\tau=\frac{\rho+\lambda_{\max}(-H)}{(N-1)\epsilon\lambda_{\max}(-H)}$, and uniqueness is among RQE at which all players have finite
risk-adjusted utility. Degenerate RQE can exist: if all players play Gaussian strategies with sufficiently large covariances, every player's risk-adjusted utility is $-\infty$ regardless of their own strategy; see~\cref{sec:proof:thm:rqe_cooperation}.}
This increasing contribution to the shared reward generally entails an increase in the players' private cost. Thus, in practice, we observe a \textit{tradeoff} between being risk neutral ($\tau\to 0$) and more risk-averse ($\tau$ large) with large personal cost but large shared reward.
In a one-dimensional example, we can study this tradeoff analytically.
\begin{figure}[t!]
\vspace{-2ex}
    \centering
    \includegraphics[width=\textwidth]{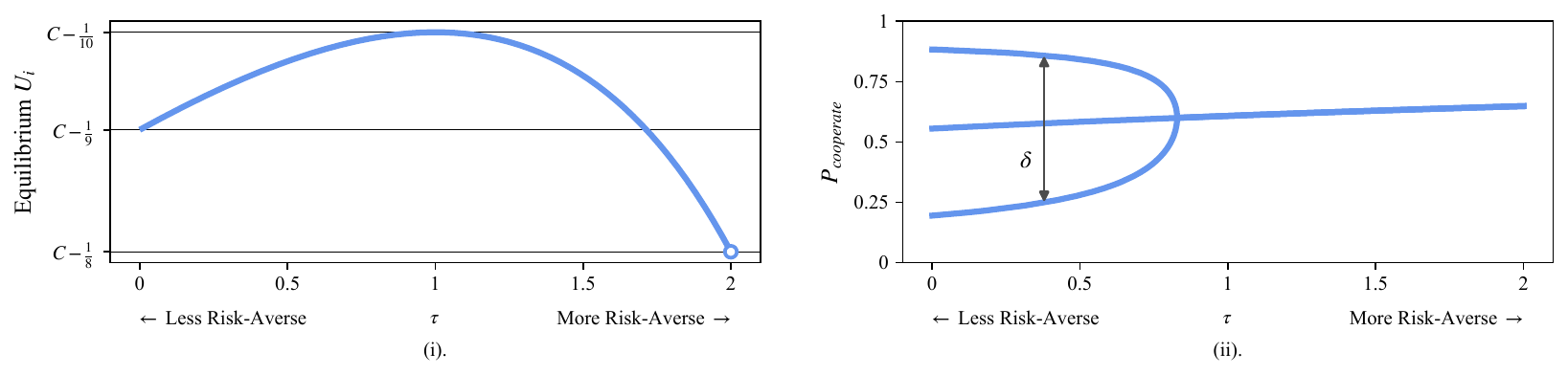}
    \caption{(i). Expected utility of each player as a function of the degree of risk aversion $\tau$ at equilibrium, when $\epsilon=1$ and $\bar{a}=1$ in~\cref{example:continuous_game}. $C=-3\epsilon/4$ is a $\tau$-independent constant. Player's utilities can first \emph{increase} with risk aversion before decreasing due to over-conservatism, meaning that strategic risk aversion can yield better performing RQE. (ii). Probability that a player collaborates at an RQE as a function of the level of risk aversion, for $\epsilon=0.2$ for the game in \cref{example:free_riding}. The arrow shows the gap in collaboration probability between the two players at a free-riding equilibrium, which is proportional to the degree of free-riding. Strategic risk aversion alleviates free-riding entirely after a given threshold (i.e., $\delta \rightarrow 0$) as our theory predicts.}
    \label{fig:theory}
    \vspace{-2ex}
\end{figure}
\begin{example}\label{example:continuous_game}
    Consider two robots that move an object from the origin to a target location $\bar a$. Each robot exerts a force $a_i$, so that the total force is $a_1+a_2$. Players incur a penalty for their own actuation effort. Thus, the shared reward is $-\frac{1}{2}(a_1+a_2 - \bar a)^2$ and the personal cost is  $\frac{1}{2} a_i^2$.
    In this case, for $\tau<2/\epsilon$, the RQE consists of the Gaussian mixed strategies with mean $m_i^\star(\tau)=\frac{\bar a}{3-\tau\epsilon/2}$ and variance $\Sigma_i=\frac{\epsilon}{2}$.
    The equilibrium expected utility is concave in $\tau$ as shown in~\cref{fig:theory}, highlighting that risk aversion can yield higher utilities.
\end{example}

We refer to \cref{thm:rqe_cooperation} as a ``free-lunch'' theorem because it suggests that one can add risk aversion into a game and \emph{not} incur a loss in performance. This can be observed in \cref{fig:theory}(i) in which equilibrium utility increases from a risk-neutral baseline before decreasing again, and is further validated in the practical settings with aggregative structure in \cref{sec:exp-overcooked}. This is in stark contrast to single-agent RL or robust optimization where robustness sacrifices performance.
We validate our theory in broader classes of games in~\cref{sec:experiments} and observe that, in some games, some degree of strategic risk aversion can increase collaboration and performance over a risk-neutral baseline.

\subsection{Strategic Risk Aversion Alleviates Free-riding}\label{subsec:risk_free_riding}

We now prove a second ``free-lunch'' theorem for strategic risk aversion in collaborative games: strategically risk-averse players free-ride less at equilibrium. We establish this result in the context of 2-player collaborative \emph{finite} symmetric games played over the probability simplex. The reward $R$ is symmetric, i.e., $R(a_1,a_2)=R(a_2,a_1)$, and both players share the same private cost function $c(a_i)$.

Although the game is symmetric, its equilibria can be asymmetric and exhibit \emph{free-riding}: one player may exert significantly less effort while still enjoying high utility. As our experiments show, such equilibria are ubiquitous across many MARL benchmarks and pose a significant challenge to the generalization of agents' strategies.
Concretely, we define free-riding in our game as follows.
\begin{definition}[free-riding]
    An RQE $(x_1, x_2)\in\Delta(\mathcal{A})\times\Delta(\mathcal{A})$ exhibits free-riding with degree $\delta\geq 0$ if the difference in costs paid by each player at the RQE is $\delta$; i.e., $| \mathbb{E}_{a_1\sim x_1}[c(a_1)] - \mathbb{E}_{a_2\sim x_2}[c(a_2)]|=\delta$. If $\delta=0$, the RQE does not exhibit free-riding.
\end{definition}
\vspace{-1ex}
We now show that as the degree of risk aversion increases, free-riding becomes less prevalent.
\begin{theorem}[risk removes free-riding]\label{thm:free_riding}
    There exists a game-dependent constant $C$ depending only on $R$, $c$, and the degree of bounded rationality $\epsilon$ such that if players' degrees of risk aversion $\tau$ in a two-player collaborative game satisfy
    $\tau >\frac{C}{\delta^2}$,
    then the game admits no RQE with degree of free-riding greater than $\delta$.
\end{theorem}

\vspace{-1ex}

This theorem (proven in \cref{app:free_riding_proof}) shows that as strategic risk aversion increases, RQE must exhibit less free-riding. The intuition for this result is as follows: suppose that player $i$ free-rides at an RQE. If they become risk-averse, the worst-case deviation for their opponent is simply to stop putting any effort into the game. Since player $i$ is free-riding and not putting their own effort, this deviation causes a large drop in performance. Thus, at a risk-averse equilibrium, a player must contribute some of their own effort. We validate this intuition in a simple coordination game.

\begin{example}\label{example:free_riding}
    Consider the game with $\mathcal{A}=\{\text{C}=\text{collaborate},\text{D}=\text{defect}\}$,  $R(a_1,a_2)=1$ if $a_1=\text{C}$ or $a_2=\text{C}$ (at least one player collaborates) and $R(a_1,a_2)=0$ if $a_1=a_2=\text{D}$ (both defect), and $c(a_i)=0.4$ if $a_i=\text{C}$ (cost of collaboration) and $c(a_i)=0.0$ otherwise.
    We fix $\epsilon_i=0.2$ and study the effect of the degree of risk aversion on the game in~\cref{fig:theory}.
    Without risk aversion, the game has one symmetric and two free-riding RQE. As the degree of risk aversion increases, players become increasingly collaborative and the degree of free-riding diminishes. Importantly, there is a threshold  where the two free-riding RQE disappear and ``merge'' into the symmetric one.%
\end{example}

We validate this theorem in both the Overcooked Gridworld and the Tag environment. In both settings, Nash agents tend to learn free-riding behaviors, which severely impair generalization to new partners. By contrast, our strategically risk-averse agents contribute more actively to the task and achieve better partner generalization. Ablation studies further reveal a threshold-like phenomenon under which free-riding disappears as risk aversion increases past a certain threshold. This validates that the implications of~\cref{thm:free_riding} hold beyond the simple classes of games covered by the theorem.

\section{MARL Algorithm Design}\label{sec:algorithm}

The theoretical benefits of strategic risk aversion---robustness to other players' behaviors, increased collaboration and no free-riding---prompt us to design scalable \emph{strategically risk-averse} training algorithms for MARL. We do so in the general setting of policy optimization algorithms. For a more detailed mathematical background of risk aversion in MARL, we refer to Appendix~\ref{app:srpo_math_background}.

\paragraph{Meta-algorithm for strategically risk-averse policy optimization.}In standard policy optimization, each agent seeks a policy $\pi_{\theta_i}$, parametrized by parameters $\theta_i$, that maximizes an objective function $\mathcal{L}_i(\theta_i,\theta_{-i})$, which is a function of the policies of all agents.
Inspired by our strategically risk-averse utility~\eqref{eq:risk_adjusted_objective}, it is tempting to replace this objective function with
\begin{equation}\label{eq:srpo:loss_with_sup}
    \inf_{\phi_i}\mathcal{L}_i(\theta_i,\phi_i) + \frac{1}{\tau_i}\KL(\phi_i,\theta_{-i}) - \epsilon_i H(\theta_i),
\end{equation}
where, to ease exposition, we slightly overload the notation of $\KL(\phi_i,\theta_{-i})$ to be the (expected) KL-divergence between the induced policies across states (if the problem is one of MARL) .
However, directly optimizing this objective is computationally prohibitive, since both evaluation and gradient computation require solving a high-dimensional non-convex inner optimization over $\phi_i$.

To resolve this complexity, we follow~\citet{mazumdar2024tractableequilibriumcomputationmarkov} and resort to an auxiliary game in which we augment the number of players.
For each agent $i$, we bypass the direct computation of the infimum in~\eqref{eq:srpo:loss_with_sup} by introducing a \emph{fictitious adversary} that directly interacts with this agent during training. This adversary aims to minimize agent $i$’s performance while remaining close to the evolving policies of the other agents, thereby evaluating the infimum in~\eqref{eq:srpo:loss_with_sup}.  As such, the objective of agent $i$ is now a function of their parameters $\theta_i$ and those of their adversary:
\begin{equation}\label{eq:srpo:loss}
    \mathcal{L}_i(\theta_i,\phi_i) + \frac{1}{\tau_i}\KL(\phi_i,\theta_{-i}) - \epsilon_i H(\theta_i),
\end{equation}
where the KL term can now be dropped as it does not affect the agent's parameters. The objective of the fictitious adversary agent, who controls the adversarial policy $\pi_{\phi_i}$ and whose goal is to minimize the objective of agent $i$, depends on the parameters of all agents and is therefore given by
\begin{equation}
    -\mathcal{L}_i(\theta_i,\phi_i) - \frac{1}{\tau_i}\KL(\phi_i,\theta_{-i}) + \epsilon_i H(\theta_i),
\end{equation}
where the entropy term can be dropped. While adversarial training has been proposed as a robustness technique, such methods are often unstable to train and result in over-conservative policies~\citep{advrobust, lauffer2025robust}. The key insight resulting from strategic risk aversion is the idea to constrain the adversary to not deviate from the evolving opponent policies. This extra step stabilizes training as the adversary is constrained to a set of plausible deviations and not arbitrary attacks.

This framing yields a meta-algorithm for strategically risk-averse MARL. In \cref{thm:matrix_game_srpo_convergence}, we prove that gradient descent on this meta-algorithm converges to an RQE in matrix games, which provides theoretical motivation for instantiating it with practical policy optimization methods; see \cref{app:convergence} for details. We then instantiate the framework with Proximal Policy Optimization (PPO~\citep{schulman2017proximalpolicyoptimizationalgorithms}), also called independent PPO (IPPO) when deployed in multi-agent settings \citep{dewitt2020independentlearningneedstarcraft}, by setting $    \mathcal{L}_i(\theta_i,\phi_i)$ to the standard clipped surrogate objective.
We refer to the resulting algorithm as \textit{Strategically Risk-averse Policy Optimization} (SRPO). See \cref{app:srpo_details} for the pseudo-code and details. In \cref{sec:experiments}, we show that policy sharing can be incorporated to improve scalability.

\begin{figure*}[t!]
    \centering
    \includegraphics[width=0.9\textwidth]{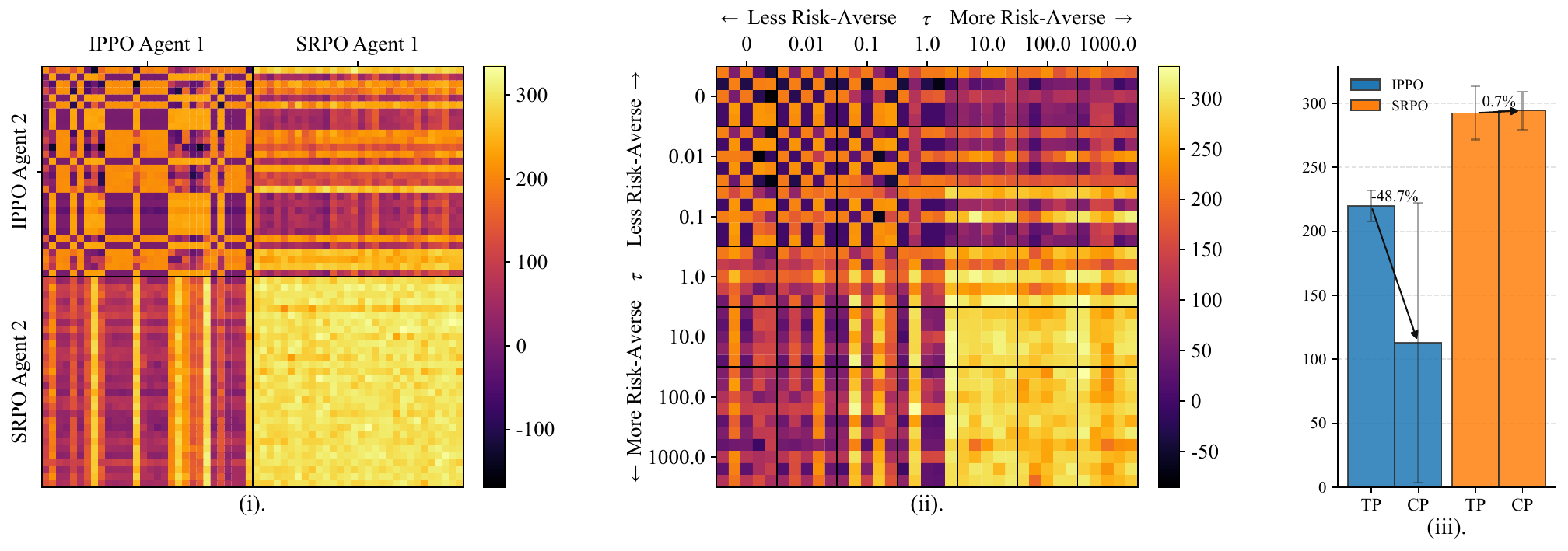}
    \caption{Overcooked cross-play and ablations. Each square is the average return over 10 episodes of length 128 for each pair of agents. Diagonal blocks represent the training performance of the agents. (i) We directly observe that IPPO ($\epsilon=0.1$) learns to free-ride while SRPO ($\tau=10, \epsilon=0.1$) does not, and SRPO yields higher utility (cf. \cref{thm:rqe_cooperation}). (ii) Ablation over $\tau$ ($\epsilon=0.1$). We empirically observe that free-riding vanishes as $\tau$ increases, matching~\cref{thm:free_riding}. (iii) Difference between Training Performance (TP) and Cross-play Performance (CP) (mean and standard deviation): the performance of IPPO drastically decreases, with lower average and larger standard deviation in cross-play, while the performance of SRPO is unaffected.}
    \label{fig:overcooked}
\end{figure*}

\section{Experiments}\label{sec:experiments}
We evaluate the proposed SRPO algorithm against IPPO, the dominant scalable algorithm for collaborative MARL~\citep{yu2022mappo,PPOentropy,dewitt2020independentlearningneedstarcraft}, on four environments: (1) a modified grid-world inspired by Overcooked AI \citep{carroll2019}, with both shared team rewards and private costs; (2) Tag \citep{lowe2017}, a control coordination task; (3) the multi-robot warehouse (RWARE) environment \citep{papoudakis2021benchmarking}, and (4) an LLM-based multi-agent debate setting \citep{du2024debate,park-etal-2025-maporl} on GSM8K~\citep{cobbe2021training}. We focus on IPPO because PPO-based methods have been shown to outperform several population-based alternatives while remaining scalable, and this also allows us to isolate the effect of adding strategic risk aversion to the same policy optimization method. Nevertheless, to ensure comparison with a representative partner-generalization baseline, we also implement FCP~\citep{strouse2022collaboratinghumanshumandata} in Overcooked Gridworld; results in \cref{sec:app-overcooked} show that SRPO substantially outperforms FCP. We further provide additional results on Hanabi~\citep{bard2020}, a coordination card game, in \cref{sec:hanabi}. Our experimental design closely follows the theoretical motivations developed earlier and evaluates the following claims:
\begin{itemize}[leftmargin=10pt,nosep,noitemsep]
    \item\textbf{Partner generalization.} SRPO achieves higher and more stable cross-play performance than IPPO when paired with previously unseen partners, and even varying number of partners.
    \item \textbf{Shared reward and free-riding.} In cooperative tasks with private costs, SRPO converges to equilibria with higher shared reward and reduced free-riding behavior.
    \item \textbf{Scalability.} SRPO remains practical in complex multi-agent settings when implemented with policy sharing and adversary sampling.
\end{itemize}

\paragraph{Training and evaluation.} Across all environments, the number of interactions with the environment for both SRPO and IPPO are kept the same. We evaluate partner generalization using \textbf{cross-play performance} (CP), where a trained agent is paired with \emph{held-out} partners trained from different seeds and not encountered during training, corresponding to off-diagonal entries in the evaluation matrix. We also report training performance (TP), where the agent is paired with itself, corresponding to diagonal entries. This evaluation makes free-riding policies especially visible: they produce a distinctive \emph{checkerboard} pattern when free-riding agents are paired with each other and fail to achieve high performance. We consistently observe this phenomenon among IPPO-trained agents across environments.

\paragraph{Ablation studies.} Across environments, we also report results of ablation studies on both entropy and strategic risk aversion to further validate our theory, show how our results are robust to hyper-parameter selection, and highlight how entropy alone cannot guarantee generalization in general collaborative games despite positive empirical evidence in special cases~\citep{PPOentropy}.
The code and visualizations of the trained policies are available \href{https://github.com/crqu/srpo}{here} and the detailed setup is in \cref{sec:experimental-details}.

\begin{figure*}[t!]
    \centering
      \includegraphics[width=0.9\textwidth]{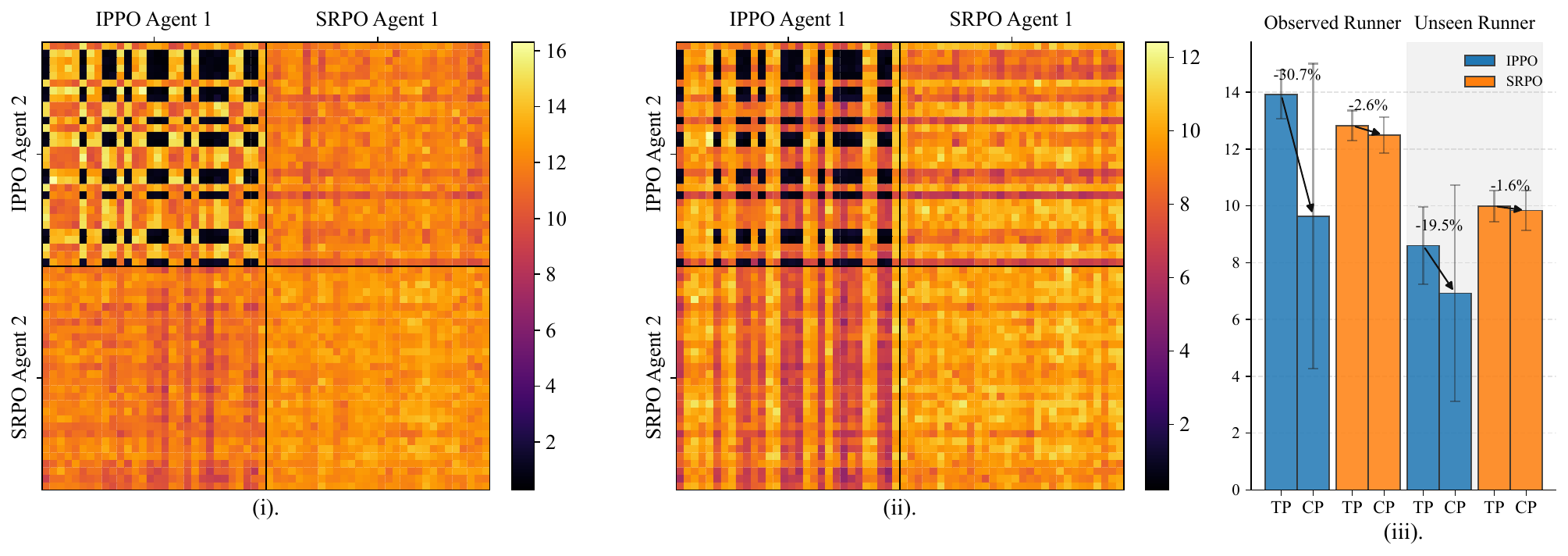}
    \caption{
    Tag cross-play with SRPO ($\tau=10, \epsilon=0.01$) and IPPO ($\epsilon=0.01$) against a runner seen during training (i) and an unseen runner (ii). Each square averages over 100 episodes. IPPO does well in training (yet still learns free-riding like policies), but degrades drastically against an unseen runner. SRPO trades a small amount of training performance for substantially smaller TP-CP gap (iii).
    }
    \label{fig:tag-cross-play}
    \vspace{-2ex}
\end{figure*}

\subsection{Overcooked Gridworld}\label{sec:exp-overcooked}
We construct an environment based on the Overcooked~\citep{carroll2019} multi-agent benchmark to serve as a simple laboratory to verify our theory. In our Overcooked Gridworld, agents receive a shared \emph{team reward} when an onion is picked up ($+1$) and placed into a pot ($+10$) and incur \emph{private costs} for movement ($-0.2$) and collision ($-2$). This induces a canonical social dilemma: although the team objective is collaborative, each agent has an incentive to avoid costly effort and rely on its teammate to complete the task. Results are shown in \cref{fig:overcooked}.
\paragraph{SRPO reduces free-riding and exhibits higher performance.} Across all 30 seeds, IPPO converges to a free-riding equilibrium where one teammate avoids movement but collects high reward due to the partner's effort, as reflected by the checkerboard pattern in the top-left block of~\cref{fig:overcooked}(i). In contrast, SRPO agents learn to coordinate and \textit{both} contribute to the task. This matches SRPO's objective: because training explicitly optimizes performance against \emph{plausible adversarial partner deviations} (\cref{sec:algorithm}), non-contribution becomes risky. As summarized in~\cref{fig:overcooked}(iii), SRPO attains both higher performances and improved cross-play reliability, consistent with our theoretical predictions. Ablations in \cref{fig:overcooked}(ii) show that IPPO ($\tau=0$) and small values of $\tau$ lead to free-riding and thus poor cross-play performance, reflecting fragile coordination conventions and lack of effort on the part of agents. As $\tau$ increases, collaboration becomes more stable: policies converge to a consistent solution, free-riding disappears, and cross-play performance improves, validating that Theorem~\ref{thm:free_riding} holds more broadly. SRPO is also robust to choices of $\tau$, with ablations showing consistent gains over IPPO across orders of magnitude of $\tau$. Similar trends are observed in Tag and RWARE in \cref{sec:ablation}.

\subsection{Tag}\label{sec:exp-tag}
In Tag, two chasers must coordinate to catch a runner. The chasers get a shared positive reward (+1) if a chaser collides with the runner. We first train runner policies independently via adversarial training, and then train chaser policies using either IPPO or SRPO against a fixed runner to keep the game collaborative. Cross-play performance is evaluated along two axes: (i) pairing trained chasers with an unseen chaser partner (teammate shift), and (ii) evaluating chasers against a runner policy not observed during training (opponent shift). The results are shown in \cref{fig:tag-cross-play}.

\paragraph{SRPO learns more generalizable policies.} In Tag, IPPO-trained agents can achieve strong training performance, but their policies often encode free-riding and overfit to both the runner encountered during training and specific coordination conventions it develops with its training partner. This yields strong in-distribution performance, but leads to sharp degradation under partner or opponent shift (\cref{fig:tag-cross-play}(iii)). In contrast, SRPO sacrifices a small amount of training performance---unsurprising here since Tag lacks the aggregative structure required for~\cref{thm:rqe_cooperation}---but attains consistently higher and more stable cross-play performance against different partners (including free-riding partners) and runners, showing their potential to generalize more widely.
\subsection{Multi-robot Warehouse}
In the multi-robot warehouse (RWARE) environment, robots must coordinate to deliver requested goods without collision. We train SRPO and IPPO independently in \texttt{rware-tiny-4ag-easy-v2} setting with four robots, and evaluate them with varying number of partners.
\paragraph{Scalability of SRPO.} To scale SRPO to larger agent populations, we use policy sharing: all agents share a single risk-averse policy, and training maintains only one adversary policy, with the adversarial role assigned to a uniformly-random agent at each rollout. For cross-play evaluation, half of the agents use one policy and the remaining half use the other. As shown in \cref{fig:rware-cross-play}(i), IPPO performance drops quickly as the number of agents increases due to more frequent robot collisions. In contrast, SRPO can achieve higher delivery rates through more robust coordination. Moreover, \cref{fig:rware-cross-play}(ii) shows that SRPO does not sacrifice training performance in the four-agent setting, while achieving substantially stronger training (all agents share the same policy) and cross-play performance in the ten-agent setting.
\begin{figure*}[t!]
    \centering
      \includegraphics[width=0.9\textwidth]{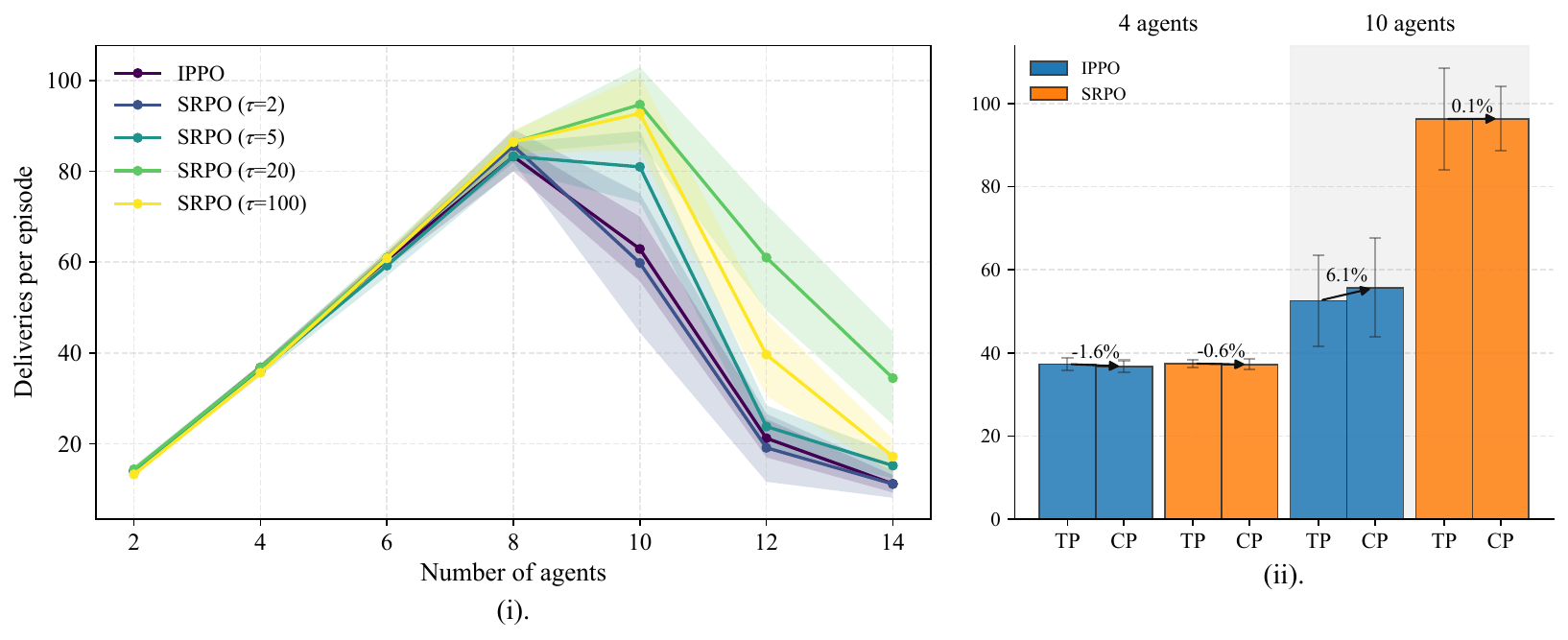}
    \caption{
    RWARE performances of SRPO and IPPO ($\epsilon=0.05$). (i) Episodic delivery rate for both IPPO and SRPO, with 95\% confidence intervals, as the number of agents varies; (ii) Difference between Training Performance (TP) and Cross-play Performance (CP) (mean and standard deviation): SRPO ($\tau=20$) does not sacrifice performance in the training (4 agents TP), but achieves much better CP and TP performances in the ten-agent cases.
    }
    \label{fig:rware-cross-play}
    \vspace{-2ex}
\end{figure*}

\begin{figure}[t!]
\centering
\includegraphics[width=0.5\linewidth]{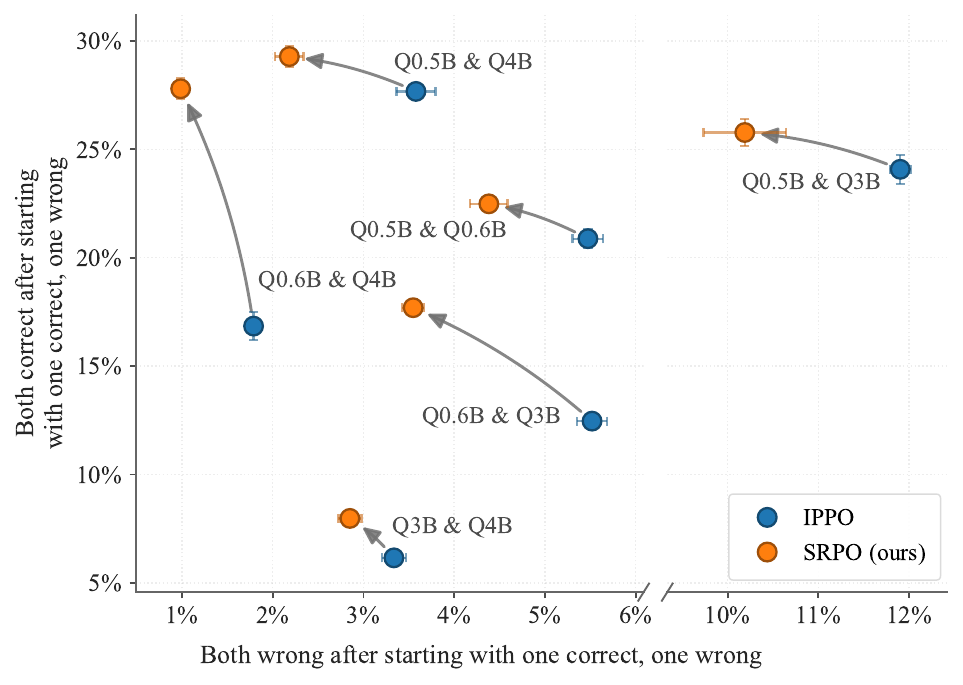}
\caption{\textbf{How strategic risk aversion helps debate.} For each method and cross-play pair, we report the fraction of all evaluated debates in which exactly one agent is correct after the first round and both agents are correct at the end \textbf{($y$-axis)} or both are wrong at the end \textbf{($x$-axis)}. SRPO consistently increases the fraction ending with both agents correct and decreases the fraction ending with both agents wrong relative to IPPO.}
\label{fig:qwen-debate}
\vspace{-2ex}
\end{figure}

\subsection{Multi-LLM-Agent Debate on GSM8K}
To conclude, in a language-based cooperative reasoning setting, we train SRPO and IPPO agents to solve GSM8K math problems~\citep{cobbe2021training} through structured debate, where each agent must reinforce correct reasoning while remaining robust to potentially misleading or incorrect proposals from its teammate.
We use Qwen2.5-\{\textbf{0.5B}, \textbf{3B}\}-Instruct~\citep{qwen2025qwen25technicalreport} and Qwen3-\{\textbf{0.6B}, \textbf{4B}-Instruct-2507\}~\citep{yang2025qwen3} as base models, trained with verl~\citep{sheng2024hybridflow}. See \cref{sec:llm-debate} for full setup. We evaluate two robustness dimensions: (i) cross-play \textbf{joint accuracy} (problem correct only if \emph{both} agents output the correct final answer) when agents with different model scales are paired, and (ii) \textbf{individual accuracy} when paired with a trained Qwen agent with an \emph{untuned} Llama~3.2-1B-Instruct model~\citep{grattafiori2024llama3herdmodels}, isolating robustness to drastic partner shift. Results are summarized in \cref{tab:gsm8k-robustness}.
\begin{table}[t]
\centering
\caption{GSM8K debate robustness. Panel (a) reports joint accuracy for pairs of Qwen models. Panel (b) reports the Qwen agent's individual accuracy with an untuned Llama~3.2-1B-Instruct partner. Values are means with 95\% confidence intervals over five evaluation rollouts. Gains are in percentage points. We set $\epsilon=0$ for both methods and $\tau=10$ for SRPO.}
\label{tab:gsm8k-robustness}
\footnotesize
\setlength{\tabcolsep}{2pt}
\begin{minipage}[t]{.52\linewidth}
\textbf{(a) Qwen cross-play joint accuracy}\par\smallskip
\begin{tabular*}{\linewidth}{@{}l@{\extracolsep{\fill}}rrr@{}}
\toprule
\textbf{Qwen pair} & \textbf{IPPO} & \textbf{SRPO} & \textbf{Gain (pp)} \\
\midrule
0.5B / 0.6B & $0.603 \pm 0.004$ & $0.621 \pm 0.001$ & $\mathbf{+1.71} \pm 0.44$ \\
0.5B / 3B & $0.649 \pm 0.005$ & $0.684 \pm 0.008$ & $\mathbf{+3.49} \pm 0.69$ \\
0.5B / 4B & $0.845 \pm 0.005$ & $0.872 \pm 0.002$ & $\mathbf{+2.67} \pm 0.50$ \\
0.6B / 3B & $0.572 \pm 0.003$ & $0.615 \pm 0.002$ & $\mathbf{+4.25} \pm 0.41$ \\
0.6B / 4B & $0.710 \pm 0.007$ & $0.848 \pm 0.001$ & $\mathbf{+13.81} \pm 0.62$ \\
3B / 4B & $0.706 \pm 0.005$ & $0.732 \pm 0.004$ & $\mathbf{+2.67} \pm 0.35$ \\
\bottomrule
\end{tabular*}
\end{minipage}
\hfill
\begin{minipage}[t]{.45\linewidth}
\textbf{(b) Accuracy with untuned Llama}\par\smallskip
\begin{tabular*}{\linewidth}{@{}l@{\extracolsep{\fill}}rrr@{}}
\toprule
\textbf{Qwen} & \textbf{IPPO} & \textbf{SRPO} & \textbf{Gain (pp)} \\
\midrule
0.5B & $0.379 \pm 0.003$ & $0.405 \pm 0.004$ & $\mathbf{+2.56} \pm 0.48$ \\
0.6B & $0.588 \pm 0.004$ & $0.673 \pm 0.009$ & $\mathbf{+8.49} \pm 0.58$ \\
3B & $0.552 \pm 0.006$ & $0.634 \pm 0.007$ & $\mathbf{+8.26} \pm 0.95$ \\
4B & $0.900 \pm 0.004$ & $0.917 \pm 0.004$ & $\mathbf{+1.65} \pm 0.14$ \\
\bottomrule
\end{tabular*}
\end{minipage}
\vspace{-3ex}
\end{table}

\paragraph{SRPO extends to LLM-based multi-agent systems.}
\cref{tab:gsm8k-robustness} shows that SRPO consistently improves cross-play joint accuracy, with gains up to $13.81$ percentage points, as well as individual accuracy when paired with an untuned partner, with gains up to $8.49$ percentage points. Across all evaluated pairings, \cref{fig:qwen-debate} further shows that SRPO increases the fraction of debates that start with one correct and one incorrect answer and end with both agents correct, and decreases the fraction that start in the same state and end with both agents wrong, relative to IPPO.
This pattern suggests that SRPO agents are harder to mislead and more effective at persuading an incorrect partner to adopt the correct answer. These behaviors illustrate how strategic risk aversion can improve robustness to partner behavior during debate. Overall, SRPO promotes more robust and generalizable cooperative reasoning in multi-agent language environments.

\section{Conclusion}

We studied strategic risk aversion as an inductive bias for training collaborative agents. In structured games, we showed that it can increase shared rewards and reduce free-riding at equilibrium. These results motivated SRPO, which improves collaboration with unseen partners across several benchmarks, including multi-agent LLM debate.
More broadly, this paper motivates further work on integrating more aspects of human decision-making from behavioral economics, such as mutual trust, altruism, and fairness, into multi-agent reinforcement learning and AI training to improve the general awareness and behavior of autonomous agents. 
\clearpage
\begin{ack}
Nicolas Lanzetti was supported by the PIMCO Fellows Program, by the National Science Foundation under Grant \#CCF-2326609, and by the Resnick Sustainability Institute. Eric Mazumdar acknowledges support from NSF Award 2240110.
\end{ack}

\bibliographystyle{abbrvnat}
\bibliography{risk}

\newpage
\appendix
\section{Further Discussion on Related Work}\label{app:related}
In this section, we discuss a broader set of related work to better position our work in the landscape of collaborative multi-agent learning and risk-averse/robust MARL.
\paragraph{Collaborative MARL and Partner Generalization}
Research into solving collaborative decision-making problems using MARL has a long history~\citep{old_survey_coop}, originating from a desire for solving decentralized decision-making problems~\citep{old_tsitskilis}. Far from fading, these issues remain at the fore of research interest with the advent of powerful AI systems and a desire to build large teams of AI agents~\citep{tran2025multiagentcollaborationmechanismssurvey} for solving collaborative tasks. Emerging from this literature, is a general consensus that classic single-agent policy optimization algorithms like PPO~\citep{schulman2017proximalpolicyoptimizationalgorithms}, used in a decentralized way, can yield state-of-the-art or near state-of-the-art performance without needing to be overly specialized to the multi-agent nature of the problem~\citep{yu2022mappo,dewitt2020independentlearningneedstarcraft,PPOentropy}. Despite such results, learned policies resulting from such approaches can fail to generalize to unseen agents~\citep{PPOentropy} and new environments. We validate this through our experiments and further observe that these state-of-the-art algorithms are consistently prone to learning free-riding strategies.

Beyond the lines of work mentioned in Section~\ref{sec:related}, other important lines of work on addressing the partner generalization problem seek to avoid brittle coordination conventions by learning symmetry-invariant or partner-agnostic policies \citep{hu20a,treutlein21a,muglich2022equivariant}. These  often require strong structural assumptions and may not scale to complex collaborative environments. Another set of approaches attempt to achieve robustness by training against partner policies derived from human data \citep{carroll2019,meta2022,liang2024learning}, but their effectiveness is constrained by data availability and coverage. Neither of these approaches are fully agnostic to the structure of the underlying problem and are thus not general principles that can be broadly applied across different tasks.

Consequently, it remains unclear how to systematically achieve partner generalization without relying on human data or ad-hoc constructions. In this paper we showed that strategic risk aversion gives us one potential approach. Furthermore, we show that it can readily be combined with policy optimization algorithms like PPO, resulting in scalable and performant algorithms for learning generalizable policies.

\paragraph{Robustness and Risk Aversion in MARL}
A second related line of work is the emerging literature on robust and risk-averse MARL~\citep{ZhangRobust,shiRobustMARL,yekkehkhany2020risk,slumbers2023game,qiu2021rmix}. While most of the papers study robustness or risk aversion to changes in the underlying environment~\citep{eriksson2022risk,ganesh2019reinforcement,qiu2021rmix,shen2023riskq,shiRobustMARL,ZhangRobust} or to tail risks in large populations of agents~\citep{yekkehkhany2020risk}, in this work we study \emph{strategic} risk aversion. Strategic risk aversion and robustness yield solutions that are more amenable to computation than classic game-theoretic solution concepts~\citep{mazumdar2024tractableequilibriumcomputationmarkov}, ensure robustness to unseen behaviors of the other agents, and sometimes even induce a coordination effect between the agents~\citep{lanzetti2025strategically}---an empirical observation also seen in~\citep{feik2026strategically,velasevic2026strategically} that we make rigorous in Section~\ref{sec:theory}. Despite the prior work in this area, the broader benefits of the concept in collaborative games are unknown. This is the problem that we study in this paper.

\section{Proofs for \cref{sec:theory}}

In this section, we provide the proofs of our two main theorems,~\cref{thm:rqe_cooperation} and~\cref{thm:free_riding}, as well as the derivation of~\cref{example:continuous_game} used in~\cref{fig:theory}.

\subsection{Proof of~\cref{thm:rqe_cooperation}}\label{sec:proof:thm:rqe_cooperation}
We begin with the proof of~\cref{thm:rqe_cooperation}. We first show that there exists an RQE in the aggregative games studied in this section. We then provide the proof of a more general version of~\cref{thm:rqe_cooperation}.
\subsubsection{Preliminaries}

Before proving~\cref{proposition:continuous_game:rqe}, we study RQE for quadratic games.
We start with the risk-averse quantal best response.
We conduct our analysis in slightly more general settings where the players' utilities are
\begin{equation}\label{proposition:continuous_game:rqe:proof:general_payoff}
    u_i(a_i,a_{-i})
    =
        \frac{1}{2}\left\langle
        \begin{bmatrix} a_i \\ a_{-i} \end{bmatrix},
        \begin{bmatrix}
            H_{i,i} & H_{i,-i} \\ H_{-i,i} & H_{-i,-i}
        \end{bmatrix}
        \begin{bmatrix} a_i \\ a_{-i} \end{bmatrix}
        \right\rangle
        +
        \left\langle
        \begin{bmatrix}
            h_i \\ h_{-i}
        \end{bmatrix},
        \begin{bmatrix} a_i \\ a_{-i} \end{bmatrix}\right\rangle,
\end{equation}
where the matrix
\begin{equation*}
    \begin{bmatrix}
        H_{i,i} & H_{i,-i} \\ H_{-i,i} & H_{-i,-i}
    \end{bmatrix}
\end{equation*}
is assumed to be symmetric negative semidefinite, and $H_{i,i}$ is assumed to be symmetric negative definite for all $i$.
Also note that, for now, we do not assume that $h\neq 0$.
Throughout this section, with slight abuse of notation, we denote by $x_{-i}$ the product measure of the mixed strategies of all players except $i$.
Moreover, in this quadratic setting, we set $U_i^{\tau_i,\epsilon_i}(x_i,x_{-i})\coloneqq-\infty$ whenever $x_i\notin\mathcal{P}_2(\mathbb{R}^n)$, i.e., whenever $x_i$ has infinite second moment. This is consistent with the fact that, since $H_{i,i}\prec0$, the expected utility $\mathbb{E}_{a_i\sim x_i}[u_i(a_i,a_{-i})]$ equals $-\infty$ for all $a_{-i}$ for such $x_i$, and it avoids undefined expressions in the entropy term.
Similarly, we restrict the adversary in~\eqref{eq:risk_adjusted_objective} to $p\in\mathcal{P}_2(\mathbb{R}^{n(N-1)})$, so that all expectations are well defined; as the proof of \cref{lemma:best_response} shows, this does not change the value of the infimum.

\begin{lemma}[risk-averse quantal best response]\label{lemma:best_response}
    Let $x_{-i}$ be a probability measure on $\mathbb{R}^{n(N-1)}$ with finite second moment, and let
    \begin{equation*}
        \varphi_i(m)
        \coloneqq
        \frac{1}{\tau_i}\log\mathbb{E}_{a_{-i}\sim x_{-i}}\left[\exp\left(-\tau_iq_i(m,a_{-i})\right)\right]
        \in(-\infty,+\infty],
    \end{equation*}
    where $q_i(m,a_{-i})\coloneqq\frac{1}{2}\langle a_{-i},H_{-i,-i}a_{-i}\rangle+\langle m,H_{i,-i}a_{-i}\rangle+\langle h_{-i},a_{-i}\rangle$.
    If $\varphi_i(m)<+\infty$ for all $m\in\mathbb{R}^n$, then the risk-averse quantal best response, i.e., the maximizer of $U_i^{\tau_i,\epsilon_i}(\cdot,x_{-i})$ over $\Delta(\mathbb{R}^n)$, is uniquely given by the Gaussian mixed strategy with covariance matrix $\Sigma_i=-\epsilon_iH_{i,i}^{-1}$ and mean
    \begin{equation}\label{eq:best_response:mean}
        m_i=\argmax_{m\in\mathbb{R}^n}\ \frac{1}{2}\langle m,H_{i,i}m\rangle+\langle h_i,m\rangle-\varphi_i(m),
    \end{equation}
    where the maximizer exists and is unique.
    In particular, let $x_{-i}$ be Gaussian with mean $m_{-i}$ and covariance matrix $\Sigma_{-i}\succ 0$, and let $P_{-i}\coloneqq\frac{1}{\tau_i}\Sigma_{-i}^{-1} + H_{-i,-i}$. If $P_{-i}$ is positive definite, then $\varphi_i<+\infty$ on $\mathbb{R}^n$ and the best response has
    \begin{align*}
        \Sigma_i &= -\epsilon_i H_{i,i}^{-1}
        \\
        m_i &=
        \left(
            -H_{i,i} + H_{i,-i} P_{-i}^{-1} H_{i,-i}^\top
        \right)^{-1}
        \left(
            h_i + H_{i,-i}P_{-i}^{-1}\left(\frac{1}{\tau_i}\Sigma_{-i}^{-1}m_{-i} - h_{-i}\right)
        \right).
    \end{align*}
    If instead $P_{-i}$ is not positive definite, then $\varphi_i\equiv+\infty$ and the risk is infinite for every mixed strategy $x_i$.
\end{lemma}
\begin{proof}
    Recall that $U_i^{\tau_i,\epsilon_i}(x_i,x_{-i})=-\infty$ whenever $x_i\notin\mathcal{P}_2(\mathbb{R}^n)$, and let $x_i\in\mathcal{P}_2(\mathbb{R}^n)$ have mean $m\coloneqq\int_{\mathbb{R}^n}a_i\diff x_i(a_i)$.
    With $V_i(a_i)\coloneqq-\frac{1}{2}\langle a_i,H_{i,i}a_i\rangle-\langle h_i,a_i\rangle$, we have $-u_i(a_i,a_{-i})=V_i(a_i)-q_i(a_i,a_{-i})$, and, since $q_i$ is affine in its first argument, $f(a_{-i})\coloneqq\mathbb{E}_{a_i\sim x_i}[-u_i(a_i,a_{-i})]=\int_{\mathbb{R}^n}V_i\diff x_i-q_i(m,a_{-i})$.
    By the duality representation of the entropic risk measure (e.g., see~\citet{follmer2002convex}),
    \begin{multline*}
        -U_i^{\tau_i,\epsilon_i}(x_i,x_{-i})
        =
        \sup_{p\in\mathcal{P}_2(\mathbb{R}^{n(N-1)})}\left(\mathbb{E}_{a_{-i}\sim p}[f(a_{-i})]-\frac{1}{\tau_i}\KL(p,x_{-i})\right)+\epsilon_iH(x_i)
        \\
        =
        \frac{1}{\tau_i}\log\mathbb{E}_{a_{-i}\sim x_{-i}}\left[e^{\tau_if(a_{-i})}\right]+\epsilon_iH(x_i)
        =
        \varphi_i(m)+\int_{\mathbb{R}^n}V_i\diff x_i+\epsilon_iH(x_i)
        \eqqcolon
        J_i(x_i).
    \end{multline*}
    \citet{follmer2002convex} state the second equality for bounded $f$; it extends, with both sides possibly $+\infty$, to our $f$, which is bounded below by an affine function since $H_{-i,-i}\preceq0$ (in particular, ``$\leq$'' follows from Jensen's inequality, and ``$\geq$'' from $p_k\propto e^{\tau_i\min\{f,k\}}x_{-i}$, which has finite second moment since its density with respect to $x_{-i}$ is bounded, with $k\to\infty$).
    Moreover, $J_i(x_i)\in(-\infty,+\infty]$, since $H(x_i)>-\infty$ for $x_i\in\mathcal{P}_2(\mathbb{R}^n)$ and, by Jensen's inequality, $\varphi_i(m)\geq-\mathbb{E}_{a_{-i}\sim x_{-i}}[q_i(m,a_{-i})]$, which is an affine function of $m$, since $x_{-i}$ has finite second moment.

    Now let $\varphi_i<+\infty$ on $\mathbb{R}^n$. Then, $\varphi_i$ is convex, by H\"older's inequality, hence continuous, and differentiable, with
    \begin{equation*}
        \nabla\varphi_i(m)
        =
        -\frac{\mathbb{E}_{a_{-i}\sim x_{-i}}\left[H_{i,-i}a_{-i}\exp\left(-\tau_iq_i(m,a_{-i})\right)\right]}{\mathbb{E}_{a_{-i}\sim x_{-i}}\left[\exp\left(-\tau_iq_i(m,a_{-i})\right)\right]},
    \end{equation*}
    by differentiation under the integral sign, which is justified since $\varphi_i$ is finite on all of $\mathbb{R}^n$.
    Since $H_{i,i}\prec0$ and $\varphi_i$ is convex and bounded below by an affine function, the function $\psi_i(m)\coloneqq\frac{1}{2}\langle m,H_{i,i}m\rangle+\langle h_i,m\rangle-\varphi_i(m)$ maximized in~\eqref{eq:best_response:mean} is strictly concave and continuous, and $\psi_i(m)\to-\infty$ as $\norm{m}\to\infty$. Hence, its maximizer $m_i$ exists, is unique, and is the unique solution of $\nabla\psi_i(m_i)=H_{i,i}m_i+h_i-\nabla\varphi_i(m_i)=0$.
    It remains to show that $x_i^\star\coloneqq\mathcal{N}(m_i,\Sigma_i)$ is the unique minimizer of $J_i$ over $\mathcal{P}_2(\mathbb{R}^n)$; since $J_i(x_i^\star)<+\infty$, $x_i^\star$ is then the unique best response. We use the calculus in the Wasserstein space of~\citet{lanzetti2024variational,lanzetti2025first}, and denote by $\partial^-J_i(x_i)$ the Wasserstein subdifferential of $J_i$ at $x_i$.

    \paragraph{Step 1: Motivation of the candidate.}
    If $x_i$ has a smooth positive density $\rho$, then, by the computations in Step~2, the condition $0\in\partial^-J_i(x_i)$ reads
    \begin{equation}\label{eq:best_response:gradient_zero}
        \nabla\varphi_i(m)
        -
        H_{i,i}a_i - h_i
        +
        \epsilon_i\frac{\nabla\rho(a_i)}{\rho(a_i)}
        =
        0
        \qquad
        \text{ for all $a_i\in\mathbb{R}^n$}.
    \end{equation}
    Thus, $\nabla\log\rho(a_i)=\frac{1}{\epsilon_i}(H_{i,i}a_i+h_i-\nabla\varphi_i(m))$ is affine and, since $H_{i,i}\prec0$, $\rho$ is the density of a Gaussian distribution with covariance $\Sigma_i=-\epsilon_iH_{i,i}^{-1}$ and mean $-H_{i,i}^{-1}(h_i-\nabla\varphi_i(m))$, which coincides with $m$ if and only if $\nabla\psi_i(m)=0$, i.e., $m=m_i$. This motivates the candidate $x_i^\star$.

    \paragraph{Step 2: Optimality and uniqueness.}
    First, $0\in\partial^-J_i(x_i^\star)$. Indeed, the first term of $J_i$ is Wasserstein differentiable at $x_i^\star$ with constant gradient $\nabla\varphi_i(m_i)$, as in the chain rule~\citep[Proposition 2.16]{lanzetti2025first}: for any optimal coupling $\gamma$ of $x_i^\star$ and $\nu$, the difference of the means of $\nu$ and $x_i^\star$ is $\int(y-x)\diff\gamma(x,y)$, whose norm is at most $W_2(x_i^\star,\nu)$, and $\varphi_i$ is differentiable. The second term is Wasserstein differentiable with gradient $\nabla V_i$~\citep[Proposition 2.20]{lanzetti2025first}. Since the density $\rho$ of $x_i^\star$ is smooth and positive with $\nabla\log\rho(a_i)=-\Sigma_i^{-1}(a_i-m_i)$ square integrable with respect to $x_i^\star$, $\nabla\log\rho$ is a Wasserstein subgradient of the negative entropy at $x_i^\star$~\citep[Example 2.27]{lanzetti2025first}. By the sum rule~\citep[Proposition 2.15]{lanzetti2025first}, the left-hand side of~\eqref{eq:best_response:gradient_zero} with $m=m_i$ is thus a Wasserstein subgradient of $J_i$ at $x_i^\star$, and it vanishes for all $a_i\in\mathbb{R}^n$ since $\epsilon_i\Sigma_i^{-1}=-H_{i,i}$ and $\nabla\psi_i(m_i)=0$.
    Second, $J_i$ is lower semicontinuous on $\mathcal{P}_2(\mathbb{R}^n)$, since its first two terms are continuous ($\varphi_i$ is continuous and $V_i$ has quadratic growth) and the negative entropy is lower semicontinuous~\citep[Proposition 2.25]{lanzetti2025first}.
    Moreover, $J_i$ is $\lambda$-geodesically convex with $\lambda\coloneqq\lambda_{\min}(-H_{i,i})>0$, since each of its terms is ($\lambda$-)convex along every geodesic: (i) the mean is affine along geodesics and $\varphi_i$ is convex; (ii) the potential energy $\int_{\mathbb{R}^n}V_i\diff x_i$ is $\lambda$-convex, since $\nabla^2V_i=-H_{i,i}\succeq\lambda I$~\citep[Proposition 2.20]{lanzetti2025first}; (iii) the negative entropy is displacement convex~\citep[Proposition 2.25]{lanzetti2025first}.
    Thus, by~\citep[Theorem 3.3]{lanzetti2025first}, $x_i^\star$ is the unique minimizer of $J_i$ over $\mathcal{P}_2(\mathbb{R}^n)$.

    \paragraph{Step 3: Gaussian opponents.}
    Let $x_{-i}$ be Gaussian with mean $m_{-i}$ and covariance $\Sigma_{-i}\succ0$. Collecting the exponents, and since $\Sigma_{-i}^{-1}+\tau_iH_{-i,-i}=\tau_iP_{-i}$, the density of $x_{-i}$ multiplied by $\exp(-\tau_iq_i(m,\cdot))$ is proportional to
    \begin{equation*}
        \exp\left(
        -\frac{1}{2}\langle a_{-i},\tau_iP_{-i}a_{-i}\rangle
        +\left\langle\Sigma_{-i}^{-1}m_{-i}-\tau_i\left(H_{i,-i}^\top m+h_{-i}\right),a_{-i}\right\rangle
        \right).
    \end{equation*}
    For every $m\in\mathbb{R}^n$, this function is integrable if and only if $P_{-i}\succ0$. Hence, if $P_{-i}$ is not positive definite, then $\varphi_i\equiv+\infty$, so that $U_i^{\tau_i,\epsilon_i}(x_i,x_{-i})=-J_i(x_i)=-\infty$ for every $x_i\in\mathcal{P}_2(\mathbb{R}^n)$, and, by convention, for every other $x_i$; i.e., the risk is infinite.
    If instead $P_{-i}\succ0$, then $\varphi_i<+\infty$ on $\mathbb{R}^n$, and this function is, up to normalization, the density of a Gaussian distribution with mean
    \begin{equation*}
        \bar m_{-i}(m)
        \coloneqq
        \frac{1}{\tau_i}P_{-i}^{-1}\left(\Sigma_{-i}^{-1}m_{-i}-\tau_i\left(H_{i,-i}^\top m+h_{-i}\right)\right),
    \end{equation*}
    so that, by the expression for $\nabla\varphi_i$ above, $\nabla\varphi_i(m)=-H_{i,-i}\bar m_{-i}(m)$.
    Thus, the stationarity condition $H_{i,i}m_i+h_i-\nabla\varphi_i(m_i)=0$ becomes the linear equation
    \begin{equation*}
        \left(-H_{i,i}+H_{i,-i}P_{-i}^{-1}H_{i,-i}^\top\right)m_i
        =
        h_i+H_{i,-i}P_{-i}^{-1}\left(\frac{1}{\tau_i}\Sigma_{-i}^{-1}m_{-i}-h_{-i}\right),
    \end{equation*}
    whose matrix is positive definite, since $H_{i,i}\prec0$ and $P_{-i}\succ0$. Solving for $m_i$ gives the expression in the statement.
\end{proof}

With this lemma, we obtain a way to compute RQE in continuous games:

\begin{lemma}\label{proposition:continuous_game:rqe:proof}
    For each player $i$, let $\Sigma_i\coloneqq-\epsilon_iH_{i,i}^{-1}\succ0$, let $\Sigma_{-i}\coloneqq\operatorname{blkdiag}(\Sigma_j)_{j\neq i}$, and let $P_{-i}\coloneqq\frac{1}{\tau_i}\Sigma_{-i}^{-1}+H_{-i,-i}$. Assume that $P_{-i}\succ0$ for all $i$.
    Consider a surrogate ``game of the means'', where each player selects a vector $m_i\in\mathbb{R}^n$ to maximize the surrogate concave quadratic utility
    \begin{equation}\label{eq:utility_game_means}
    \begin{aligned}
        \bar U_i(m_i,m_{-i})
        &=
        \frac{1}{2}
        \langle m_i,
        (H_{i,i} - H_{i,-i} P_{-i}^{-1} H_{i,-i}^\top) m_i\rangle
        \\
        &\qquad\qquad+
        \frac{1}{\tau_i}\langle m_i, H_{i,-i}P_{-i}^{-1}\Sigma_{-i}^{-1}m_{-i}\rangle
        +
        \langle m_i,h_i-H_{i,-i}P_{-i}^{-1}h_{-i}\rangle,
    \end{aligned}
    \end{equation}
    and, for $m=(m_1,\ldots,m_N)$, let $x(m)\coloneqq(x_1(m_1),\ldots,x_N(m_N))$ with $x_i(m_i)\coloneqq\mathcal{N}(m_i,\Sigma_i)$. Then:
    \begin{itemize}
        \item[(i)] $m^\star$ is a Nash equilibrium of the surrogate game if and only if $x(m^\star)$ is an RQE of the game with utilities~\eqref{proposition:continuous_game:rqe:proof:general_payoff};
        \item[(ii)] if $(x_1^\star,\ldots,x_N^\star)$ is an RQE with $U_i^{\tau_i,\epsilon_i}(x_i^\star,x_{-i}^\star)>-\infty$ for all $i$, then $(x_1^\star,\ldots,x_N^\star)=x(m^\star)$ for some Nash equilibrium $m^\star$ of the surrogate game.
    \end{itemize}
\end{lemma}

\begin{proof}
    Fix a player $i$ and $m_{-i}\in\mathbb{R}^{n(N-1)}$, and let $x_{-i}(m_{-i})\coloneqq\bigotimes_{j\neq i}x_j(m_j)$, i.e., the Gaussian distribution with mean $m_{-i}$ and covariance $\Sigma_{-i}\succ0$.
    Since $P_{-i}\succ0$, \cref{lemma:best_response} implies that $x_{-i}(m_{-i})$ satisfies the assumptions of its first part, and that the unique best response to $x_{-i}(m_{-i})$ is $\mathcal{N}(b_i(m_{-i}),\Sigma_i)=x_i(b_i(m_{-i}))$, where
    \begin{equation*}
        b_i(m_{-i})
        \coloneqq
        \left(-H_{i,i}+H_{i,-i}P_{-i}^{-1}H_{i,-i}^\top\right)^{-1}
        \left(h_i+H_{i,-i}P_{-i}^{-1}\left(\frac{1}{\tau_i}\Sigma_{-i}^{-1}m_{-i}-h_{-i}\right)\right).
    \end{equation*}
    On the other hand, since $H_{i,i}\prec0$ and $P_{-i}\succ0$, we have $H_{i,i}-H_{i,-i}P_{-i}^{-1}H_{i,-i}^\top\prec0$. Thus, $\bar U_i(\cdot,m_{-i})$ is strictly concave, and its unique maximizer is the unique solution of
    \begin{equation*}
        \left(H_{i,i}-H_{i,-i}P_{-i}^{-1}H_{i,-i}^\top\right)m_i
        +\frac{1}{\tau_i}H_{i,-i}P_{-i}^{-1}\Sigma_{-i}^{-1}m_{-i}
        +h_i-H_{i,-i}P_{-i}^{-1}h_{-i}
        =0,
    \end{equation*}
    namely $m_i=b_i(m_{-i})$.
    Hence, for all $i$ and $m_{-i}$, the unique best response to $x_{-i}(m_{-i})$ in the original game is $x_i(b_i(m_{-i}))$, where $b_i(m_{-i})$ is the unique best response to $m_{-i}$ in the surrogate game.

    \emph{Proof of (i).}
    If $m^\star$ is a Nash equilibrium of the surrogate game, then $m_i^\star=b_i(m_{-i}^\star)$ for all $i$, so $x_i(m_i^\star)$ is the best response to $x_{-i}(m_{-i}^\star)$ for all $i$; i.e., $x(m^\star)$ is an RQE.
    Conversely, if $x(m^\star)$ is an RQE, then, for all $i$, $x_i(m_i^\star)$ is a best response to $x_{-i}(m_{-i}^\star)$, and, by uniqueness of the best response, $x_i(m_i^\star)=x_i(b_i(m_{-i}^\star))$, i.e., $m_i^\star=b_i(m_{-i}^\star)$. Thus, $m^\star$ is a Nash equilibrium of the surrogate game.

    \emph{Proof of (ii).}
    Let $(x_1^\star,\ldots,x_N^\star)$ be an RQE with $U_j^{\tau_j,\epsilon_j}(x_j^\star,x_{-j}^\star)>-\infty$ for all $j$, and fix a player $i$. By the proof of~\cref{lemma:best_response}, each $x_j^\star$ has finite second moment, and, for every $x_i$ with finite second moment and mean $m$,
    \begin{equation*}
        U_i^{\tau_i,\epsilon_i}(x_i,x_{-i}^\star)
        =
        -\varphi_i(m)-\int_{\mathbb{R}^n}V_i(a_i)\diff x_i(a_i)-\epsilon_iH(x_i),
    \end{equation*}
    with $\varphi_i$ defined with $x_{-i}=x_{-i}^\star$.
    Let $m_i^\star$ be the mean of $x_i^\star$ and $G_i\coloneqq\mathcal{N}(m_i^\star,\Sigma_i)$. Since $U_i^{\tau_i,\epsilon_i}(x_i^\star,x_{-i}^\star)>-\infty$, we have $\varphi_i(m_i^\star)<+\infty$ and $H(x_i^\star)<+\infty$, so $U_i^{\tau_i,\epsilon_i}(G_i,x_{-i}^\star)$ is finite as well.
    Since $\Sigma_i^{-1}=-H_{i,i}/\epsilon_i$, the log-density of $G_i$ is $-V_i/\epsilon_i$ plus an affine function of $a_i$; as $x_i^\star$ and $G_i$ have the same mean, this yields $\epsilon_i\int_{\mathbb{R}^n}\log\frac{\diff G_i}{\diff a_i}\diff(x_i^\star-G_i)=-\int_{\mathbb{R}^n}V_i\diff(x_i^\star-G_i)$, and thus
    \begin{equation*}
        U_i^{\tau_i,\epsilon_i}(G_i,x_{-i}^\star)
        -
        U_i^{\tau_i,\epsilon_i}(x_i^\star,x_{-i}^\star)
        =
        \int_{\mathbb{R}^n}V_i\diff(x_i^\star-G_i)
        +\epsilon_i\left(H(x_i^\star)-H(G_i)\right)
        =
        \epsilon_i\KL(x_i^\star,G_i).
    \end{equation*}
    Since $x_i^\star$ is a best response to $x_{-i}^\star$, the left-hand side is nonpositive, so $\KL(x_i^\star,G_i)=0$, i.e., $x_i^\star=G_i=x_i(m_i^\star)$. As $i$ was arbitrary, $(x_1^\star,\ldots,x_N^\star)=x(m^\star)$, and, by~(i), $m^\star$ is a Nash equilibrium of the surrogate game.
\end{proof}

At this point, we can study RQE for games in which
\begin{equation*}
    R(a_1,\ldots,a_N)
    =
    -\frac{1}{2}\left\Vert\sum_{i=1}^Na_i- \bar a\right\Vert_H^2
    =
    -\frac{1}{2}
    \left\langle \sum_{i=1}^Na_i- \bar a, H\left(\sum_{i=1}^Na_i- \bar a\right)\right\rangle,\
    c_i(a_i)=\frac{\rho_i}{2}\norm{a_i}^2,
\end{equation*}
where we assume that the matrix $H$, used in the weighted norm, is symmetric positive definite and $\rho_i>0$.
As we will show below, considering this class is sufficient.

\begin{proposition}[computation of RQE]\label{proposition:continuous_game:rqe}
    Let $\epsilon_i>0$, let $\Sigma_i\coloneqq\epsilon_i(\rho_iI+H)^{-1}\succ0$, and let $\tau_i>0$ be such that
    \begin{equation}\label{proposition:continuous_game:rqe:assumption}
        \frac{1}{\tau_i}\Big(\sum_{j\neq i}\Sigma_j\Big)^{-1}-H\succ0
        \qquad\forall i\in[N],
    \end{equation}
    or, equivalently, $\tau_i\sum_{j\neq i}\Sigma_j\prec H^{-1}$.
    Let $W_i\coloneqq\big(H^{-1}-\tau_i\sum_{j\neq i}\Sigma_j\big)^{-1}\succ0$, and consider the surrogate ``game of the means'', where each player selects $m_i\in\mathbb{R}^n$ to maximize the strongly concave quadratic utility
    \begin{equation}\label{eq:rqe:surrogate}
        \bar U_i(m_i,m_{-i})
        =
        -\frac{1}{2}\left\langle m_i,\left(\rho_iI+W_i\right)m_i\right\rangle
        -
        \left\langle m_i,W_i\Big(\sum_{j\neq i}m_j-\bar a\Big)\right\rangle.
    \end{equation}
    Then, the surrogate game has a unique Nash equilibrium $(m_1^\star,\ldots,m_N^\star)$, given by
    \begin{equation}\label{eq:rqe:closed_form}
        m_i^\star=\frac{1}{\rho_i}W_i\Big(I+\sum_{j=1}^N\frac{1}{\rho_j}W_j\Big)^{-1}\bar a,
        \qquad i\in[N],
    \end{equation}
    and the Gaussian mixed strategies $x_i^\star\coloneqq\mathcal{N}(m_i^\star,\Sigma_i)$ form an RQE of the game at which all players have finite risk-adjusted utility.
    Moreover, $(x_1^\star,\ldots,x_N^\star)$ is the unique RQE at which all players have finite risk-adjusted utility.
\end{proposition}

\begin{proof}
    We derive the statement from~\cref{proposition:continuous_game:rqe:proof}.
    Fix a player $i$, let $S_i\coloneqq\sum_{j\neq i}\Sigma_j$, and let $E\coloneqq\mathbf{1}_{N-1}\otimes I\in\mathbb{R}^{n(N-1)\times n}$, where $\mathbf{1}_{N-1}\in\mathbb{R}^{N-1}$ is the vector of all ones, so that $E^\top a_{-i}=\sum_{j\neq i}a_j$.

    \paragraph{Step 1: Quadratic form of the utilities.}
    Expanding the square in $u_i(a_i,a_{-i})=-\frac{1}{2}\Vert a_i+E^\top a_{-i}-\bar a\Vert_H^2-\frac{\rho_i}{2}\Vert a_i\Vert^2$ shows that, up to the additive constant $-\frac{1}{2}\Vert\bar a\Vert_H^2$, the utilities are of the form~\eqref{proposition:continuous_game:rqe:proof:general_payoff} with
    \begin{equation*}
        H_{i,i}=-(\rho_iI+H),
        \quad
        H_{i,-i}=-HE^\top,
        \quad
        H_{-i,-i}=-EHE^\top,
        \quad
        h_i=H\bar a,
        \quad
        h_{-i}=EH\bar a.
    \end{equation*}
    In particular, every $n\times n$ block of $H_{-i,-i}$ equals $-H$, since the opponents are coupled through the aggregate $\sum_{j\neq i}a_j$.
    The additive constant shifts $U_i^{\tau_i,\epsilon_i}$ by the same constant and thus affects neither best responses nor the finiteness of the risk-adjusted utilities.
    Moreover, $H_{i,i}\prec0$, the block matrix in~\eqref{proposition:continuous_game:rqe:proof:general_payoff} equals $-\left[\begin{smallmatrix}I\\E\end{smallmatrix}\right]H\left[\begin{smallmatrix}I&E^\top\end{smallmatrix}\right]-\left[\begin{smallmatrix}\rho_iI&0\\0&0\end{smallmatrix}\right]\preceq0$, and $-\epsilon_iH_{i,i}^{-1}=\Sigma_i$.

    \paragraph{Step 2: Positive definiteness of $P_{-i}$.}
    Let $\Sigma_{-i}$ and $P_{-i}=\frac{1}{\tau_i}\Sigma_{-i}^{-1}-EHE^\top$ be as in~\cref{proposition:continuous_game:rqe:proof}, and note that $E^\top\Sigma_{-i}E=S_i$.
    With $K\coloneqq\Sigma_{-i}^{1/2}EH^{1/2}$, we have $P_{-i}=\frac{1}{\tau_i}\Sigma_{-i}^{-1/2}(I-\tau_iKK^\top)\Sigma_{-i}^{-1/2}$.
    Since $KK^\top$ and $K^\top K=H^{1/2}S_iH^{1/2}$ have the same nonzero eigenvalues, $P_{-i}\succ0$ if and only if $\tau_iH^{1/2}S_iH^{1/2}\prec I$, i.e., $\tau_iS_i\prec H^{-1}$, which is~\eqref{proposition:continuous_game:rqe:assumption}.

    \paragraph{Step 3: Identification of the surrogate game.}
    By the Woodbury identity and since $E^\top\Sigma_{-i}E=S_i$ and $W_i^{-1}=H^{-1}-\tau_iS_i$, we have $P_{-i}^{-1}=\tau_i\Sigma_{-i}+\tau_i^2\Sigma_{-i}EW_iE^\top\Sigma_{-i}$, and thus
    \begin{equation*}
        \frac{1}{\tau_i}E^\top P_{-i}^{-1}\Sigma_{-i}^{-1}
        =
        (I+\tau_iS_iW_i)E^\top
        =
        H^{-1}W_iE^\top.
    \end{equation*}
    Multiplying from the left by $H$ and from the right by $\tau_i\Sigma_{-i}EH$ gives $HE^\top P_{-i}^{-1}EH=\tau_iW_iS_iH=W_i-H$.
    Therefore, the coefficients of the surrogate utility~\eqref{eq:utility_game_means} are
    \begin{align*}
        H_{i,i}-H_{i,-i}P_{-i}^{-1}H_{i,-i}^\top
        &=
        -(\rho_iI+H)-(W_i-H)
        =
        -(\rho_iI+W_i),
        \\
        \frac{1}{\tau_i}H_{i,-i}P_{-i}^{-1}\Sigma_{-i}^{-1}m_{-i}
        &=
        -W_iE^\top m_{-i}
        =
        -W_i\sum_{j\neq i}m_j,
        \\
        h_i-H_{i,-i}P_{-i}^{-1}h_{-i}
        &=
        H\bar a+(W_i-H)\bar a
        =
        W_i\bar a,
    \end{align*}
    so that~\eqref{eq:utility_game_means} coincides with~\eqref{eq:rqe:surrogate}.

    \paragraph{Step 4: Closed form.}
    Since $\rho_iI+W_i\succ0$, each $\bar U_i(\cdot,m_{-i})$ is strongly concave, so $(m_1,\ldots,m_N)$ is a Nash equilibrium of the surrogate game if and only if $(\rho_iI+W_i)m_i+W_i(\sum_{j\neq i}m_j-\bar a)=0$ for all $i$, i.e.,
    \begin{equation*}
        \rho_im_i+W_iv=0
        \quad\forall i\in[N],
        \qquad\text{where }
        v\coloneqq\sum_{j=1}^Nm_j-\bar a.
    \end{equation*}
    Summing $m_i=-\frac{1}{\rho_i}W_iv$ over $i$ gives $(I+\sum_{j}\frac{1}{\rho_j}W_j)v=-\bar a$. Since $W_j\succ0$ and $\rho_j>0$, the matrix $I+\sum_{j}\frac{1}{\rho_j}W_j$ is positive definite and hence invertible, so $v$, and with it each $m_i=-\frac{1}{\rho_i}W_iv$, is uniquely determined, which yields~\eqref{eq:rqe:closed_form}.
    Conversely, \eqref{eq:rqe:closed_form} satisfies these conditions, since it gives $\sum_jm_j-\bar a=v$ for $v\coloneqq-(I+\sum_{j}\frac{1}{\rho_j}W_j)^{-1}\bar a$.

    \paragraph{Step 5: Conclusion.}
    By Steps~1 and~2, \cref{proposition:continuous_game:rqe:proof} applies, and, by Step~3, its surrogate game is the one in~\eqref{eq:rqe:surrogate}, whose unique Nash equilibrium is~\eqref{eq:rqe:closed_form} by Step~4.
    Hence, by~\cref{proposition:continuous_game:rqe:proof}(i), $(x_1^\star,\ldots,x_N^\star)$ is an RQE.
    At this RQE, all players have finite risk-adjusted utility: since $P_{-i}\succ0$, \cref{lemma:best_response} gives $\varphi_i<+\infty$ on $\mathbb{R}^n$, and $x_i^\star$ is Gaussian, so that it has finite second moment and finite entropy.
    Finally, by~\cref{proposition:continuous_game:rqe:proof}(ii), every RQE at which all players have finite risk-adjusted utility is of the form $(\mathcal{N}(m_1,\Sigma_1),\ldots,\mathcal{N}(m_N,\Sigma_N))$ for some Nash equilibrium $(m_1,\ldots,m_N)$ of the surrogate game, i.e., by Step~4, for $(m_1,\ldots,m_N)=(m_1^\star,\ldots,m_N^\star)$. Thus, $(x_1^\star,\ldots,x_N^\star)$ is the unique such RQE.
\end{proof}

\paragraph{Discussion}
We make two observations on this result.
First, the equilibrium is available in closed form for any number of players and for player-dependent parameters $\rho_i$, $\tau_i$, and $\epsilon_i$; moreover, the covariances $\Sigma_i$ do not depend on the degrees of risk aversion.
Second, uniqueness holds among all RQE at which all players have finite risk-adjusted utility. This qualification cannot be dropped: if all players play Gaussian strategies with sufficiently large covariances, then, by~\cref{lemma:best_response}, every player has infinite risk regardless of their own mixed strategy, so that any such profile is trivially a (degenerate) RQE.

\subsubsection{Proof of~\cref{thm:rqe_cooperation}}

We prove a more general version of~\cref{thm:rqe_cooperation}, valid for player-dependent degrees of risk aversion $\tau_i$, bounded rationality $\epsilon_i$, and effort costs $\rho_i$.

\begin{theorem}[risk induces collaboration]\label{proposition:continuous_game:rqe_cooperation:general}
    Consider the game of~\cref{subsec:risk_collaboration} and let $\epsilon_i>0$ for all $i$.
    Let $\Sigma_i\coloneqq\epsilon_i(\rho_iI-H)^{-1}\succ0$, and define the open box of admissible degrees of risk aversion
    \begin{equation}\label{eq:rqe_cooperation:domain}
        \mathcal{T}
        \coloneqq
        (0,\bar\tau_1)\times\cdots\times(0,\bar\tau_N),
        \qquad
        \bar\tau_i\coloneqq\frac{1}{\lambda_{\max}\big(-H\sum_{j\neq i}\Sigma_j\big)}.
    \end{equation}
    For $\tau=(\tau_1,\ldots,\tau_N)\in\mathcal{T}$, let $(x_1^\star(\tau),\ldots,x_N^\star(\tau))$ be the unique RQE of the game at which all players have finite risk-adjusted utility, which exists and is Gaussian by~\cref{proposition:continuous_game:rqe}.
    Then, the expected shared reward
    \begin{equation*}
        J(\tau)
        \coloneqq
        \mathbb{E}_{a_i\sim x_i^\star(\tau)}
        \left[
            R(a_1,\ldots,a_N)
        \right]
    \end{equation*}
    is increasing on $\mathcal{T}$: if $\tau,\tau'\in\mathcal{T}$ with $\tau_i'\geq\tau_i$ for all $i\in[N]$, then $J(\tau')\geq J(\tau)$, with strict inequality whenever $\tau'\neq\tau$.
    That is, players contribute more to the shared reward as they become more risk-averse.
\end{theorem}

\begin{proof}
    \textbf{Step 1: Reduction.}
    Let $\bar H\coloneqq-H\succ0$ and $\bar a\coloneqq-H^{-1}h$. Completing the square in the shared reward of~\cref{subsec:risk_collaboration} gives
    \begin{equation*}
        R(a_1,\ldots,a_N)
        =
        -\frac{1}{2}\Big\Vert\sum\nolimits_{i=1}^Na_i-\bar a\Big\Vert_{\bar H}^2
        -\frac{1}{2}\langle h,H^{-1}h\rangle.
    \end{equation*}
    Thus, up to the additive constant $-\frac{1}{2}\langle h,H^{-1}h\rangle$, the utilities are those of the game studied in~\cref{proposition:continuous_game:rqe}, with $\bar H$ in place of $H$.
    This constant shifts every $U_i^{\tau_i,\epsilon_i}$ by the same constant and thus affects neither best responses nor the finiteness of the risk-adjusted utilities; hence, the two games have the same RQE at which all players have finite risk-adjusted utility.
    Moreover, $\Sigma_i=\epsilon_i(\rho_iI+\bar H)^{-1}$ is the covariance in~\cref{proposition:continuous_game:rqe}.

    \paragraph{Step 2: Diagonalization.}
    Let $\bar H=\sum_{k=1}^n\lambda_ke_ke_k^\top$ be a spectral decomposition of $\bar H$, with $\lambda_k>0$ and $e_1,\ldots,e_n$ orthonormal. For $y\in\mathbb{R}^n$, let $y_k\coloneqq\langle e_k,y\rangle$, so that $\Vert y\Vert_{\bar H}^2=\sum_{k=1}^n\lambda_ky_k^2$.
    Linear combinations and inverses of matrices that are diagonal in the basis $e_1,\ldots,e_n$ are again diagonal in this basis, with the corresponding eigenvalues.
    Hence, $\Sigma_i$ and $S_i\coloneqq\sum_{j\neq i}\Sigma_j$ are diagonal in this basis, with eigenvalues $\frac{\epsilon_i}{\rho_i+\lambda_k}$ and
    \begin{equation*}
        s_{i,k}\coloneqq\sum_{j\neq i}\frac{\epsilon_j}{\rho_j+\lambda_k}>0,
    \end{equation*}
    respectively.
    In particular, $-H\sum_{j\neq i}\Sigma_j=\bar HS_i$ has eigenvalues $\lambda_ks_{i,k}>0$, so that $\bar\tau_i=(\max_k\lambda_ks_{i,k})^{-1}\in(0,+\infty)$ and, for $\tau_i>0$,
    \begin{equation*}
        \tau_i<\bar\tau_i
        \quad\iff\quad
        \tau_is_{i,k}<\frac{1}{\lambda_k}\ \ \forall k\in[n]
        \quad\iff\quad
        \tau_iS_i\prec\bar H^{-1}.
    \end{equation*}
    Therefore, $\tau\in\mathcal{T}$ if and only if $\tau_i>0$ for all $i$ and assumption~\eqref{proposition:continuous_game:rqe:assumption} of~\cref{proposition:continuous_game:rqe} holds with $\bar H$ in place of $H$.

    \paragraph{Step 3: Shared reward at equilibrium.}
    Fix $\tau\in\mathcal{T}$. By Steps~1 and~2, \cref{proposition:continuous_game:rqe} applies, so that $x_i^\star(\tau)=\mathcal{N}(m_i^\star(\tau),\Sigma_i)$, where $\Sigma_i$ does not depend on $\tau$ and $m_i^\star(\tau)$ is given by~\eqref{eq:rqe:closed_form} with $W_i=W_i(\tau_i)\coloneqq(\bar H^{-1}-\tau_iS_i)^{-1}$.
    By Step~2, $W_i(\tau_i)$ and $M(\tau)\coloneqq I+\sum_{j=1}^N\frac{1}{\rho_j}W_j(\tau_j)$ are diagonal in the basis $e_1,\ldots,e_n$, with eigenvalues
    \begin{equation*}
        w_{i,k}(\tau_i)\coloneqq\Big(\frac{1}{\lambda_k}-\tau_is_{i,k}\Big)^{-1}>0
        \qquad\text{and}\qquad
        \mu_k(\tau)\coloneqq1+\sum_{j=1}^N\frac{w_{j,k}(\tau_j)}{\rho_j}>1,
    \end{equation*}
    respectively.
    Hence, by~\eqref{eq:rqe:closed_form}, $\big(m_i^\star(\tau)\big)_k=\frac{w_{i,k}(\tau_i)}{\rho_i\mu_k(\tau)}\bar a_k$, and, by definition of $\mu_k(\tau)$,
    \begin{equation*}
        \Big(\sum\nolimits_{i=1}^Nm_i^\star(\tau)-\bar a\Big)_k
        =
        \frac{\mu_k(\tau)-1}{\mu_k(\tau)}\,\bar a_k-\bar a_k
        =
        -\frac{\bar a_k}{\mu_k(\tau)}.
    \end{equation*}
    Since $a_1,\ldots,a_N$ are independent in the definition of $J$, the vector $\sum_{i=1}^Na_i-\bar a$ is Gaussian with mean $\sum_{i=1}^Nm_i^\star(\tau)-\bar a$ and covariance $\sum_{i=1}^N\Sigma_i$.
    Using $\mathbb{E}[\Vert X\Vert_{\bar H}^2]=\Vert\mathbb{E}[X]\Vert_{\bar H}^2+\operatorname{tr}(\bar H\operatorname{Cov}(X))$ for every random vector $X$ with finite second moment, we obtain
    \begin{equation*}
        J(\tau)
        =
        -\frac{1}{2}\sum_{k=1}^n\frac{\lambda_k\bar a_k^2}{\mu_k(\tau)^2}+C,
        \qquad
        C\coloneqq-\frac{1}{2}\operatorname{tr}\Big(\bar H\sum\nolimits_{i=1}^N\Sigma_i\Big)-\frac{1}{2}\langle h,H^{-1}h\rangle,
    \end{equation*}
    where $C$ does not depend on $\tau$.

    \paragraph{Step 4: Monotonicity.}
    Let $\tau,\tau'\in\mathcal{T}$ with $\tau_i'\geq\tau_i$ for all $i$.
    Since $s_{i,k}>0$, each $w_{i,k}$ is strictly increasing on $(0,\bar\tau_i)$. Hence, $\mu_k(\tau')\geq\mu_k(\tau)>0$ for all $k$, with strict inequality for all $k$ if $\tau'\neq\tau$. Therefore,
    \begin{equation*}
        J(\tau')-J(\tau)
        =
        \frac{1}{2}\sum_{k=1}^n\lambda_k\bar a_k^2\left(\frac{1}{\mu_k(\tau)^2}-\frac{1}{\mu_k(\tau')^2}\right)
        \geq0.
    \end{equation*}
    If, moreover, $\tau'\neq\tau$, then the inequality is strict, since $h\neq0$ gives $\bar a=-H^{-1}h\neq0$, so that $\bar a_k\neq0$ for some $k$.
\end{proof}

The proof of~\cref{thm:rqe_cooperation} is now a direct consequence:

\begin{proof}[Proof of~\cref{thm:rqe_cooperation}]
    The proof follows from~\cref{proposition:continuous_game:rqe_cooperation:general} and \cref{proposition:continuous_game:rqe}, using $\rho_i=\rho$, $\tau_i=\tau$, and $\epsilon_i=\epsilon$ for all agents.
\end{proof}

To conclude, we provide more details on~\cref{example:continuous_game}:

\begin{example}[Details on~\cref{example:continuous_game}]
    For the game in~\cref{example:continuous_game}, we have $H=1$ and $\rho_i=1$ in the notation of~\cref{proposition:continuous_game:rqe}, so that $\Sigma_i=\frac{\epsilon_i}{2}$ and $W_i=\big(1-\frac{\tau_i\epsilon_j}{2}\big)^{-1}=\frac{2}{2-\tau_i\epsilon_j}$, with $j\neq i$, so that
    \begin{equation*}
        \bar U_i(m_i,m_{-i})
        =
        -\frac{1}{2}\left(1+\frac{2}{2-\tau_i\epsilon_j}\right)m_i^2
        -\frac{2}{2-\tau_i\epsilon_j}m_i\left(m_{-i}-\bar a\right).
    \end{equation*}
    With $\tau_1=\tau_2=\tau$ and $\epsilon_1=\epsilon_2=\epsilon$ (so that~\eqref{proposition:continuous_game:rqe:assumption} reads $\tau\epsilon<2$), \eqref{eq:rqe:closed_form} gives $m_i^\star=\frac{\bar a}{3-\tau\epsilon/2}$, and we obtain the expected utility of each player from the expected shared reward $J(\tau)$ minus the expected personal cost and it can be shown to be
    \begin{equation*}
        -\frac{\bar a^2}{2}\frac{2-\tau\epsilon+(\frac{\tau\epsilon}{2})^2}{(3-\frac{\tau\epsilon}{2})^2}-\frac{3\epsilon}{4},
    \end{equation*}
    where the constant $-\frac{3\epsilon}{4}$ results from the variances $\Sigma_i=\frac{\epsilon}{2}$ of the mixed strategies and is therefore not dependent on $\tau$ but only on $\epsilon$.
\end{example}

\subsection{Proof of~\cref{thm:free_riding}}\label{app:free_riding_proof}
In this section, we provide the proof of~\cref{thm:free_riding}. To do so, we first derive several helper-lemmas which guarantee that (i) players' strategies in an RQE are on the interior of the simplex, (ii) the KL-divergence between any strategy and an RQE strategy is uniformly bounded, and (iii) we can relate the degree of free-riding to a \emph{lower bound} on the distance between equilibrium strategies.
To begin, we define some useful quantities that capture measures of the spread of players' costs and reward functions $c_\mathrm{max}\coloneqq\max_a c(a)$, $c_\mathrm{min}\coloneqq\min_a c(a)$,  $v_\mathrm{min}\coloneqq\min_{a_1,a_2} R(a_1,a_2)-c_\mathrm{max}$, $v_\mathrm{max}\coloneqq\max_{a_1,a_2} R(a_1,a_2)-c_\mathrm{min}$, and $\bar v\coloneqq v_\mathrm{max}-v_\mathrm{min}$.
Without loss of generality, we assume $c_\mathrm{max}>c_\mathrm{min}$; otherwise, both players always pay the same cost, so that no RQE exhibits free-riding of any degree $\delta>0$ and~\cref{thm:free_riding_app} holds trivially.
Throughout the proof, we let $n\coloneqq |\mathcal{A}|$ be the number of actions and $\Delta_n=\Delta(\mathcal A)$ be the probability simplex. For a mixed strategy $x\in\Delta_n$, we denote by $x(a)\in[0,1]$ the probability mass assigned to the action $a$.

\begin{lemma}\label{lemma:ent_bnd}
    Suppose that players have degree of entropy regularization $\epsilon$. At an RQE, a player's strategy $x_i^*$ has support on all actions, with probability at least $m=\frac{\exp(-\bar v/\epsilon)}{n}$.
\end{lemma}
\begin{proof}
    In an RQE, a player's strategy $x_i^*$ must satisfy
    \begin{equation*}
        x_i^*=\argmax_{x_i \in \Delta_n} \ \  \langle x_i,  R p^*\rangle  - \langle c, x_i\rangle -\epsilon H(x_i),
    \end{equation*}
    where $p^*=\argmin_{p \in \Delta_n}  \langle x_i^*, R p \rangle +\frac{1}{\tau} \KL(p,x_{-i}^*)$.
    Indeed, since $p^*$ is unique by strict convexity of $\KL(\cdot,x_{-i}^*)$, by Danskin's theorem the risk-adjusted utility $U^{\tau}(\cdot,x_{-i}^*)$ has gradient $Rp^*-c$ at $x_i^*$, so $x_i^*$ satisfies the first-order optimality conditions of this strictly concave problem.
    Letting $v\coloneqq Rp^* -c$, the classic results show that $x_i^*$ is a Boltzmann distribution of the form:
    \[ x_i^*(a)=\frac{\exp(v(a)/\epsilon)}{\sum_{a'\in\mathcal{A}} \exp(v(a')/\epsilon)},\]
    where $x_i^*(a)$ is the probability mass assigned to action $a$.
    By minimizing the numerator and maximizing the denominator of the right hand side we find that:
    \[ x_i^*(a)\ge \frac{\exp(-\bar v/\epsilon)}{n}, \]
    which completes the proof.
\end{proof}

\begin{lemma}\label{lemma:kl_bnd}
Suppose $x_2 \in \sigma(\Delta_n)$ where $\sigma(\Delta_n)\coloneqq\{x \in \Delta_n : x(a) \ge m \text{ for all } a\in\mathcal A\}$ where  $m \in (0, \frac{1}{n}]$. Then the KL divergence between any distribution $q\in \Delta_n$ and $x_2$ is bounded by
\begin{equation*}
    \sup_{q \in \Delta_n} \KL(q,x_2)\le \log(\frac{1}{m}).
\end{equation*}
\end{lemma}
\begin{proof}
We directly have
\begin{equation*}
    \KL(q,x_2) = \sum_{a\in\mathcal A} q(a)\log(\frac{q(a)}{x_2(a)}) \le \sum_a q(a)\log(\frac{1}{m}) = \log(\frac{1}{m})
\end{equation*}
since $x_2(a) \ge m$.
\end{proof}

\begin{lemma}
\label{lemma:cost_bnd}
Let $x_1, x_2 \in \Delta_n$.
If $|\langle c, x_1 - x_2\rangle| \ge  \delta$, then
\[
\|x_1 - x_2\|_1 \ge \frac{2\delta}{c_{\max} - c_{\min}}.
\]
\end{lemma}
\begin{proof}
Let $d\coloneqq x_1 - x_2$. Since $x_1, x_2 \in \Delta_n$, $\sum_a d(a) = 0$. We use the fact that, for any $d \in \mathbb{R}^{n}$ with $\sum_a d(a) = 0$, $|\langle c, d \rangle| \le \frac{c_{\max} - c_{\min}}{2} \|d\|_1$ (indeed, with $\bar c\coloneqq(c_{\max}+c_{\min})/2$ and $\mathbf{1}$ the vector of all ones, $|\langle c,d\rangle|=|\langle c-\bar c\mathbf{1},d\rangle|\le\|c-\bar c\mathbf{1}\|_\infty\|d\|_1=\frac{c_{\max}-c_{\min}}{2}\|d\|_1$). This directly yields
\begin{equation*}
    \delta \le |\langle c, d \rangle| \le \frac{c_{\max} - c_{\min}}{2} \|x_1 - x_2\|_1.
\end{equation*}
 Rearranging yields the lower bound.
\end{proof}

We can now prove~\cref{thm:free_riding}, restated below with all constants made explicit.
\begin{theorem}\label{thm:free_riding_app}
    Consider a two-player game as in~\cref{subsec:risk_free_riding}, let $\delta>0$, and let both players have entropy regularization $\epsilon>0$ and risk aversion $\tau>0$. If
    \begin{equation*}
        \tau >\frac{2 (\epsilon \log{n}+\bar v)(c_\mathrm{max}-c_\mathrm{min})^2}{\epsilon^2 \delta^2},
    \end{equation*}
    then the game admits no RQE with degree of free-riding greater than $\delta$.

\end{theorem}
\begin{proof}
We prove this by contradiction. We assume that $\tau> \frac{2 (\epsilon \log{n}+\bar v)(c_\mathrm{max}-c_\mathrm{min})^2}{\epsilon^2 \delta^2}$ and that $x_1$ and $x_2$ constitute a resulting RQE with a level of free-riding larger than $\delta$; i.e.,  $|\langle c,x_1-x_2\rangle| \ge \delta$ .
Define the worst-case utility $W(x)$ and the robust-regularized objective $\Phi(x)$:
\begin{align*}
    U(x, q) &\coloneqq\langle x, R q\rangle - \langle c,x \rangle\\
    W(x) &\coloneqq\inf_{q \in \Delta_n} U(x, q)\\
    \Phi(x) &\coloneqq W(x) - \epsilon H(x)
\end{align*}
Since $W$ is concave (as the infimum of affine functions) and $H$ is $1$-strongly convex on $\Delta_n$ with respect to the $\ell_1$ norm (by Pinsker's inequality), $\Phi$ is $\epsilon$-strongly concave on $\Delta_n$ with respect to the $\ell_1$ norm. In particular, $\Phi$ has a unique maximizer $x^\star\coloneqq\argmax_{x\in\Delta_n}\Phi(x)$.

Let us consider player $1$'s risk-adjusted utility. Without loss of generality, we develop the following for $x_1$; since $R$ is symmetric and both players share the cost $c$, player 2's objective has the same form, so the same arguments apply to $x_2$ with the same $\Phi$ and, hence, the same maximizer $x^\star$. We first show that we can lower bound the risk-adjusted objective by our worst case objective for any $x_1,x_2 \in \Delta_n$:
\begin{equation*}
    U^\tau(x_1,x_2) = \min_{q \in \Delta_n} U(x_1, q) + \frac{1}{\tau}\KL(q,x_2) \ge \min_{q \in \Delta_n} U(x_1, q) = W(x_1).
\end{equation*}
Now consider an RQE made up of strategies $(x_1,x_2)$. We upper bound the risk-adjusted objective by using the fact that $x_1$ and $x_2$ are on the interior of the simplex since they are entropy-regularized best responses (\cref{lemma:ent_bnd}), and thus the $\KL$ term is bounded by~\cref{lemma:kl_bnd}.
Letting $K\coloneqq\log{n}+\frac{\bar v}{\epsilon}$ be the resulting upper bound (which depends solely on $R,c,\epsilon$ and the dimension $n$) we conclude that
\begin{equation*}
    U^\tau(x_1,x_2) = \min_{q\in\Delta_n} U(x_1, q) + \frac{1}{\tau}\KL(q,x_2) \le W(x_1) +\frac{K}{\tau}.
\end{equation*}
Putting these together, we have
\[
\Phi(x_1)=W(x_1)- \epsilon H(x_1) \le U^\tau(x_1, x_2) - \epsilon H(x_1) \le W(x_1) +\frac{K}{\tau} - \epsilon H(x_1) =\Phi(x_1) + \frac{K}{\tau}.
\]
Using the definition of an RQE, we can further develop this to find that, for all $x \in \Delta_n$:
\[
\Phi(x_1)+\frac{K}{\tau} \ge U^\tau(x_1, x_2) - \epsilon H(x_1) \ge U^\tau(x, x_2) - \epsilon H(x) \ge \Phi(x)
\]
Moreover, since $\Phi$ is $\epsilon$-strongly concave and $x^\star$ maximizes it over $\Delta_n$, we have
\begin{equation*}
    \Phi(x_1)\leq \Phi(x^\star)-\frac{\epsilon }{2}\|x_1-x^\star\|_1^2
\end{equation*}
Thus,
\begin{equation*}
    \frac{\epsilon }{2}\|x_1-x^\star\|_1^2 \le \Phi(x^\star)-\Phi(x_1) \le \frac{K}{\tau} \implies \|x_1-x^\star\|_1 \le \sqrt{\frac{2K}{\epsilon\tau}}.
\end{equation*}
The same bound holds for $x_2$, and thus via the triangle inequality we obtain
\begin{equation*}
   \|x_1-x_2 \|_1\le \|x_1-x^\star\|_1+\|x_2-x^\star\|_1 \le2\sqrt{\frac{2K}{\epsilon\tau}}.
\end{equation*}
Via~\cref{lemma:cost_bnd}, however we have that
\begin{equation*}
    \frac{2\delta}{c_\mathrm{max}-c_\mathrm{min}}\le \|x_1-x_2 \|_1\le \|x_1-x^\star\|_1+\|x_2-x^\star\|_1 \le2\sqrt{\frac{2K}{\epsilon\tau}}.
\end{equation*}
Re-arranging, we find that
\[ \tau\le \frac{2K(c_\mathrm{max}-c_\mathrm{min})^2}{\epsilon\delta^2}.\]
Plugging in our value for $K=\log{n}+\frac{\bar v}{\epsilon}$, this implies that
\[ \tau\le \frac{2 (\epsilon \log{n}+\bar v)(c_\mathrm{max}-c_\mathrm{min})^2}{\epsilon^2 \delta^2}.\]
However, we assumed that the RQE resulted from the use of a $\tau>\frac{2 (\epsilon \log{n}+\bar v)(c_\mathrm{max}-c_\mathrm{min})^2}{\epsilon^2 \delta^2}$ which is a contradiction.
\end{proof}

\section{Convergence Guarantee for the Meta Algorithm}\label{app:convergence}
In this section, we establish a convergence guarantee for the projected policy-gradient instantiation of our meta algorithm in the two-player finite matrix game setting (\cref{thm:matrix_game_srpo_convergence}). The argument proceeds in three stages: \cref{app:convergence:setup} introduces the lifted four-player regularized game and the operator $G$ whose variational inequality on $\mathcal Z_\delta$ characterizes its RQE; \cref{app:convergence:lemmas} establishes an a priori interior bound on the RQE and two properties of $G$ (Lipschitz continuity and strong monotonicity) that drive the analysis; \cref{app:convergence:theorem} combines these properties into a linear-convergence guarantee for the projected gradient iteration.
\subsection{Lifted game and operator formulation}\label{app:convergence:setup}

For each player $i\in\{1,2\}$, let $R_i$ denote the payoff matrix of the original two-player matrix game. We work with the lifted state
\[
z=(x_1,x_2,p_1,p_2)\in \mathcal Z,
\qquad
\mathcal Z
:=
\Delta(\mathcal A_1)\times \Delta(\mathcal A_2)\times
\Delta(\mathcal A_2)\times \Delta(\mathcal A_1),
\]
in which $x_i$ is the mixed strategy of original player $i$ and $p_i$ is the mixed strategy of the adversary against player $i$. The original-player objective is
\begin{equation*}
J_i(x_i,x_{-i},p_i)
=
x_i^\top R_i p_i
+\frac{1}{\tau_i}\KL(p_i,x_{-i})
-\epsilon_i H(x_i),
\end{equation*}
and the corresponding adversary objective is
\begin{equation*}
\bar J_i(x_i,x_{-i},p_i)
=
-x_i^\top R_i p_i
-\frac{1}{\tau_i}\KL(p_i,x_{-i}).
\end{equation*}
For any \(\delta>0\), define the \(\delta\)-interior set
\begin{equation*}
\mathcal Z_\delta
:=
\left\{
z=(x_1,x_2,p_1,p_2)\in\mathcal Z\;:\;
\begin{aligned}
&x_i(a)\ge\delta\quad\forall a\in\mathcal A_i,\ i\in\{1,2\},\\
&p_i(b)\ge\delta\quad\forall b\in\mathcal A_{-i},\ i\in\{1,2\}
\end{aligned}
\right\}.
\end{equation*}
Thus $\mathcal Z$ is the ambient product of simplices, while $\mathcal Z_\delta$ is the closed convex feasible set used throughout the convergence argument. We define the lifted game operator $G:\mathcal Z_\delta\to\mathbb R^d$ by
\begin{equation}\label{eq:G-definition}
G(z)
:=
\Bigl(
-\nabla_{x_1}J_1(z),\,-\nabla_{x_2}J_2(z),\,
-\nabla_{p_1}\bar J_1(z),\,-\nabla_{p_2}\bar J_2(z)
\Bigr),
\end{equation}
where \(d\) is the total dimension of \(z\).

Throughout, by an RQE of the lifted game we mean a Nash equilibrium $z^\star=(x_1^\star,x_2^\star,p_1^\star,p_2^\star)$: each $x_i^\star$ maximizes $J_i(\cdot,x_{-i}^\star,p_i^\star)$ over $\Delta(\mathcal A_i)$ and each $p_i^\star$ maximizes $\bar J_i(x_i^\star,x_{-i}^\star,\cdot)$ over $\Delta(\mathcal A_{-i})$. Since, for fixed $x_{-i}$, $J_i$ is strictly concave in $x_i$ and strictly convex in $p_i$, the minimax theorem shows that this holds if and only if $(x_1^\star,x_2^\star)$ is an RQE of the original game (\cref{def:RQE_normal_form}) and $p_i^\star$ attains the infimum in~\eqref{eq:risk_adjusted_objective} at $(x_i^\star,x_{-i}^\star)$~\citep[Proposition~1]{mazumdar2024tractableequilibriumcomputationmarkov}.

\subsection{Quantitative properties of the operator $G$}\label{app:convergence:lemmas}

The convergence proof rests on an a priori lower bound on the coordinates of any RQE (\cref{lem:matrix_game_policies_bounded_below}) and two quantitative properties of $G$: Lipschitz continuity on $\mathcal Z_\delta$ (\cref{lem:matrix_game_operator_lipschitz}) and strong monotonicity on $\mathcal Z_\delta$ (\cref{lem:matrix_game_operator_strong_monotone}).
\begin{lemma}[Interior bound on the RQE]\label{lem:matrix_game_policies_bounded_below}
Consider the lifted four-player game above for a two-player matrix game with bounded payoff matrices $R_1$ and $R_2$, and let
\[
z^\star=(x_1^\star,x_2^\star,p_1^\star,p_2^\star)
\]
be any RQE. For each $i\in\{1,2\}$, define the payoff-span constant
\[
S_i
:=
\max_{\substack{a\in\mathcal A_i\\ b\in\mathcal A_{-i}}} R_i(a,b)
-
\min_{\substack{a\in\mathcal A_i\\ b\in\mathcal A_{-i}}} R_i(a,b).
\]
Then every coordinate of every block of $z^\star$ is bounded away from zero. More precisely,
\begin{equation*}
x_i^\star(a)\ge
\underline\delta_{x,i}
:=
\frac{1}{|\mathcal A_i|}\exp\!\left(-\frac{S_i}{\epsilon_i}\right),
\qquad
\forall a\in\mathcal A_i,
\end{equation*}
and
\begin{equation*}
p_i^\star(b)\ge
\underline\delta_{p,i}
:=
\frac{1}{|\mathcal A_{-i}|}
\exp\!\left(-\frac{S_{-i}}{\epsilon_{-i}}-\tau_i S_i\right),
\qquad
\forall b\in\mathcal A_{-i}.
\end{equation*}
Consequently,
\begin{equation}\label{eq:delta-definition}
\underline\delta
:=
\min\{\underline\delta_{x,1},\underline\delta_{x,2},\underline\delta_{p,1},\underline\delta_{p,2}\}
>0
\end{equation}
satisfies $z^\star\in\mathcal Z_{\underline\delta}$.
\end{lemma}

\begin{proof}
For a vector $v$, write $\operatorname{sp}(v):=\max_j v_j-\min_j v_j$. Fix $i\in\{1,2\}$ and define
\[
u_i^\star:=R_i p_i^\star\in\mathbb R^{|\mathcal A_i|}.
\]
Since each coordinate $u_i^\star(a)$ is a convex combination of entries of $R_i$, we have $\operatorname{sp}(u_i^\star)\le S_i$. Now $x_i^\star$ maximizes
\[
x_i^\top R_i p_i^\star-\epsilon_i H(x_i)
\]
over $\Delta(\mathcal A_i)$. The entropy term forces the optimizer into the simplex interior, so the KKT conditions give a multiplier $\alpha_i\in\mathbb R$ such that
\begin{equation*}
(R_i p_i^\star)(a)-\epsilon_i\bigl(\log x_i^\star(a)+1\bigr)-\alpha_i=0,
\qquad
\forall a\in\mathcal A_i.
\end{equation*}
Rearranging and normalizing over $a$ yields
\begin{equation*}
x_i^\star(a)
=
\frac{\exp((R_i p_i^\star)(a)/\epsilon_i)}
{\sum_{a'\in\mathcal A_i}\exp((R_i p_i^\star)(a')/\epsilon_i)}.
\end{equation*}
Therefore,
\begin{align*}
x_i^\star(a)
&\ge
\frac{\exp((R_i p_i^\star)(a)/\epsilon_i)}
{|\mathcal A_i|\exp(\max_{a'}(R_i p_i^\star)(a')/\epsilon_i)} \\
&=
\frac{1}{|\mathcal A_i|}
\exp\!\left(
-\frac{\max_{a'}(R_i p_i^\star)(a')-(R_i p_i^\star)(a)}{\epsilon_i}
\right) \\
&\ge
\frac{1}{|\mathcal A_i|}
\exp\!\left(-\frac{\operatorname{sp}(R_i p_i^\star)}{\epsilon_i}\right)
\ge
\frac{1}{|\mathcal A_i|}
\exp\!\left(-\frac{S_i}{\epsilon_i}\right)
=
\underline\delta_{x,i}.
\end{align*}

Next define $q_i^\star:=R_i^\top x_i^\star\in\mathbb R^{|\mathcal A_{-i}|}$. Again, each coordinate $q_i^\star(b)$ is a convex combination of entries of $R_i$, so $\operatorname{sp}(q_i^\star)\le S_i$. Because $x_{-i}^\star$ is strictly positive by the previous step, the adversary problem
\[
\min_{p_i\in\Delta(\mathcal A_{-i})}
\Bigl\langle q_i^\star,p_i\Bigr\rangle
+
\frac{1}{\tau_i}\KL(p_i,x_{-i}^\star)
\]
has an interior minimizer $p_i^\star$. Its KKT conditions give a multiplier $\beta_i\in\mathbb R$ such that
\begin{equation*}
q_i^\star(b)
+
\frac{1}{\tau_i}\left(\log\frac{p_i^\star(b)}{x_{-i}^\star(b)}+1\right)
-\beta_i
=0,
\qquad
\forall b\in\mathcal A_{-i}.
\end{equation*}
Hence $p_i^\star$ has the Gibbs form
\begin{equation*}
p_i^\star(b)
=
\frac{x_{-i}^\star(b)\exp(-\tau_i q_i^\star(b))}
{\sum_{b'\in\mathcal A_{-i}}x_{-i}^\star(b')\exp(-\tau_i q_i^\star(b'))}.
\end{equation*}
Using $\sum_{b'}x_{-i}^\star(b')=1$, we obtain
\begin{equation*}
\sum_{b'}x_{-i}^\star(b')\exp(-\tau_i q_i^\star(b'))
\le
\exp(-\tau_i \min_{b'} q_i^\star(b')).
\end{equation*}
Therefore, for every $b\in\mathcal A_{-i}$,
\begin{align*}
p_i^\star(b)
&\ge
x_{-i}^\star(b)
\exp\!\left(-\tau_i\bigl(q_i^\star(b)-\min_{b'}q_i^\star(b')\bigr)\right) \\
&\ge
x_{-i}^\star(b)\exp\!\left(-\tau_i\operatorname{sp}(q_i^\star)\right)
\ge
\underline\delta_{x,-i}\exp(-\tau_i S_i)
=
\underline\delta_{p,i}.
\end{align*}
Taking the minimum over the four blockwise lower bounds proves the claim.
\end{proof}

The next two lemmas provide explicit constants for Lipschitz continuity and strong monotonicity of the KL- and negative-entropy-induced operator $G$ on $\mathcal Z_\delta$.

\begin{lemma}[Lipschitz continuity of $G$]\label{lem:matrix_game_operator_lipschitz}
Let $\|\cdot\|_{\mathrm{op}}$ denote the Euclidean operator norm of a matrix. Then $G$ is $L_\delta$-Lipschitz on $\mathcal Z_\delta$ with
\begin{equation*}
L_\delta
:=
\Biggl(
\Bigl(\tfrac{\epsilon_1}{\delta}+\|R_1\|_{\mathrm{op}}+\tfrac{1}{\tau_2\delta}\Bigr)^2
+
\Bigl(\tfrac{\epsilon_2}{\delta}+\|R_2\|_{\mathrm{op}}+\tfrac{1}{\tau_1\delta}\Bigr)^2
+
\Bigl(\|R_1\|_{\mathrm{op}}+\tfrac{1}{\tau_1\delta}\Bigr)^2
+
\Bigl(\|R_2\|_{\mathrm{op}}+\tfrac{1}{\tau_2\delta}\Bigr)^2
\Biggr)^{1/2}.
\end{equation*}
\end{lemma}
\begin{proof}
Let $z=(x_1,x_2,p_1,p_2)$ and $z'=(x_1',x_2',p_1',p_2')$ lie in $\mathcal Z_\delta$. Since every coordinate of $x_i,x_i',p_i,p_i'$ is at least $\delta$, the scalar map $s\mapsto \log s$ has derivative bounded by $\delta^{-1}$ on the relevant interval. Hence, component-wise,
\begin{equation*}
\|\log x_i-\log x_i'\|_2\le \tfrac{1}{\delta}\|x_i-x_i'\|_2,
\qquad
\|\log p_i-\log p_i'\|_2\le \tfrac{1}{\delta}\|p_i-p_i'\|_2.
\end{equation*}
Using
\[
G(z)=
\begin{pmatrix}
-R_1p_1+\epsilon_1(\log x_1+\mathbf 1) \\
-R_2p_2+\epsilon_2(\log x_2+\mathbf 1) \\
R_1^\top x_1+\tau_1^{-1}(\log p_1-\log x_2+\mathbf 1) \\
R_2^\top x_2+\tau_2^{-1}(\log p_2-\log x_1+\mathbf 1)
\end{pmatrix},
\]
we obtain
\begin{align*}
\|G_1(z)-G_1(z')\|_2
&\le
\|R_1\|_{\mathrm{op}}\|p_1-p_1'\|_2
+
\tfrac{\epsilon_1}{\delta}\|x_1-x_1'\|_2, \\
\|G_2(z)-G_2(z')\|_2
&\le
\|R_2\|_{\mathrm{op}}\|p_2-p_2'\|_2
+
\tfrac{\epsilon_2}{\delta}\|x_2-x_2'\|_2, \\
\|G_3(z)-G_3(z')\|_2
&\le
\|R_1\|_{\mathrm{op}}\|x_1-x_1'\|_2
+
\tfrac{1}{\tau_1\delta}\|p_1-p_1'\|_2
+
\tfrac{1}{\tau_1\delta}\|x_2-x_2'\|_2, \\
\|G_4(z)-G_4(z')\|_2
&\le
\|R_2\|_{\mathrm{op}}\|x_2-x_2'\|_2
+
\tfrac{1}{\tau_2\delta}\|p_2-p_2'\|_2
+
\tfrac{1}{\tau_2\delta}\|x_1-x_1'\|_2.
\end{align*}
Summing these four bounds and applying the Cauchy--Schwarz inequality gives $\|G(z)-G(z')\|_2\le L_\delta\|z-z'\|_2$, as claimed.
\end{proof}

\begin{lemma}[Strong monotonicity of $G$]\label{lem:matrix_game_operator_strong_monotone}
Suppose $4\epsilon_1\tau_2\delta>1$ and $4\epsilon_2\tau_1\delta>1$. Define
\begin{equation*}
\mu_{i,\delta}
:=
\frac{1}{2}
\left(
\epsilon_i+\frac{\delta}{\tau_{-i}}
-
\sqrt{\left(\epsilon_i-\frac{\delta}{\tau_{-i}}\right)^2+\frac{1}{\tau_{-i}^2}}
\right),
\qquad i\in\{1,2\},
\end{equation*}
and let $\mu_\delta:=\min\{\mu_{1,\delta},\mu_{2,\delta}\}>0$. Then $G$ is $\mu_\delta$-strongly monotone on $\mathcal Z_\delta$. That is,
\begin{equation*}
    \langle G(z)-G(z'),z-z'\rangle \geq \mu_\delta\|z-z'\|_2^2.
\end{equation*}
\end{lemma}
\begin{proof}
For $z=(x_1,x_2,p_1,p_2)\in \mathcal Z_\delta$, let
\[
D_{x_i}:=\operatorname{diag}(x_i^{-1}),
\qquad
D_{p_i}:=\operatorname{diag}(p_i^{-1}).
\]
From the explicit formula for $G$, its Jacobian is
\begin{equation*}
\nabla G(z)=
\begin{bmatrix}
\epsilon_1D_{x_1} & 0 & -R_1 & 0 \\
0 & \epsilon_2D_{x_2} & 0 & -R_2 \\
R_1^\top & -\tau_1^{-1}D_{x_2} & \tau_1^{-1}D_{p_1} & 0 \\
-\tau_2^{-1}D_{x_1} & R_2^\top & 0 & \tau_2^{-1}D_{p_2}
\end{bmatrix}.
\end{equation*}
Therefore the symmetric part is
\begin{equation*}
\frac{\nabla G(z)+\nabla G(z)^\top}{2}
=
\begin{bmatrix}
\epsilon_1D_{x_1} & 0 & 0 & -(2\tau_2)^{-1}D_{x_1} \\
0 & \epsilon_2D_{x_2} & -(2\tau_1)^{-1}D_{x_2} & 0 \\
0 & -(2\tau_1)^{-1}D_{x_2} & \tau_1^{-1}D_{p_1} & 0 \\
-(2\tau_2)^{-1}D_{x_1} & 0 & 0 & \tau_2^{-1}D_{p_2}
\end{bmatrix}.
\end{equation*}
The payoff blocks cancel from the symmetric part. After permuting the coordinates to $(x_1,p_2,x_2,p_1)$, this matrix becomes block diagonal with blocks
\begin{equation*}
\mathcal M_1(z):=
\begin{bmatrix}
\epsilon_1D_{x_1} & -(2\tau_2)^{-1}D_{x_1} \\
-(2\tau_2)^{-1}D_{x_1} & \tau_2^{-1}D_{p_2}
\end{bmatrix},
\qquad
\mathcal M_2(z):=
\begin{bmatrix}
\epsilon_2D_{x_2} & -(2\tau_1)^{-1}D_{x_2} \\
-(2\tau_1)^{-1}D_{x_2} & \tau_1^{-1}D_{p_1}
\end{bmatrix}.
\end{equation*}
Since these blocks are diagonal across coordinates, it suffices to bound the smallest eigenvalue of the corresponding $2\times 2$ coordinate matrices. For $a\in \mathcal A_1$,
\begin{equation*}
M_{1,a}(z)=
\begin{bmatrix}
\epsilon_1/x_1(a) & -1/(2\tau_2x_1(a)) \\
-1/(2\tau_2x_1(a)) & 1/(\tau_2p_2(a))
\end{bmatrix}
=
\frac{1}{x_1(a)}
\begin{bmatrix}
\epsilon_1 & -1/(2\tau_2) \\
-1/(2\tau_2) & x_1(a)/(\tau_2p_2(a))
\end{bmatrix}.
\end{equation*}
Now $x_1(a)\in[\delta,1]$, $p_2(a)\in[\delta,1]$, so $x_1(a)/p_2(a)\ge \delta$ and $x_1(a)^{-1}\ge 1$. Hence
\begin{equation*}
\lambda_{\min}\bigl(M_{1,a}(z)\bigr)
\ge
\lambda_{\min}
\begin{bmatrix}
\epsilon_1 & -1/(2\tau_2) \\
-1/(2\tau_2) & \delta/\tau_2
\end{bmatrix}
=
\mu_{1,\delta}.
\end{equation*}
The condition $4\epsilon_1\tau_2\delta>1$ is exactly the positivity condition for this $2\times 2$ determinant, so $\mu_{1,\delta}>0$. The same argument gives $\lambda_{\min}\bigl(M_{2,b}(z)\bigr)\ge \mu_{2,\delta}>0$ for every $b\in\mathcal A_2$. Therefore,
\begin{equation*}
\frac{\nabla G(z)+\nabla G(z)^\top}{2}\succeq \mu_\delta I,
\qquad
\forall z\in\mathcal Z_\delta.
\end{equation*}
Now let $z,z'\in\mathcal Z_\delta$, set $v:=z-z'$, and define $\phi(t):=z'+tv$. Since $\mathcal Z_\delta$ is convex, $\phi(t)\in \mathcal Z_\delta$ for all $t\in[0,1]$. Then
\begin{align*}
\langle G(z)-G(z'),\,z-z'\rangle
&=
\int_0^1 v^\top \nabla G(\phi(t)) v\,dt
=
\int_0^1
v^\top
\frac{\nabla G(\phi(t))+\nabla G(\phi(t))^\top}{2}
v\,dt \\
&\ge
\mu_\delta \int_0^1 \|v\|_2^2\,dt
=
\mu_\delta \|z-z'\|_2^2.
\end{align*}
Thus $G$ is $\mu_\delta$-strongly monotone on $\mathcal Z_\delta$.
\end{proof}

\subsection{Linear convergence of projected policy gradient}\label{app:convergence:theorem}

Combining the three lemmas above, we obtain our main convergence result: on the feasible set $\mathcal Z_{\underline\delta}$ given by \cref{lem:matrix_game_policies_bounded_below}, the projected policy-gradient iteration converges linearly to the unique RQE.

\begin{theorem}[Projected policy gradient converges to the RQE in two-player matrix games]
\label{thm:matrix_game_srpo_convergence}
Consider the lifted four-player regularized game associated with a two-player matrix game with bounded payoff matrices $R_1,R_2$, and let $\underline\delta>0$ be the interior bound from \cref{lem:matrix_game_policies_bounded_below}. Assume the regularization parameters satisfy
\begin{equation}\label{eq:matrix-game-parameter-condition}
4\epsilon_1\tau_2\underline\delta>1,
\qquad
4\epsilon_2\tau_1\underline\delta>1,
\end{equation}
and let $L_{\underline\delta}$ and $\mu_{\underline\delta}$ denote, respectively, the Lipschitz constant from \cref{lem:matrix_game_operator_lipschitz} and the strong-monotonicity modulus from \cref{lem:matrix_game_operator_strong_monotone}, both evaluated at $\delta=\underline\delta$. Then:
\begin{enumerate}
    \item the lifted game admits a unique RQE $z^\star\in\mathcal Z_{\underline\delta}$;
    \item for every initialization $z^0\in\mathcal Z_{\underline\delta}$ and every step size satisfying
    \[
    0<\eta<\frac{2\mu_{\underline\delta}}{L_{\underline\delta}^2},
    \]
    the projected policy-gradient iterates
    \begin{equation}\label{eq:projected-gradient-update}
    z^{k+1}
    =
    \Pi_{\mathcal Z_{\underline\delta}}\bigl(z^k-\eta G(z^k)\bigr)
    \end{equation}
    converge linearly to $z^\star$: setting
    $\rho_{\underline\delta}:=1-2\eta\mu_{\underline\delta}+\eta^2 L_{\underline\delta}^2\in[0,1)$,
    \begin{equation}\label{eq:linear-convergence-rate}
    \|z^k-z^\star\|_2^2\le \rho_{\underline\delta}^{\,k}\,\|z^0-z^\star\|_2^2,
    \quad\forall k\ge 0.
    \end{equation}
\end{enumerate}
\end{theorem}
\begin{proof}
We write $\delta:=\underline\delta$ throughout the proof and proceed in four steps.

\paragraph{Step 1: Existence, variational-inequality characterization, and uniqueness.}
By the Gibbs variational principle, the infimum in~\eqref{eq:risk_adjusted_objective} equals $-\tau_i^{-1}\log\sum_{b\in\mathcal A_{-i}}x_{-i}(b)\exp(-\tau_i(R_i^\top x_i)(b))$, so $U_i^{\tau_i,\epsilon_i}$ is continuous on $\Delta(\mathcal A_1)\times\Delta(\mathcal A_2)$ and concave in $x_i$. Hence the original game has an RQE~\citep{rosen1965existence} and, by the equivalence in \cref{app:convergence:setup}, so does the lifted game.
Let $z^\star=(x_1^\star,x_2^\star,p_1^\star,p_2^\star)$ be any RQE; by \cref{lem:matrix_game_policies_bounded_below}, $z^\star\in\mathcal Z_\delta$ and each block of $z^\star$ lies in the relative interior of its simplex. Since $z^\star$ is an RQE of the lifted four-player game, each block is a best response to the others. For each $i\in\{1,2\}$, the function $x_i\mapsto J_i(x_i,x_{-i}^\star,p_i^\star)$ is concave and differentiable at $x_i^\star$, while $p_i\mapsto \bar J_i(x_i^\star,x_{-i}^\star,p_i)$ is concave and differentiable at $p_i^\star$. Therefore the first-order optimality condition for concave maximization on a convex set yields
\begin{align*}
\bigl\langle -\nabla_{x_1}J_1(z^\star),\,x_1-x_1^\star\bigr\rangle &\ge 0,
\qquad \forall x_1\in\Delta(\mathcal A_1), \\
\bigl\langle -\nabla_{x_2}J_2(z^\star),\,x_2-x_2^\star\bigr\rangle &\ge 0,
\qquad \forall x_2\in\Delta(\mathcal A_2), \\
\bigl\langle -\nabla_{p_1}\bar J_1(z^\star),\,p_1-p_1^\star\bigr\rangle &\ge 0,
\qquad \forall p_1\in\Delta(\mathcal A_2), \\
\bigl\langle -\nabla_{p_2}\bar J_2(z^\star),\,p_2-p_2^\star\bigr\rangle &\ge 0,
\qquad \forall p_2\in\Delta(\mathcal A_1).
\end{align*}
Adding these four inequalities and using the definition of $G$ in~\eqref{eq:G-definition}, we obtain
\begin{equation}\label{eq:vi-rqe}
\langle G(z^\star),\, z-z^\star\rangle \ge 0,
\qquad \forall z\in\mathcal Z_\delta.
\end{equation}
Hence the RQE is a solution of $\mathrm{VI}(G,\mathcal Z_\delta)$.

To show uniqueness, suppose $z^\star,\tilde z^\star\in\mathcal Z_\delta$ both solve $\mathrm{VI}(G,\mathcal Z_\delta)$. Plugging each into~\eqref{eq:vi-rqe} with the other as test point and adding gives
\begin{equation*}
\langle G(z^\star)-G(\tilde z^\star),\, z^\star-\tilde z^\star\rangle \le 0.
\end{equation*}
On the other hand, by \cref{lem:matrix_game_operator_strong_monotone},
\begin{equation*}
\langle G(z^\star)-G(\tilde z^\star),\, z^\star-\tilde z^\star\rangle
\ge
\mu_\delta \|z^\star-\tilde z^\star\|_2^2.
\end{equation*}
Therefore $\|z^\star-\tilde z^\star\|_2^2=0$, so $z^\star=\tilde z^\star$. This proves part~(1).

\paragraph{Step 2: Euclidean projection fixed-point identity.}
We claim that
\begin{equation}\label{eq:fixed-point}
z^\star=\Pi_{\mathcal Z_\delta}\bigl(z^\star-\eta G(z^\star)\bigr).
\end{equation}
Indeed, the Euclidean projection $\hat z:=\Pi_{\mathcal Z_\delta}(z^\star-\eta G(z^\star))$ is characterized by
\begin{equation*}
\bigl\langle z^\star-\eta G(z^\star)-\hat z,\, z-\hat z\bigr\rangle \le 0,
\qquad \forall z\in\mathcal Z_\delta.
\end{equation*}
Setting $\hat z=z^\star$, the left-hand side becomes
\begin{equation*}
\bigl\langle z^\star-\eta G(z^\star)-z^\star,\, z-z^\star\bigr\rangle
=
-\eta\langle G(z^\star),\, z-z^\star\rangle
\le 0
\end{equation*}
by~\eqref{eq:vi-rqe}. Hence $z^\star$ satisfies the projection optimality condition, proving~\eqref{eq:fixed-point}.

\paragraph{Step 3: One-step contraction estimate.}
Using \eqref{eq:fixed-point} and the nonexpansiveness of Euclidean projection onto the closed convex set $\mathcal Z_\delta$,
\begin{align*}
\|z^{k+1}-z^\star\|_2
&=
\|\Pi_{\mathcal Z_\delta}(z^k-\eta G(z^k))
-\Pi_{\mathcal Z_\delta}(z^\star-\eta G(z^\star))\|_2 \\
&\le
\|(z^k-\eta G(z^k))-(z^\star-\eta G(z^\star))\|_2.
\end{align*}
Squaring and expanding gives
\begin{equation*}
\|z^{k+1}-z^\star\|_2^2
\le
\|z^k-z^\star\|_2^2
-2\eta\langle G(z^k)-G(z^\star),\, z^k-z^\star\rangle
+\eta^2\|G(z^k)-G(z^\star)\|_2^2.
\end{equation*}
Since $z^k,z^\star\in\mathcal Z_\delta$, applying \cref{lem:matrix_game_operator_lipschitz,lem:matrix_game_operator_strong_monotone} yields
\begin{equation*}
\langle G(z^k)-G(z^\star),\, z^k-z^\star\rangle
\ge
\mu_\delta \|z^k-z^\star\|_2^2,
\qquad
\|G(z^k)-G(z^\star)\|_2
\le
L_\delta \|z^k-z^\star\|_2.
\end{equation*}
Substituting into the previous inequality yields the one-step contraction
\begin{equation}\label{eq:contraction}
\|z^{k+1}-z^\star\|_2^2
\le
\bigl(1-2\eta\mu_\delta+\eta^2L_\delta^2\bigr)\|z^k-z^\star\|_2^2.
\end{equation}

\paragraph{Step 4: Linear convergence.}
Let $\rho_\delta:=1-2\eta\mu_\delta+\eta^2L_\delta^2$. The step-size condition $0<\eta<2\mu_\delta/L_\delta^2$ gives $\rho_\delta<1$. Moreover, $\mu_\delta\le L_\delta$ (by Cauchy--Schwarz combined with the $L_\delta$-Lipschitz bound), so the quadratic $\rho_\delta$ has minimum value $1-\mu_\delta^2/L_\delta^2\ge 0$. Therefore $\rho_\delta\in[0,1)$, and iterating~\eqref{eq:contraction} yields
\begin{equation*}
\|z^k-z^\star\|_2^2
\le
\rho_\delta^{\,k}\, \|z^0-z^\star\|_2^2,
\qquad\forall k\ge 0.
\qedhere
\end{equation*}
\end{proof}

\section{Mathematical Background of SRPO}\label{app:srpo_math_background}
In this section, we provide a formal framework for risk aversion and entropy regularization in MARL, and present the motivation for our algorithm SRPO. We first introduce the mathematical framework of discounted general-sum Markov games, and then discuss how to extend risk aversion and entropy regularization from normal-form games to Markov games. Next, we present the pseudocode of SRPO and a short discussion. To better motivate SRPO, we provide policy gradient theorems for each original player and adversary as well as their performance difference lemmas (PDLs), and introduce how to (mathematically) obtain the SRPO loss from the PDLs.

\subsection{Discounted General-sum Markov Games}
We consider a discounted $N$-player general-sum Markov game is specified by a tuple:
$$\mathcal{MG}=\{\mathcal{S}, \{\mathcal{A}_i\}_{i=1}^N,\{r_i\}_{i=1}^N, \gamma, P ,\rho_0\},$$
where $\mathcal{S}$ is the state space of the underlying MDP, $\mathcal{A}_i$ is the action space of player $i\in[N]$, and we use the notation $\mathcal{A}=\prod_{i=1}^N \mathcal{A}_i$ to denote the product action space of both players. We assume that $|\mathcal{S}|$ and $|\mathcal{A}_i|$ are all finite.
Here, $r_i:\mathcal{S}\times \mathcal{A}\rightarrow [0, 1]$ is the reward function of player $i$, which we assume to be deterministic. We use $\mathbf{r}$ to denote the joint reward function $\mathbf{r}\coloneqq(r_i)_{i=1}^N$. Moreover, $\gamma\in[0,1)$ is the discount factor and $P:\mathcal{S}\times \mathcal{A}\rightarrow \Delta(\mathcal{S})$ is the transition kernel, where $\Delta(\mathcal{S})$ is the probability simplex of $\mathcal{S}$ and $P(s'|s,\mathbf{a})$ is the probability of the next state being $s'$ given the current state $s$ and the current actions $\mathbf{a}=(a_i)_{i=1}^N$ of the players. We use $\rho_0\in \Delta(\mathcal{S})$ to denote the initial state distribution.

We focus on \textit{Markov policies}, the class of policies where the action selection probability only depends on the current state instead of the entire gameplay trajectory, i.e., $\pi=(\pi_i)_{i=1}^N$ where $\pi_i:\mathcal{S}\rightarrow \Delta(\mathcal{A}_i), i\in [N]$.
Given a product Markov policy $\pi$, without considering risk aversion and entropy regularization, player $i$ has an expected discounted cumulative reward given by $\mathbb{E}_\pi[\sum_{t=0}^\infty \gamma^t r_i(s_t,a_t)]$.

To incorporate risk aversion and entropy regularization in discounted infinite-horizon Markov games, we slightly overload the notations in normal-form games and consider the following risk-adjusted objective of player $i$ that maximizes $f_i(\pi_i,\pi_{-i})=\min_{p_i:\mathcal{S}\rightarrow \Delta(\mathcal{A}_{-i})} J_i(\pi_i,\pi_{-i},p_i)$ where $J_i$ is defined as
\begin{equation}\label{eq:MARL_risk_adjusted_objective}\begin{aligned}
    J_i(\pi_i,\pi_{-i}, p_i)=\mathbb{E}_{\pi,p, s_0\sim\rho_0}\bigg[\sum_{t=0}^\infty \gamma^t\bigg(r_i(s_t,\mathbf{a}_t)
    +\frac{1}{\tau_i}D_i\left(p_i,\pi_{-i};s_t\right)-\epsilon_i\nu_i(\pi_i;s_t)\bigg) \bigg]
\end{aligned}\end{equation}
where the joint actions $\mathbf{a}_{t}$ are sampled through $a_{i,t}\sim \pi_i(\cdot|s_t)$ and $\mathbf{a}_{-i,t}\sim p_i(\cdot|s_t)$ and the next state $s_{t+1}$ is sampled from $s_{t+1}\sim P(\cdot|s_t,\mathbf{a}_t)$, and the notations of $D_i(p_i,\pi_{-i};s)$ and $\nu_i(\pi_i;s)$ are abbreviations of $D_i(p_i(\cdot|s),\pi_{-i}(\cdot|s))$ and $\nu_i(\pi_i(\cdot|s))$ respectively. Here, we allow for more generality of the regularizers, and, if we want to be consistent with the main body, we can choose $D_i$ to be $\KL$, and $\nu_i$ to be negative entropy $H$.

Given a set of original player policies $\pi_i$ and adversarial policies $p_i$, we define the value function for each state $s\in\mathcal{S}$ as:
\begin{equation}\label{eq:value_function}\begin{aligned}
    V_i^{\pi,p}(s)=\mathbb{E}_{\pi,p, s_0=s}\bigg[\sum_{t=0}^\infty \gamma^t\bigg(r_i(s_t,\mathbf{a}_t)
    +\frac{1}{\tau_i}D_i\left(p_i,\pi_{-i};s_t\right)-\epsilon_i\nu_i(\pi_i;s_t)\bigg) \bigg]
\end{aligned}\end{equation}
so that $J_i(\pi,p)=\mathbb{E}_{s\sim\rho_0}[V_i^{\pi,p}(s)]$, and the $Q$ function as:
\begin{equation}\label{eq:Q_function}
    Q_i^{\pi,p}(s,\mathbf{a})=r_i(s,\mathbf{a})+\gamma \mathbb{E}_{s'\sim P(\cdot|s,\mathbf{a})}V_i^{\pi,p}(s').
\end{equation}
It is easy to verify that
\begin{equation}\begin{aligned}
    V_i^{\pi,p}(s)=&\pi_i(\cdot|s)^TQ_i^{\pi,p}(s,\cdot)p_i(\cdot|s)
    +\frac{1}{\tau_i}D_i(p_i,\pi_{-i};s)-\epsilon_i\nu_i(\pi_i;s).
\end{aligned}\end{equation}

Finally, we define given policies $\pi_i,p_i$, we define the discounted state visitation probability as
\begin{equation}\begin{aligned}
    d_{s_0}^{\pi_i,p_i}(s)\coloneqq (1-\gamma)\sum_{t=0}^\infty \gamma^t \Pr_{\pi_i,p_{i}}(s_t=s|s_0)
    =(1-\gamma)\sum_{t=0}^\infty \gamma^t e_{s_0}^T(P^{\pi_i,p_i})^t e_s,
\end{aligned}\end{equation}
where $e_s$ denotes the one-hot vector corresponding to state $s$. To learn the RQE of the Markov game, we have to optimize the policies $\pi_i$ to maximize (and $p_i$ to minimize) the risk-adjusted objective \eqref{eq:MARL_risk_adjusted_objective}. Here, we provide some useful tools for optimizing these policies.

\subsection{Details of SRPO}\label{app:srpo_details}
When using PPO~\citep{schulman2017proximalpolicyoptimizationalgorithms}, the objective of agent $i$ is given by
\begin{equation}\begin{aligned}
    &\mathcal{L}_i^\mathrm{IPPO}(\theta_i,\theta_{-i})
    =\mathcal{L}_i^\mathrm{CLIP}(\theta_i,\theta_{-i})-\epsilon_i H(\theta_i)(o_i^t),
\end{aligned}\end{equation}
where the clipped surrogate objective is
\begin{equation*}\begin{aligned}
    \mathcal{L}_i^\mathrm{CLIP}(\theta_i,\theta_{-i})
    =& \:\mathbb{E}_{t} \Big[ \min \Big( r_i^t(\theta_i) \hat{A}_i^t(\theta_i,\theta_{-i}),
    \text{clip}(r_i^t(\theta_i), 1 - \delta, 1 + \delta) \hat{A}_i^t(\theta_i,\theta_{-i}) \Big) \Big],
\end{aligned}\end{equation*}
with $r_i^t(\theta_i) = \frac{\pi_{\theta_i}(a_i^t | o_i^t)}{\pi_{\theta_{i, \text{old}}}(a_i^t | o_i^t)}$ being the importance sampling ratio, $\hat{A}_i^t(\theta_i,\theta_{-i})$ the estimated advantage at time $t$ (and depends on $\theta_{-i}$ as it is computed using samples of the policies of all agents)%
, $\epsilon_i H(\theta_i)(o_i^t)$ a (negative) entropy regularizer, and $o_i^t$ the observation of agent $i$ at time $t$.

We can now follow the meta-algorithm above to derive strategically risk-averse proximal policy optimization.
Specifically, the objective of agent $i$ is
\begin{equation}\label{eq:srpo_loss_agent}
    \mathcal{L}_i^\mathrm{SRPO}(\theta_i,\phi_i)
    =
    \mathcal{L}_i^\mathrm{CLIP}(\theta_i,\phi_i)-\epsilon_i H(\theta_i)(o_i^t).
\end{equation}
Since the PPO objective already includes entropy regularization, the additional entropy term in~\eqref{eq:srpo:loss} is superfluous. The objective of player $i$'s adversary is then
\begin{equation}\label{eq:srpo_loss_adversary}
    \bar{\mathcal{L}}_i^\mathrm{SRPO}(\phi_i,(\theta_i,\theta_{-i}))
    =
    -\mathcal{L}_i^\mathrm{CLIP}(\theta_i,\phi_i)- \frac{1}{\tau_i}\KL(\phi_i,\theta_{-i}).
\end{equation}
Overall, SRPO consists of running gradient descent on these objectives, and the iteration structure and complexity of SRPO is very similar to that of IPPO.

We now present the pseudocode for SRPO in \Cref{alg:srpo}. Notice that although the agents and adversaries are maximizing two losses $\mathcal{L}_i^\mathrm{SRPO}$ and $\bar{\mathcal{L}}_i^\mathrm{SRPO}$ respectively, we unify these two losses by considering the joint loss:
\begin{equation}
    \mathcal{L}_{i,\mathrm{joint}}^\mathrm{SRPO}(\theta_i,\theta_{-i},\phi_i)
    =
    \mathcal{L}_i^\mathrm{CLIP}(\theta_i,\phi_i)+\frac{1}{\tau_i}\KL(\phi_i,\theta_{-i})-\epsilon_i H(\theta_i)(o_i^t).
\end{equation}
The update of agent $i$ can be seen as maximizing $\mathcal{L}_{i,\mathrm{joint}}^\mathrm{SRPO}$ as a function of $\theta_i$, and the update of adversary $i$ can be seen as minimizing $\mathcal{L}_{i,\mathrm{joint}}^\mathrm{SRPO}$ as a function of $\phi_i$.

\begin{algorithm}[htb]
\caption{Strategically Risk-averse Policy Optimization (SRPO) for $2N$ agents}
\label{alg:srpo}
\begin{algorithmic}[1]
    \STATE {\bfseries Initialize:} Policies $\{\pi_{\theta_i}\}_{i=1}^N$, adversaries $\{\pi_{\phi_i}\}_{i=1}^N$, and critics $\{V_{\psi_i}\}_{i=1}^N$.
    \STATE {\bfseries Parameters:} Risk aversion $\{\tau_i\}_{i=1}^N$, entropy regularization $\{\epsilon_i\}_{i=1}^N$, learning rates $\eta_\theta, \eta_\phi, \eta_\psi$, clipping $\delta$.
    \FOR{iteration $k=1, 2, \dots$}
        \FOR{each agent $i \in \{1, \dots, N\}$}
            \STATE {\bfseries // Data Collection}
            \STATE Run joint policy $\boldsymbol{\pi}_{\theta_i,\phi_i} = (\pi_{\theta_i},\pi_{\phi_i})$ of agent $i$ and adversary $i$ in the environment for $T$ steps.
            \STATE Store trajectories $\mathcal{D} = \{ (o_i^t, a_i^t, r_i^t) \}_{i=1, t=1}^{N, T}$.
            \STATE Compute advantage estimates $\hat{A}_i^t$ using local critic $V_{\psi_i}$.
            \STATE {\bfseries // Maximin-Optimization}
            \FOR{step in steps}
                \STATE \textit{\# Update Agent by maximizing $\mathcal{L}_i^\mathrm{SRPO}$}
                \STATE $\theta_i \leftarrow \theta_i + \eta_\theta \nabla_{\theta_i} \mathcal{L}_i^\mathrm{SRPO}$

                \STATE \textit{\# Update Adversary by maximizing $\bar{\mathcal{L}}_i^\mathrm{SRPO}$}
                \STATE $\phi_i \leftarrow \phi_i + \eta_\phi \nabla_{\phi_i} \bar{\mathcal{L}}_i^\mathrm{SRPO}$

                \STATE \textit{\# Update Local Critic}
                \STATE $\psi_i \leftarrow \psi_i - \eta_\psi \nabla_{\psi_i} \mathcal{L}_i^{VF}(\psi_i)$
            \ENDFOR
        \ENDFOR
        \STATE Update old policy parameters: $\theta_{i, \text{old}} \leftarrow \theta_i$ for all $i$.
    \ENDFOR
\end{algorithmic}
\end{algorithm}

\paragraph{Discussion}
The iteration structure of SRPO is very similar to that of IPPO. In the data collection phase, unlike IPPO, agent $i$ computes the advantage estimate based on rollouts played against adversary $i$ instead of other real agents. In the optimization phase, both agent $i$ and their adversary update their policies to maximize the SRPO losses \eqref{eq:srpo_loss_agent} and \eqref{eq:srpo_loss_adversary}. Here, we have introduced an additional critic with parameter $\psi_i$ with the MSE loss function $\mathcal{L}_i^{VF}(\theta_i,\theta_{-i})$, as a standard practice for variance reduction in IPPO, which is only used when computing the advantage estimates.

\subsection{Policy Gradient Theorems}
Policy gradient theorems are the most fundamental theoretical tool for nearly all policy-based algorithms. As the foundation of our framework, we provide the expression of policy gradients for the original agents and the adversaries under the risk-adjusted objectives, respectively. Although our results are the first to present these, similar results for single-agent version can be found in \citep{lan2022policymirrordescentreinforcement}.

For the original agents, the policy gradient theorem has the following form:
\begin{lemma}[Policy Gradient for $\pi_i$]\label{lem:policy_gradient_pi}
    The gradient $\nabla_{\pi_i} V_i^{\pi,p}(s)$ can be written as
    \begin{equation*}\begin{aligned}
        \nabla_{\pi_i(\cdot|x)} V_i^{\pi,p}(s)=\frac{1}{1-\gamma} d_{s}^{\pi_i,p_i}(x)
        \left[\mathbb{E}_{\mathbf{a}_{-i}\sim p_i(\cdot|x)}[Q_i^{\pi,p}(x,\cdot, \mathbf{a}_{-i})] -\epsilon_i\nabla_{\pi_i}\nu_i(\pi_i;x)\right].
    \end{aligned}\end{equation*}
    If $\pi_i$ is parameterized by $\theta_i$, we have
    \begin{equation*}\begin{aligned}
        \nabla_{\theta_i}V_i^{\pi,p}(s)=\frac{1}{1-\gamma}\mathbb{E}_{x\sim d_{s}^{\pi_i,p_i},\mathbf{a}\sim(\pi_i\otimes p_i)(\cdot|x)}
        \Bigg[\nabla_{\theta_i} \log{\pi_i(a_i|x)}
        \left(Q_i^{\pi,p}(x,\mathbf{a})-\epsilon_i\frac{\partial \nu_i(\pi_i;x)}{\partial \pi_i(a_i|x)}\right)\Bigg],
    \end{aligned}\end{equation*}
    where $\mathbf{a}=(a_i,\mathbf{a}_{-i})$.
\end{lemma}
\begin{proof}
    We have
    \begin{equation*}
        \begin{aligned}
            \nabla_{\pi_i(\cdot|x)} V_i^{\pi,p}(s)=&\langle \frac{\partial}{\partial \pi_i(\cdot|x)}Q_i^{\pi,p}(s,\cdot),(\pi_i \otimes p_i)(\cdot|s)\rangle\\
            &+\left(\mathbb{E}_{a_{-i}\sim p_i(\cdot|s)}[Q_i^{\pi,p}(s,\cdot, a_{-i})]-\epsilon_i\nabla_{\pi_i}\nu_i(\pi_i;s)\right)\mathbb{I}[s=x]\\
            =&\gamma \mathbb{E}_{s'\sim P^{\pi_i,p_i}(s)}[\nabla_{\pi_i} V_i^{\pi,p}(s')]\\
            &+\left(\mathbb{E}_{a_{-i}\sim p_{i}(\cdot|s)}[Q_i^{\pi,p}(s,\cdot, a_{-i})]-\epsilon_i\nabla_{\pi_i}\nu_i(\pi_i;s)\right)\mathbb{I}[s=x]\\
            =&\dots\\
            =&\frac{1}{1-\gamma} d_{s}^{\pi_i,p_i}(x)\left[\mathbb{E}_{a_{-i}\sim p_i(\cdot|x)}[Q_i^{\pi,p}(x,\cdot, a_{-i})]-\epsilon_i\nabla_{\pi_i}\nu_i(\pi_i;x)\right].
        \end{aligned}
    \end{equation*}
    Similarly, for the parameterized case, we have
    \begin{equation*}
        \begin{aligned}
            \frac{\partial}{\partial \pi_i(a_i|x)} &V_i^{\pi,p}(s)\\
            =&\left\langle \frac{\partial}{\partial \pi_i(a_i|x)}Q_i^{\pi,p}(s,\cdot),(\pi_i \otimes p_i)(\cdot|s)\right\rangle\\
            &+\left(\mathbb{E}_{a_{-i}\sim p_{i}(\cdot|s)}[Q_i^{\pi,p}(s,a_i, a_{-i})]-\epsilon_i\frac{\partial}{\partial \pi_i(a_i|x)}\nu_i(\pi_i;s)\right)\mathbb{I}[s=x]\\
            =&\gamma \mathbb{E}_{s'\sim P^{\pi_i,p_i}(s)}[\frac{\partial}{\partial \pi_i(a_i|x)} V_i^{\pi,p}(s')]\\
            &+\left(\mathbb{E}_{a_{-i}\sim p_i(\cdot|s)}[Q_i^{\pi,p}(s,a_i, a_{-i})]-\epsilon_i\frac{\partial}{\partial \pi_i(a_i|x)}\nu_i(\pi_i;s)\right)\mathbb{I}[s=x]\\
            =&\dots\\
            =&\frac{1}{1-\gamma} d_{s}^{\pi_i,p_i}(x)\left[\mathbb{E}_{a_{-i}\sim p_{i}(\cdot|x)}[Q_i^{\pi,p}(x,a_i, a_{-i})]-\epsilon_i\frac{\partial}{\partial \pi_i(a_i|x)}\nu_i(\pi_i;x)\right],
        \end{aligned}
    \end{equation*}
    therefore,
    \begin{equation*}
        \begin{aligned}
            \nabla_{\theta_i} V_i^{\pi,p}(s)=&\sum_{x,a_i} \frac{\partial}{\partial \pi_i(a_i|x)} V_i^{\pi,p}(s) \nabla_{\theta_i} \pi_i(a_i|x)\\
            =&\frac{1}{1-\gamma}\sum_{x,a_i} d_{s}^{\pi_i,p_i}(x)\left[\mathbb{E}_{a_{-i}\sim p_i(\cdot|x)}[Q_i^{\pi,p}(x,a_i, a_{-i})]-\epsilon_i\frac{\partial}{\partial \pi_i(a_i|x)}\nu_i(\pi_i;x)\right]\nabla_{\theta_i} \pi_i(a_i|x)\\
            =&\frac{1}{1-\gamma}\mathbb{E}_{x\sim d_{s}^{\pi_i,p_i}}\mathbb{E}_{a\sim(\pi_i\otimes p_i)(\cdot|x)}\left[\nabla_{\theta_i} \log{\pi_i(a_i|x)}\left(Q_i^{\pi,p}(x,a)-\epsilon_i\frac{\partial}{\partial \pi_i(a_i|x)}\nu_i(\pi_i;x)\right)\right].
        \end{aligned}
    \end{equation*}
    This concludes the proof.
\end{proof}

Similarly, we have the policy gradient for $p_i$ as follows:
\begin{lemma}[Policy gradient for $p_i$]\label{lem:policy_gradient_p}
    The gradient $\nabla_{p_i} V_i^{\pi,p}(s)$ can be written as
    \begin{equation*}
    \begin{aligned}
        \nabla_{p_i(\cdot|x)}V_i^{\pi,p}(s)=\frac{1}{1-\gamma}d_s^{\pi_i,p_i}(x)
        \left[\mathbb{E}_{a_i\sim \pi_i(\cdot|x)}[Q_i^{\pi,p}(x,a_i,\cdot)]+\frac{1}{\tau_i}\nabla_{p_i}D_i(p_i,\pi_{-i};x)\right]
    \end{aligned}
    \end{equation*}
    and if $p_i$ is parameterized by $\bar{\theta}_i$, we have
    \begin{equation*}
        \begin{aligned}
            \nabla_{\bar{\theta}_i}V_i^{\pi,p}(s)=\frac{1}{1-\gamma} \mathbb{E}_{x\sim d_s^{\pi_i,p_i},\mathbf{a}\sim(\pi_i\otimes p_i)(\cdot|x)}
            \left[\nabla_{\bar{\theta}_i} \log p_i(\mathbf{a}_{-i}|x)\left(Q_i^{\pi,p}(x,\mathbf{a})+\frac{1}{\tau_i}\frac{\partial D_i(p_i,\pi_{-i};x)}{\partial p_i(\mathbf{a}_{-i}|x)} \right) \right],
        \end{aligned}
    \end{equation*}
    where $\mathbf{a}=(a_i,\mathbf{a}_{-i})$.
\end{lemma}
\begin{proof}
We have
\begin{equation*}
    \begin{aligned}
        \nabla_{p_i(\cdot|x)}V_i^{\pi,p}(s)=&\nabla_{p_i(\cdot|x)}(\langle Q_i^{\pi,p}(s,\cdot), (\pi_i\otimes p_i)(\cdot|s) \rangle)+\frac{1}{\tau_i}\nabla_{p_i}D_i(p_i,\pi_{-i};s)\mathbb{I}[s=x]\\
        =&\langle \frac{\partial}{\partial p_i(\cdot|x)} Q_i^{\pi,p}(s,\cdot), (\pi_i\otimes p_i)(\cdot|s) \rangle\\
        &+ \left(\mathbb{E}_{a_i\sim \pi_i(\cdot|s)}[Q_i^{\pi,p}(s,a_i,\cdot)]+\frac{1}{\tau_i}\nabla_{p_i}D_i(p_i,\pi_{-i};s)\right)\mathbb{I}[s=x]\\
        =&\gamma \mathbb{E}_{s'\sim P^{\pi_i,p_i}(s)}[\nabla_{p_i(\cdot|x)} V_i^{\pi,p}(s')]\\
        &+\left(\mathbb{E}_{a_i\sim \pi_i(\cdot|s)}[Q_i^{\pi,p}(s,a_i,\cdot)]+\frac{1}{\tau_i}\nabla_{p_i}D_i(p_i,\pi_{-i};s)\right)\mathbb{I}[s=x]\\
        =& \dots\\
        =&\frac{1}{1-\gamma}d_s^{\pi_i,p_i}(x)\left[\mathbb{E}_{a_i\sim \pi_i(\cdot|x)}[Q_i^{\pi,p}(x,a_i,\cdot)]+\frac{1}{\tau_i}\nabla_{p_i}D_i(p_i,\pi_{-i};x)\right],
    \end{aligned}
\end{equation*}
and similarly for the parameterized case,
\begin{equation*}\begin{aligned}
    &\nabla_{\bar{\theta}_i}V_i^{\pi,p}(s)\\
    =&\frac{1}{1-\gamma} \mathbb{E}_{x\sim d_s^{\pi_i,p_i}}\mathbb{E}_{a\sim(\pi_i\otimes p_i)(\cdot|x)}\left[\nabla_{\bar{\theta}_i} \log p_i(a_{-i}|x)\left(Q_i^{\pi,p}(x,a)+\frac{1}{\tau_i}\frac{\partial}{\partial p_i(a_{-i}|x)}D_i(p_i,\pi_{-i};x) \right) \right].
\end{aligned}\end{equation*}
This concludes the proof.
\end{proof}

\subsection{Performance Difference Lemmas}
Another crucial tool for quantifying the difference in value functions (hence expected returns) between different policies through the advantage function is the \textit{performance difference lemma} (PDL). For our 4-player risk-adjusted game, we state the PDLs for original agents and adversaries below:
\begin{lemma}[Performance Difference Lemma for $\pi_i$]\label{lem:PDL_pi}
    For two policies $\pi_i$ and $\pi_i'$ of player $i$, we have
    \begin{equation}\begin{aligned}
        V_i^{\pi_i,z}(s)-V_i^{\pi_i',z}(s)=\frac{1}{1-\gamma}
        \mathbb{E}_{s'\sim d_s^{\pi_i,p_i},\mathbf{a}\sim (\pi_i\otimes p_i)(\cdot|s')}\left[A_i^{\pi_i',z}(s',\mathbf{a})-\epsilon_i\nu_i(\pi_i;s') \right]
    \end{aligned}\end{equation}
    where the advantage function $A_i$ is defined as:
    \begin{equation*}
        A_i^{\pi_i,z}(s,\mathbf{a})=Q_i^{\pi_i,z}(s,\mathbf{a})-V_i^{\pi_i,z}(s)+\frac{1}{\tau_i}D_i(z;s).
    \end{equation*}
    Here, $z$ is the shorthand notation of $z=(\pi_{-i},p_i)$.
\end{lemma}
\begin{proof}
    We can write the difference in value function as
    \begin{align*}\allowdisplaybreaks
            &V_i^{\pi_i',z}(s)-V_i^{\pi_i,z}(s)\\
            =&\mathbb{E}_{\pi_i',z,s_0=s}\left[\sum_{t=0}^\infty \gamma^t\bigg(r_i(s_t,\mathbf{a}_t)+\frac{1}{\tau_i}D_i\left(z;s_t\right)-\epsilon_i\nu_i(\pi_i';s_t)\bigg) \right]-V_i^{\pi_i,z}(s)\\
            =&V_i^{\pi_i',z}(s)-\mathbb{E}_{(a_{i,0},\mathbf{a}_{-i,0})\sim (\pi_i\otimes p_i)(\cdot|s)}\left[r_i(s,\mathbf{a}_0)+\frac{1}{\tau_i}D_i(z;s)-\epsilon_i\nu_i(\pi_i;s)+\gamma \mathbb{E}_{s'\sim P(\cdot|s,\mathbf{a}_0)}V_i^{\pi_i',z}(s')\right]\\
            &+\mathbb{E}_{(a_{i,0},\mathbf{a}_{-i,0})\sim (\pi_i\otimes p_i)(\cdot|s)}\left[r_i(s,\mathbf{a}_0)+\frac{1}{\tau_i}D_i(z;s)-\epsilon_i\nu_i(\pi_i;s)+\gamma \mathbb{E}_{s'\sim P(\cdot|s,\mathbf{a}_0)}V_i^{\pi_i',z}(s')\right]-V_i^{\pi_i,z}(s)\\
            =&\mathbb{E}_{(a_{i,0},\mathbf{a}_{-i,0})\sim (\pi_i\otimes p_i)(\cdot|s)}\left[V_i^{\pi_i',z}(s)-Q_i^{\pi_i',z}(s,\mathbf{a}_0)-\frac{1}{\tau_i}D_i(z;s)+\epsilon_i\nu_i(\pi_i;s) \right]\\
            &+\gamma\mathbb{E}_{(a_{i,0},\mathbf{a}_{-i,0})\sim (\pi_i\otimes p_i)(\cdot|s)}\mathbb{E}_{s'\sim P(s,\mathbf{a}_0)}\left[V_i^{\pi_i',z}(s')-V_i^{\pi_i,z}(s')\right]\\
            =&\cdots\\
            =&\frac{1}{1-\gamma} \mathbb{E}_{s'\sim d_s^{\pi_i,p_i}}\mathbb{E}_{\mathbf{a}\sim (\pi_i\otimes p_i)(\cdot|s')}\left[V_i^{\pi_i',z}(s')-Q_i^{\pi_i',z}(s',\mathbf{a})-\frac{1}{\tau_i}D_i(z;s')+\epsilon_i\nu_i(\pi_i;s') \right]
    \end{align*}
    This concludes the proof.
\end{proof}

\begin{lemma}[Performance Difference Lemma for $p_i$]\label{lem:PDL_p}
    For two policies $p_i$ and $p_i'$ of adversary $i$, we have:
    \begin{equation*}\begin{aligned}
        V_i^{p_i,\pi}(s)-V_i^{p_i',\pi}(s)=\frac{1}{1-\gamma}
        \mathbb{E}_{s'\sim d_s^{\pi_i,p_i}, \mathbf{a}\sim (\pi_i\otimes p_i)(\cdot|s')}\left[A_i^{p_i',\pi}(s',\mathbf{a})+\frac{1}{\tau_i}D_i(p_i,\pi_{-i};s')\right]
    \end{aligned}
    \end{equation*}
    where
    \begin{equation*}
        A_i^{p_i,\pi}(s,\mathbf{a})=Q_i^{p_i,\pi}(s,\mathbf{a})-V_i^{p_i,\pi}(s)-\epsilon_i\nu_i(\pi_i;s).
    \end{equation*}
\end{lemma}

\begin{proof}\allowdisplaybreaks
    We can write the value function difference as
    \begin{align*}
            &V_i^{p_i,\pi}(s)-V_i^{p_i',\pi}(s)\\
            =&V_i^{p_i,\pi}(s)-\mathbb{E}_{(a_{i,0},\mathbf{a}_{-i,0})\sim (\pi_i\otimes p_i)(\cdot|s)}\left[r_i(s,\mathbf{a}_0)+\frac{1}{\tau_i}D_i(p_i,\pi_{-i};s)-\epsilon_i\nu_i(\pi_i;s)+\gamma \mathbb{E}_{s'\sim P(\cdot|s,\mathbf{a}_0)}V_i^{p_i',\pi}(s')\right]\\
            &+\mathbb{E}_{(a_{i,0},\mathbf{a}_{-i,0})\sim (\pi_i\otimes p_i)(\cdot|s)}\left[r_i(s,\mathbf{a}_0)+\frac{1}{\tau_i}D_i(p_i, \pi_{-i};s)-\epsilon_i\nu_i(\pi_i;s)+\gamma \mathbb{E}_{s'\sim P(\cdot|s,\mathbf{a}_0)}V_i^{p_i',\pi}(s')\right]-V_i^{p_i',\pi}(s)\\
            =&\gamma \mathbb{E}_{(a_{i,0},\mathbf{a}_{-i,0})\sim (\pi_i\otimes p_i)(\cdot|s)}\mathbb{E}_{s'\sim P(\cdot|s,\mathbf{a}_0)}[V_i^{p_i,\pi}(s')-V_i^{p_i',\pi}(s')]\\
            &+\mathbb{E}_{(a_{i,0},\mathbf{a}_{-i,0})\sim (\pi_i\otimes p_i)(\cdot|s)}\left[Q_i^{p_i',\pi}(s,\mathbf{a}_0)+\frac{1}{\tau_i}D_i(p_i, \pi_{-i};s)-\epsilon_i\nu_i(\pi_i;s)-V_i^{p_i',\pi}(s) \right]\\
            =&\dots\\
            =&\frac{1}{1-\gamma}\mathbb{E}_{s'\sim d_s^{\pi_i,p_i}}\mathbb{E}_{\mathbf{a}\sim (\pi_i\otimes p_i)(\cdot|s')}\left[Q_i^{p_i',\pi}(s',\mathbf{a})+\frac{1}{\tau_i}D_i(p_i,\pi_{-i};s')-\epsilon_i\nu_i(\pi_i;s')-V_i^{p_i',\pi}(s')\right].
    \end{align*}
    This concludes the proof.
\end{proof}

\subsection{Motivating the SRPO Objective}
Having stated the policy gradient theorems and PDLs, we now motivate our SRPO objective following the same rationale as that in TRPO \citep{schulman2015trust}.
Let $z=(\pi^0,p^0)$ be the joint policy at the current timestep, by \Cref{lem:PDL_pi}, consider policy update for $\pi_i$, we can write the value function for any policy $\pi_i$ as
\begin{equation*}\begin{aligned}
    J_i^{\pi_i,\pi_{-i}^0,p_i^0}=J_i^{\pi^0,p^0}+\frac{1}{1-\gamma}
    \mathbb{E}_{s\sim d^{\pi_i,p_i^0},\mathbf{a}\sim (\pi_i\otimes p_i^0)(\cdot|s)}\left[A_i^{\pi_i^0,\pi_{-i}^0,p_i^0}(s,\mathbf{a})-\epsilon_i\nu_i(\pi_i;s) \right],
\end{aligned}\end{equation*}
where $d^{\pi_i,p_i^0}$ is the state distribution for the initial state distribution $\rho_0$.
When $\pi_i$ is constrained to be close to $\pi_i^0$, we use the following surrogate objective first proposed in TRPO \citep{schulman2015trust} to approximate the expectation term by replacing the sample distribution from $d^{\pi_i,p_i^0}$ to $d^{\pi_i^0,p_i^0}$:
\begin{equation}\label{eq:surrogate_objective_pi}\begin{aligned}
    \mathbb{E}_{s\sim d^{\pi_i^0,p_i^0},\mathbf{a}\sim (\pi_i^0\otimes p_i^0)(\cdot|s)}\left[\frac{\pi_i(a_i|s)}{\pi_i^0(a_i|s)}A_i^{\pi^0,p_i^0}(s,\mathbf{a})-\epsilon_i\nu_i(\pi_i;s) \right],
\end{aligned}\end{equation}
where the approximation $d^{\pi_i^0,p_i^0}\approx d^{\pi_i,p_i^0}$ holds when $\pi^0$ stays close to $\pi$, and importance sampling through the ratio $\frac{\pi_i(a_i|s)}{\pi_i^0(a_i|s)}$ on $a_i$. Similarly by \Cref{lem:PDL_p}, we have the following surrogate objective for $p_i$ as follows:
\begin{equation}\label{eq:surrogate_objective_p}\begin{aligned}
    \mathbb{E}_{s\sim d^{\pi_i^0,p_i^0}, \mathbf{a}\sim (\pi_i^0\otimes p_i^0)(\cdot|s)}\Bigg[&\frac{p_i(\mathbf{a}_{-i}|s)}{p_i^0(\mathbf{a}_{-i}|s)}A_i^{p_i^0,\pi^0}(s,\mathbf{a})+\frac{1}{\tau_i}D_i(p_i,\pi_{-i};s)\Bigg].
\end{aligned}\end{equation}
The original TRPO objective suggests optimizing \eqref{eq:surrogate_objective_pi} and \eqref{eq:surrogate_objective_p} subject to the constraints
$$\mathbb{E}_{d^{\pi_i^0,p_i^0}}[\KL(\pi_i(\cdot|s)\|\pi_i^0(\cdot|s))]\leq \delta,$$
and $$\mathbb{E}_{d^{\pi_i^0,p_i^0}}[\KL(p_i(\cdot|s)\|p_i^0(\cdot|s))]\leq \delta,$$respectively. A later adaptation PPO \citep{schulman2017proximalpolicyoptimizationalgorithms} replaces these hard constraints with the clipped surrogate objective. Following this adaptation and replacing the policy to condition on observations instead of the global states, we obtain our SRPO losses \eqref{eq:srpo_loss_agent} and \eqref{eq:srpo_loss_adversary} respectively.

\section{Experimental Details}\label{sec:experimental-details}
\subsection{Overall Setup}
Across all experiments, we compare SRPO against IPPO under a unified training and evaluation protocol. Each method is trained from scratch using multiple random seeds, and performance is reported as episodic return averaged over rollouts.
\paragraph{Training.} Agents are trained for a fixed interaction budget per environment, with periodic evaluation during training. Unless otherwise specified, SRPO and IPPO agents share the same architecture, optimizer settings, amount of interaction with the environment, and entropy regularization, differing only in the inclusion of the strategic risk-averse objective. Given the same total amount of interactions, SRPO takes roughly twice the training time as IPPO.
\paragraph{Cross-play evaluation.} To assess unseen partner generalization, we adopt a \textbf{cross-play} protocol. For each environment, we collect a population of independently trained agents and evaluate all pairwise combinations without further learning. Cross-play performance is reported as a matrix, where each entry corresponds to the average episodic return of a fixed agent pair evaluated over multiple episodes. This protocol measures zero-shot coordination with previously unseen partners.
\subsection{Overcooked Gridworld}\label{sec:app-overcooked}
\paragraph{Environments.}We implemented a revised version of Overcooked AI, shown in \cref{fig:overcooked-env}. The task is defined on a $5\times 5$ grid with two agents acting simultaneously. Agents start from one of two symmetric initial configurations in the top left and must cooperatively pick up onions from two fixed sources and deliver them to a central pot. Each agent selects from five discrete actions (up, down, left, right, stay); non-stay actions incur a private movement cost of 0.2. Agent updates are applied sequentially in random order each step, with collisions blocked and penalized by 2.0. An agent can carry at most one onion; stepping onto an available onion grants a shared reward of +1 and removes the onion, while delivering an onion by attempting to move into the (solid) pot yields a shared reward of +10. Each onion independently respawns with probability 0.2 per step. Observations are fully shared and consist of both agents' positions, carry status, and onion availability. In this setup, rewards combine private movement penalties with shared team rewards, inducing cooperative behavior under individual costs.

\begin{figure}[t]
\centering
\begin{minipage}{0.3\linewidth}
    \centering
    \includegraphics[height=2.1cm]{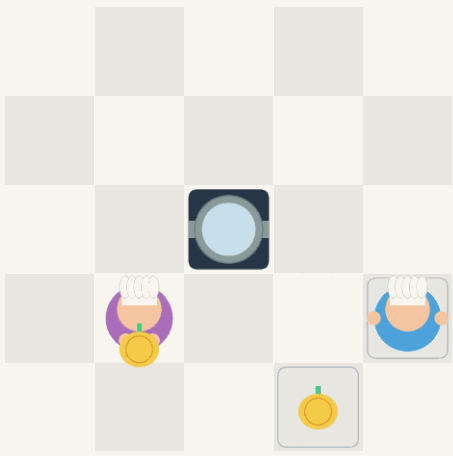}
    \caption{Overcooked Gridworld.}
    \label{fig:overcooked-env}
\end{minipage}
\hfill
\begin{minipage}{0.3\linewidth}
    \centering
    \includegraphics[height=2.1cm]{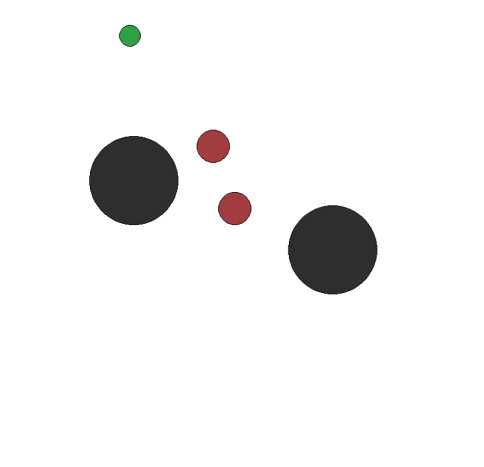}
    \caption{Tag.}
    \label{fig:tag}
\end{minipage}
\hfill
\begin{minipage}{0.3\linewidth}
    \centering
    \includegraphics[height=2.1cm]{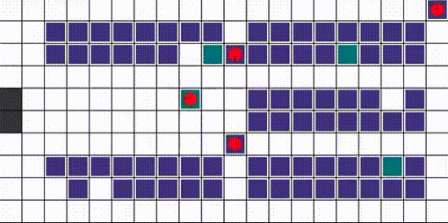}
    \caption{RWARE.}
    \label{fig:rware}
\end{minipage}
\end{figure}
\paragraph{Training and evaluation.} Each agent is trained for approximately $2\times 10^6$ environment interactions. Both the actor and the critic are two-layer MLPs with $64$ hidden units and $\tanh$ activations. Training statistics are recorded every 10240 interactions, where each evaluation point reports the average return over 5 rollouts of length 128, and each update performs $4$ PPO epochs over the full batch with the clipped surrogate objective ($\epsilon_{\text{clip}}=0.2$), GAE ($\gamma=0.99$ and $\lambda=0.95$), and Adam at learning rate $6\times 10^{-4}$. We perform 30 independent runs for each method. IPPO agents use an entropy coefficient of $\epsilon=0.1$, while SRPO agents use $\tau = 10$ with the same entropy coefficient.

For the FCP baseline~\citep{strouse2022collaboratinghumanshumandata}, we first train $N=10$ self-play seeds with a single shared
actor-critic controlling both seats, snapshotting $K=3$ checkpoints per seed (early, mid, and final), yielding a frozen partner
pool of $N\times K = 30$ partners of varying skill levels. We then train an FCP best-response policy with PPO against this
pool, where at each episode every parallel environment independently samples a partner index and an FCP seat assignment; only
the FCP policy receives gradient updates. To keep the FCP network's input distribution consistent across seat assignments, we
apply a slot-swap to the observation when FCP occupies seat~$1$ so that its own seat features always appear first, and we
mirror this swap at cross-play evaluation time.

For the main experiments, we construct a $60 \times 60$ cross-play matrix consisting of 30 IPPO agents and 30 SRPO agents. Each entry is averaged over 10 evaluation episodes of length 128. We report cross-play results for both the training environment and the held-out test environment. The results are shown in \cref{fig:overcooked}, and comparison with the FCP baseline is reported in \cref{fig:app-baseline-FCP}.
\paragraph{On the FCP baseline.} Although FCP is a natural reference point for partner generalization, we find it both
computationally expensive and empirically weaker than SRPO on this task, as shown in \cref{fig:app-baseline-FCP}. In terms of scalability, FCP requires a two-stage pipeline:
$N$ self-play runs must complete before any best-response training can begin, and the FCP run itself consumes the same $2\times
10^6$ interaction budget as a single IPPO/SRPO seed. In terms of cross-play return, FCP underperforms SRPO primarily because the partner pool is generated by self-play, which tends to converge to similar policies in this gridworld. As a result, the FCP best response effectively learns to free-ride rather than developing robust coordination behavior.

\begin{figure}[h]
    \centering
    \includegraphics[width=0.45\textwidth]{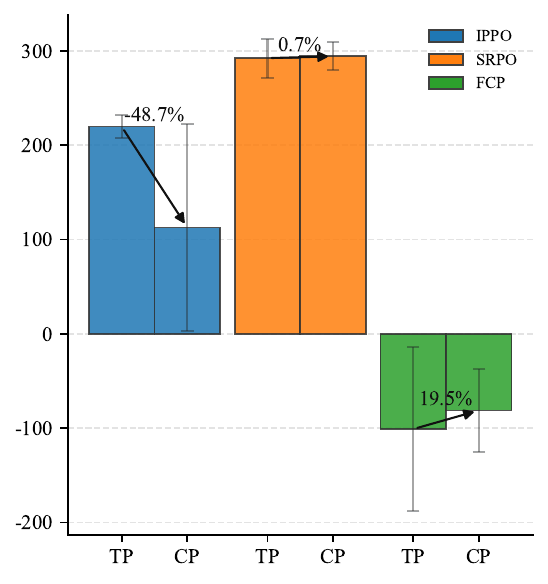}
    \caption{Difference between training performance (TP) and cross-play performance (CP), reported as mean and standard deviation in Overcooked Gridworld. IPPO performance drops sharply in cross-play, with lower average return and larger variance, while SRPO remains stable. FCP also achieves relatively low returns when paired with itself or with FCP agents from other seeds, suggesting that the self-play partner pool induces free-riding rather than robust coordination.}
    \label{fig:app-baseline-FCP}
\end{figure}

\subsection{Tag}\label{app:tag}
\paragraph{Environment.} We evaluate on the PettingZoo Multi-Agent Particle Environment (MPE) simple\_tag\_v3 \citep{terry2021pettingzoogymmultiagentreinforcement}, configured with 1 runner (prey), 2 chasers, and 2 static obstacles, using discrete actions and a horizon of 100 cycles per episode. The environment is shown in \cref{fig:tag}. Rewards are taken directly from the environment and scaled by 0.1, and an episode terminates when any agent is terminated or truncated.
\paragraph{Training and evaluation.} Each agent is trained for approximately $3\times 10^{7}$ environment interactions. Both the actor and the critic are two-layer MLPs with $128$ hidden units and ReLU activations. Training performance is recorded every $40960$ interactions, where each evaluation point reports the average episodic return over 64 rollouts of length 100, and each update performs $4$ PPO epochs over the full batch with the clipped surrogate objective ($\epsilon_{\text{clip}}=0.2$), GAE ($\gamma=0.99$ and $\lambda=0.95$), and Adam at learning rate $10^{-3}$. We perform 30 independent runs for each method. IPPO agents use an entropy coefficient of $\epsilon=0.01$, while SRPO agents use the risk aversion parameter $\tau=10$ with the same entropy coefficient. To evaluate partner generalization, we construct a $60\times 60$ cross-play matrix consisting of 30 IPPO agents and 30 SRPO agents. Each matrix entry corresponds to the average return obtained by a fixed agent pair, evaluated over 100 episodes of length 100 without further learning. In addition to cross-play on the training environment, we evaluate all agent pairs in a held-out test environment with a different runner configuration and report these results separately. The results are shown in \cref{fig:tag-cross-play}.

\subsection{Multi-Robot Warehouse}\label{app:rware}
\paragraph{Environment.} We evaluate on the Multi-Robot Warehouse environment \citep{papoudakis2021benchmarking}, configured as \texttt{rware-tiny-4ag-easy-v2} with 4 agents,
discrete actions, and 71-dimensional partial observations. The environment is shown in \cref{fig:rware}. Rewards are taken
directly from the environment, augmented with a small pickup-shaping bonus of $0.01$ awarded the first step an agent grasps a
requested shelf, and an episode terminates only at the environment-imposed horizon.
\paragraph{Training and evaluation.} Each agent is trained for $2\times 10^{7}$ environment interactions. Both the actor and
the critic are two-layer MLPs with $64$ hidden units and $\tanh$ activations. Rollouts use $16$ parallel environments of length
$500$, and each update performs $4$ PPO epochs over minibatches of size $2048$ with the clipped surrogate objective
($\epsilon_{\text{clip}}=0.2$), GAE ($\gamma=0.99$ and $\lambda=0.95$), and Adam at learning
rate $5\times 10^{-4}$. Both IPPO and SRPO use parameter sharing across agents and an entropy coefficient of $\epsilon=0.05$. We perform 10
independent runs for each method. To evaluate partner generalization, we construct a $20\times 20$ cross-play matrix between
the 10 IPPO and 10 SRPO agents, where each entry reports the average return obtained by a fixed agent pair evaluated over
rollouts without further learning. The results are shown in \cref{fig:rware-appendix}.
\begin{figure*}[h]
    \begin{subfigure}{0.48\textwidth}
        \centering
        \includegraphics[height=5cm]{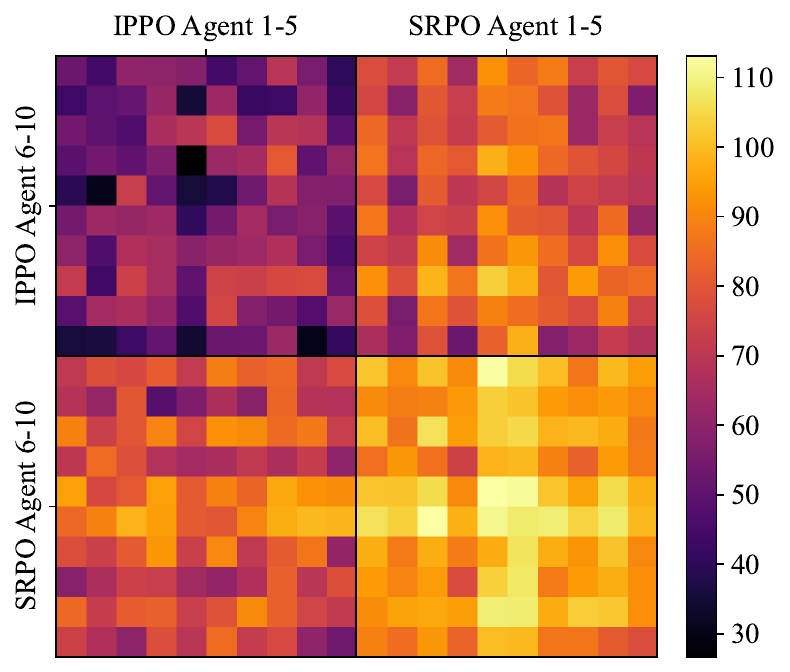}
        \caption{RWare 10 Agents.}
        \label{fig:rware-cross-play-app}
    \end{subfigure}
    \hfill
    \begin{subfigure}{0.48\textwidth}
        \centering
        \includegraphics[height=5cm]{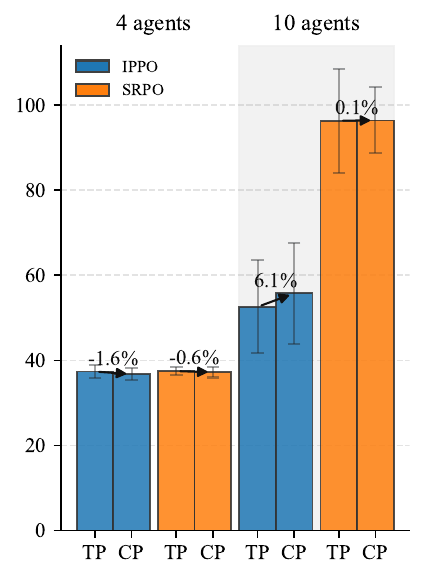}
        \caption{TP-CP gaps.}
        \label{fig:drop-warehouse}
    \end{subfigure}
    \vspace{-1ex}\caption{
    RWARE cross-play performances of SRPO ($\epsilon=0.05, \tau=20$) and IPPO ($\epsilon=0.05$) in the ten-agent setting. SRPO does not sacrifice performance in the training (4 agents TP), but achieve much better CP and TP performances in the ten-agent cases.
    }
    \label{fig:rware-appendix}
\end{figure*}

\subsection{Hanabi}\label{sec:hanabi}
\paragraph{Environments.} Hanabi is a cooperative card game where players act as distracted pyrotechnicians who must work together to launch a spectacular fireworks display by playing cards in the correct sequence, with the twist that you can see everyone's cards except your own. As such, Hanabi is a canonical benchmark for collaboration. At each turn, the active player chooses from moves such as playing a card to the shared firework piles, discarding a card to regain an information token, or hinting another player about the color or rank of cards in their hand; hints consume a finite pool of information tokens, and illegal plays consume life tokens. Following \citet{lauffer2025robust}, we consider a version with $3$ colors and $3$ ranks. We conduct our experiments in both the $2$ player and the $4$ player settings.
\paragraph{Training and evaluation.}
Each run uses a single GPU with 1 training thread and 1000 rollout threads, collecting trajectories of episode length $100$. We train for $3\times 10^7$ environment steps with PPO updates using 15 epochs and 1 minibatch per update. We set the learning rates to $7\times 10^{-4}$ (actor) and $1\times 10^{-3}$ (critic), and set the initialization gain to $0.01$. We set the entropy regularization to $0.001$ and the risk aversion parameter $\tau=0.01$. In the $2$-player setting, policies use a $2$-layer MLP with hidden size 128. In the $4$-player setting, policies use a $2$-layer MLP with hidden size 256, and each agent is randomly set to be adversary agent during each rollout for SRPO. For each setting, we construct a $20 \times 20$ cross-play matrix consisting of $10$ IPPO agents and $10$ SRPO agents, trained independently with different seeds. Each entry is averaged over $50$ evaluation episodes of length $100$. The results are shown in \cref{fig:hanabi-appendix,fig:hanabi-drop-appendix}.
\begin{figure*}[h]
    \begin{subfigure}{0.48\textwidth}
        \centering
        \includegraphics[width=0.8\linewidth]{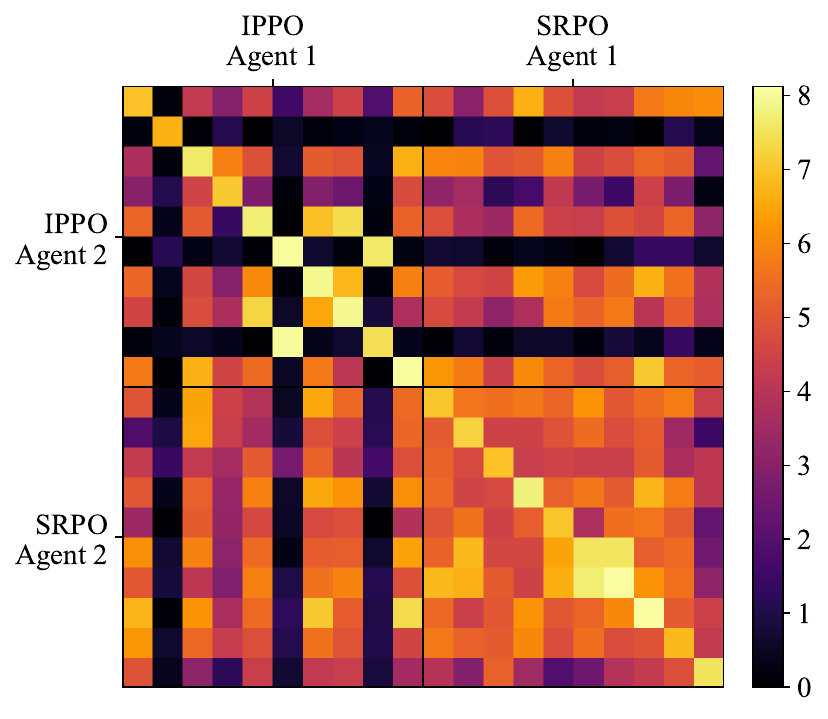}
        \caption{Hanabi $2$-player.}
        \label{fig:hanabi-2-appendix}
    \end{subfigure}
    \hfill
    \begin{subfigure}{0.48\textwidth}
        \centering
        \includegraphics[width=0.8\linewidth]{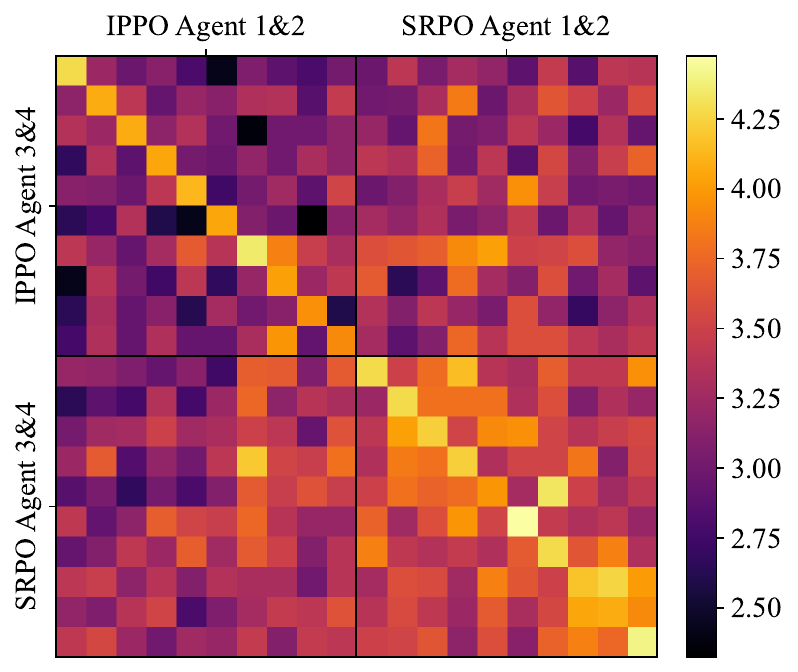}
        \caption{Hanabi $4$-player.}
        \label{fig:hanabi-4-appendix}
    \end{subfigure}
    \vspace{-1ex}\caption{
    Cross-play performances of SRPO ($\epsilon=0.001, \tau=0.01$) and IPPO ($\epsilon=0.001$) in both $2$-player and $4$-player Hanabi games. Each square represents the average reward of two agents across 50 runs of length 100.
    }
    \label{fig:hanabi-appendix}
\end{figure*}
\begin{figure*}[h]
    \begin{subfigure}{0.48\textwidth}
        \centering
        \includegraphics[height=5.5cm]{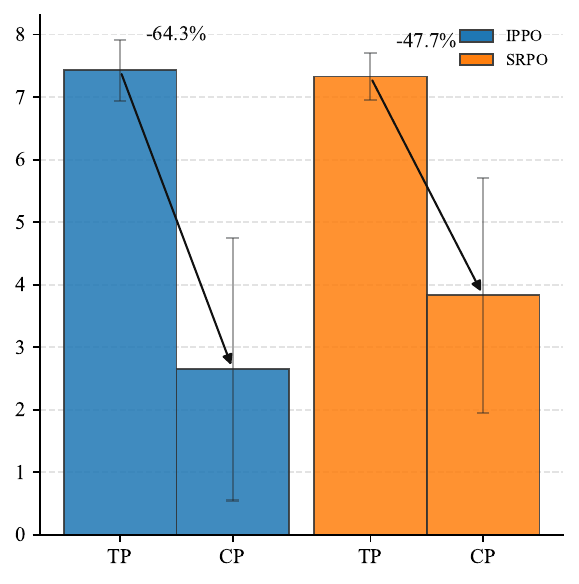}
        \caption{Hanabi $2$-player.}
        \label{fig:hanabi-2-drop-appendix}
    \end{subfigure}
    \hfill
    \begin{subfigure}{0.48\textwidth}
        \centering
        \includegraphics[height=5.5cm]{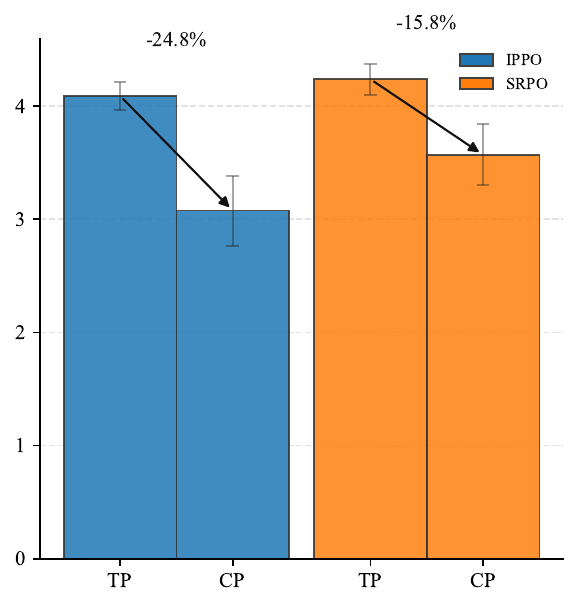}
        \caption{Hanabi $4$-player.}
        \label{fig:hanabi-4-drop-appendix}
    \end{subfigure}
    \vspace{-1ex}\caption{
    Performance change (mean and standard deviation) between Training Performance (TP) and Cross-play Performance (CP): the performance of IPPO drastically decreases, with lower average and larger standard deviation in cross-play, while the performance of SRPO drops less severely.
    }\vspace{-3ex}
    \label{fig:hanabi-drop-appendix}
\end{figure*}
\subsection{Multi-agent debate on GSM8k}\label{sec:llm-debate}
\paragraph{Environment.}In this setup, two agents engage in a three-round iterative debate protocol. In the first round, each agent independently observes the question and produces its own reasoning and answer. In subsequent rounds, each agent observes the original question as well as both agents' outputs from the previous round, and then refines its response. Successfully solving a problem therefore requires more than producing a correct answer in isolation: each agent must reinforce correct reasoning while remaining robust to potentially misleading or incorrect proposals from its teammate. This naturally induces a cooperative yet adversarial interaction, where robustness to the partner’s policy plays a critical role in achieving reliable joint performance.
\paragraph{Training and evaluation.}Both IPPO (the existing state-of-the-art) and SRPO agents are trained using multiple base language models, including Qwen2.5-0.5B-Instruct (Q0.5B) and Qwen2.5-3B-Instruct (Q3B)~\citep{qwen2025qwen25technicalreport}, as well as Qwen3-0.6B (Q0.6B) and Qwen3-4B-Instruct-2507 (Q4B)~\citep{yang2025qwen3}, using the verl training framework~\citep{sheng2024hybridflow}. Here, we set the temperature $T=1$ for training, the entropy coefficient to be $\epsilon=0$ for both IPPO and SRPO, and the risk degree $\tau=10$ for SRPO. Qwen2.5-0.5B-Instruct is trained with $200$ steps on two NVIDIA H100s, which takes around 8 hours. Qwen3-0.6B is trained with $300$ steps on 2 h100s, which takes around 12 hours. Qwen2.5-3B-Instruct is trained with $200$ steps on four NVIDIA H100s, which takes around 12 hours. Qwen3-4B-Instruct-2507 is trained with $350$ steps on four NVIDIA H100s, which takes around 24 hours. Every $10$ steps, we evaluate the performances on a held-out validation set using the same debate setup. The training curves are shown in \cref{fig:qwen2.5,fig:qwen3,fig:qwen3B,fig:qwen4B}. Due to the cost of LLM RL training, we report a single training run per setting. For evaluation, we set the temperature to 0.1 and evaluate on the validation set for 5 runs, and the results are reported in \cref{tab:gsm8k-robustness}. Compared with IPPO, SRPO primarily affects coordination rather than individual reasoning. During training, SRPO achieves slightly lower final debate accuracy than IPPO because it interacts with an adversarial partner, whereas IPPO does not. During evaluation, however, SRPO consistently achieves more robust cross-play performance across model sizes, as shown in \cref{tab:gsm8k-robustness}.
\begin{figure*}[h]
    \begin{subfigure}{0.48\textwidth}
        \centering
        \includegraphics[height=5cm]{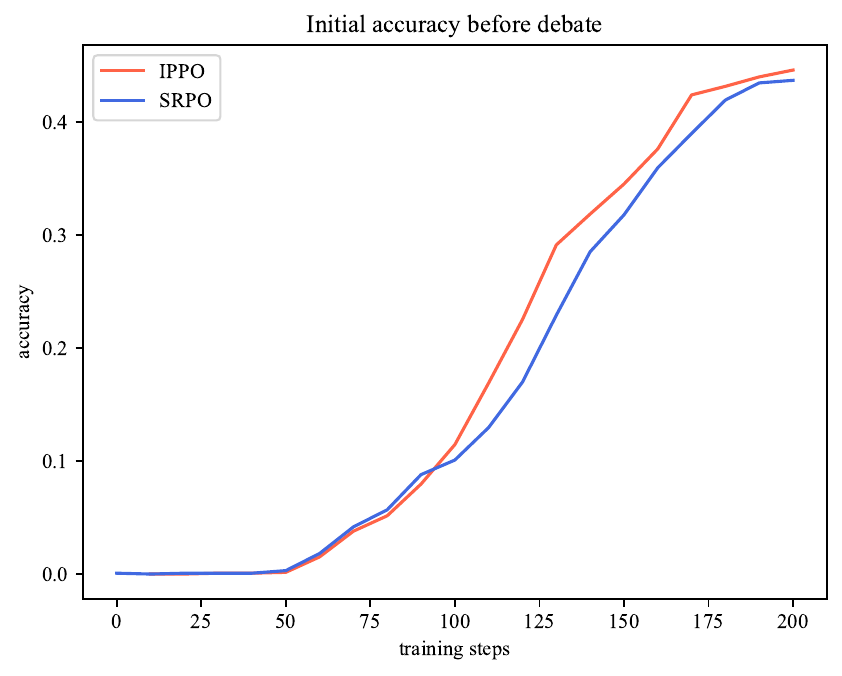}
        \caption{Initial accuracy before debate.}
        \label{fig:qwen2.5-before}
    \end{subfigure}
    \hfill
    \begin{subfigure}{0.48\textwidth}
        \centering
        \includegraphics[height=5cm]{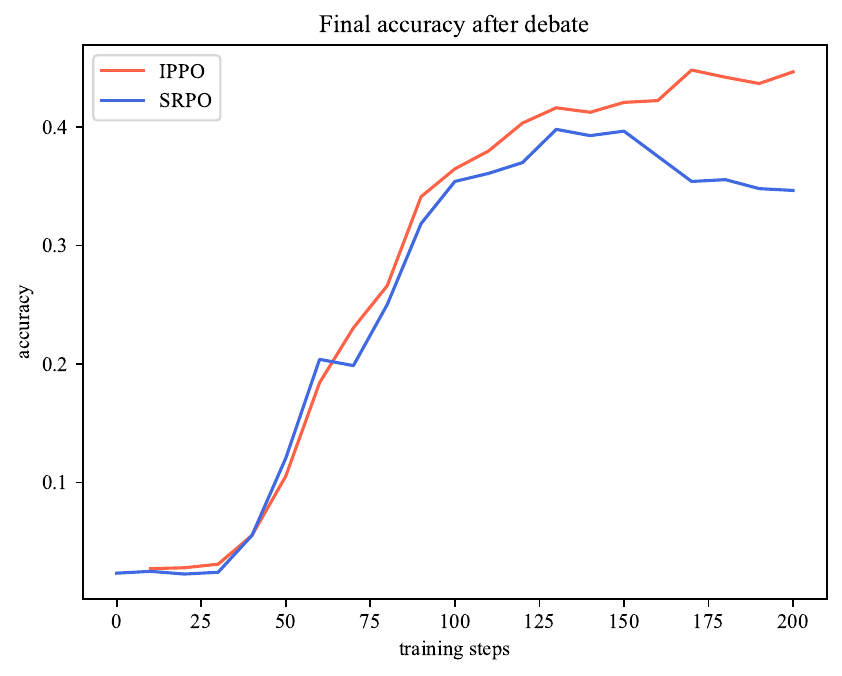}
        \caption{Final accuracy after debate.}
        \label{fig:qwen2.5-after}
    \end{subfigure}
    \vspace{-1ex}\caption{
    The training curve for SRPO and IPPO with Qwen2.5-0.5B-Instruct. We point out that SRPO is trained against an adversary---meaning lower training reward (post debate) is expected as the adversary learns to mislead the agent to minimize reward.
    }\vspace{-2ex}
    \label{fig:qwen2.5}
\end{figure*}

\begin{figure*}[h]
    \begin{subfigure}{0.48\textwidth}
        \centering
        \includegraphics[height=5cm]{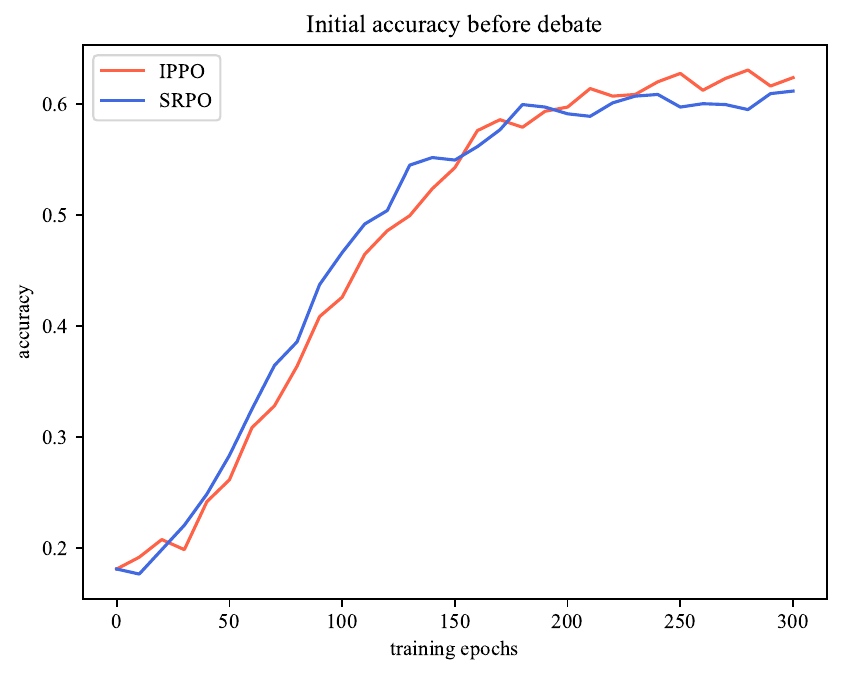}
        \caption{Initial accuracy before debate.}
        \label{fig:qwen3-before}
    \end{subfigure}
    \hfill
    \begin{subfigure}{0.48\textwidth}
        \centering
        \includegraphics[height=5cm]{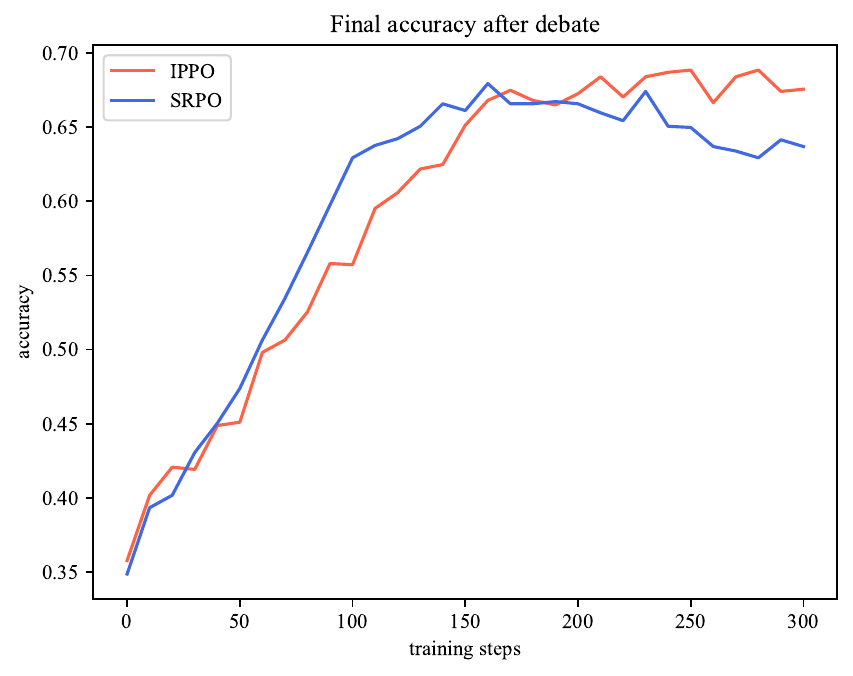}
        \caption{Final accuracy after debate.}
        \label{fig:qwen3-after}
    \end{subfigure}
    \vspace{-1ex}\caption{
    The training curve for SRPO and IPPO with Qwen3-0.6B. We point out that SRPO is trained against an adversary---meaning lower training reward (post debate) is expected as the adversary learns to mislead the agent to minimize reward.
    }\vspace{-3ex}
    \label{fig:qwen3}
\end{figure*}

\begin{figure*}[h]
    \begin{subfigure}{0.48\textwidth}
        \centering
        \includegraphics[height=5cm]{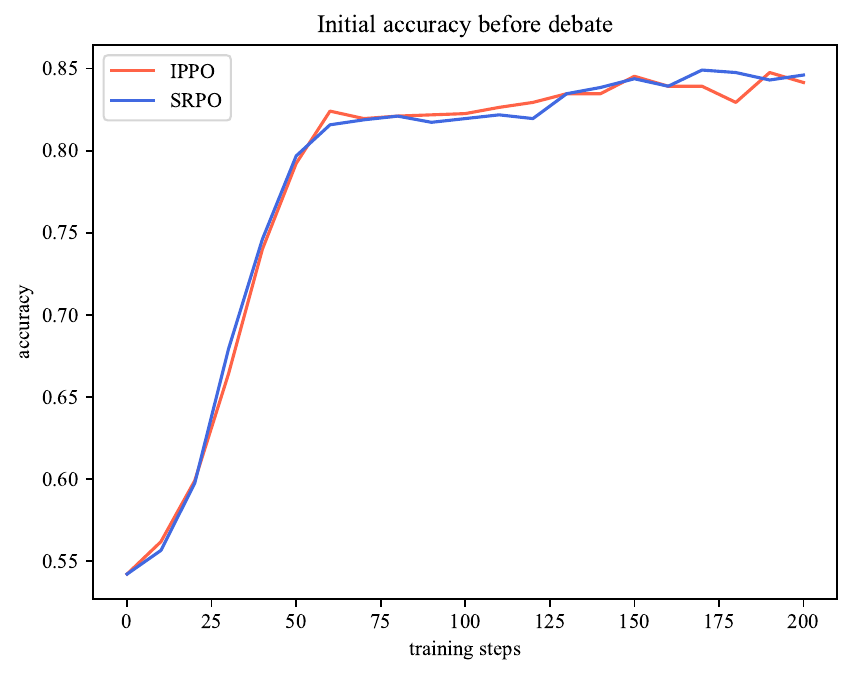}
        \caption{Initial accuracy before debate.}
        \label{fig:qwen3B-before}
    \end{subfigure}
    \hfill
    \begin{subfigure}{0.48\textwidth}
        \centering
        \includegraphics[height=5cm]{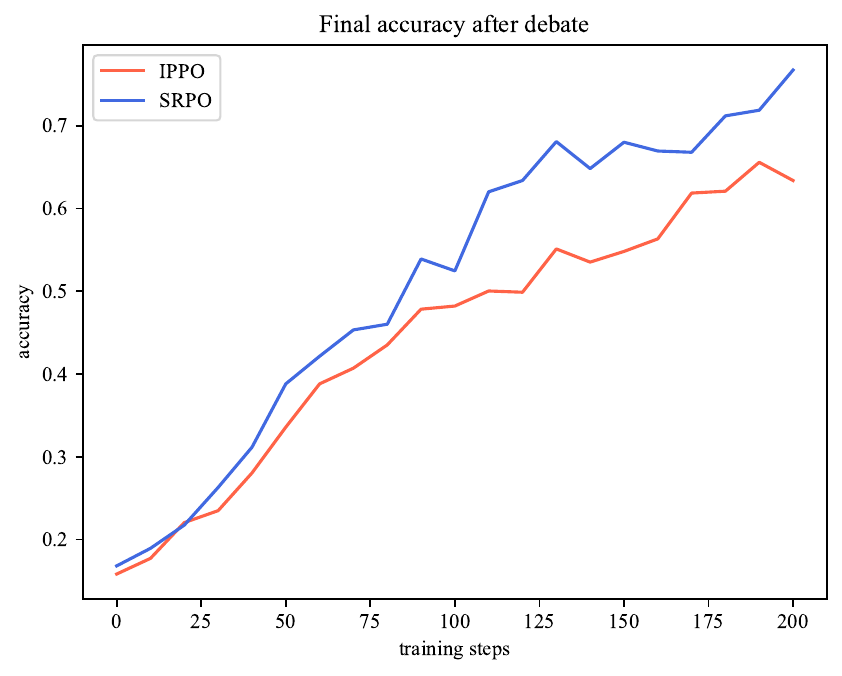}
        \caption{Final accuracy after debate.}
        \label{fig:qwen3B-after}
    \end{subfigure}
    \vspace{-1ex}\caption{
    The training curve for SRPO and IPPO with Qwen2.5-3B-Instruct. We point out that SRPO is trained against an adversary---meaning lower training reward (post debate) is expected as the adversary learns to mislead the agent to minimize reward.
    }\vspace{-3ex}
    \label{fig:qwen3B}
\end{figure*}

\begin{figure*}[h]
    \begin{subfigure}{0.48\textwidth}
        \centering
        \includegraphics[height=5cm]{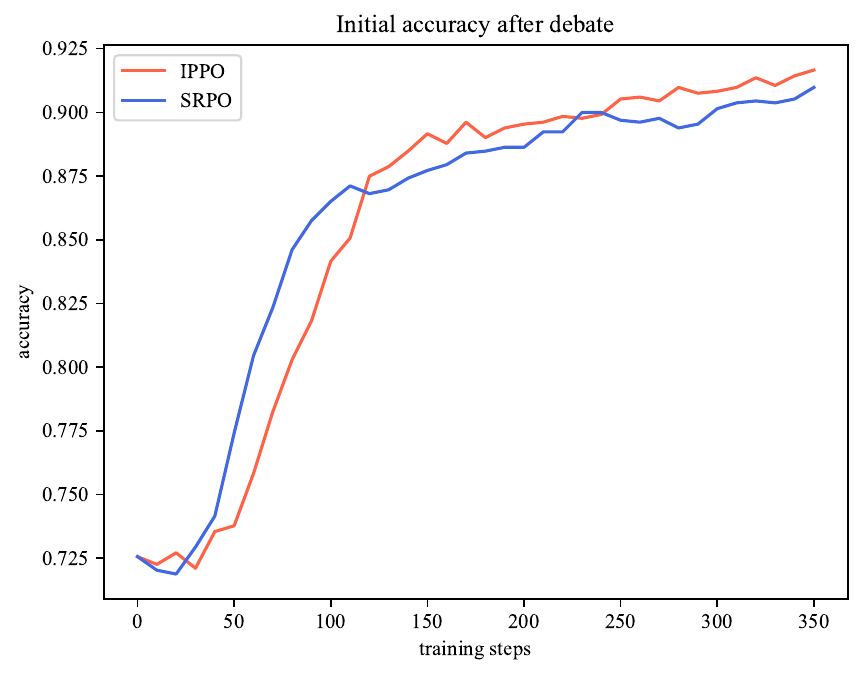}
        \caption{Initial accuracy before debate.}
        \label{fig:qwen4B-before}
    \end{subfigure}
    \hfill
    \begin{subfigure}{0.48\textwidth}
        \centering
        \includegraphics[height=5cm]{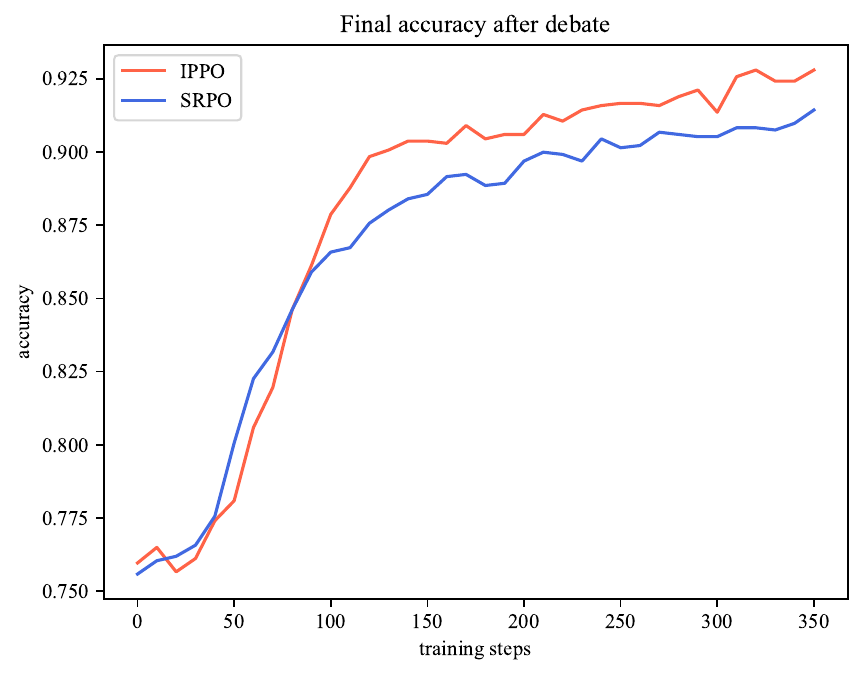}
        \caption{Final accuracy after debate.}
        \label{fig:qwen4B-after}
    \end{subfigure}
    \vspace{-1ex}\caption{
    The training curve for SRPO and IPPO with Qwen3-4B-Instruct-2507. We point out that SRPO is trained against an adversary---meaning lower training reward (post debate) is expected as the adversary learns to mislead the agent to minimize reward.
    }\vspace{-3ex}
    \label{fig:qwen4B}
\end{figure*}

\subsection{Ablation study}\label{sec:ablation}
We conduct ablation studies on both the risk aversion parameter $\tau$ and the entropy coefficient $\epsilon$ in the Overcooked and Tag environments, with additional ablation studies on $\tau$ in RWARE. The results are shown in \cref{fig:entropy-ablation,fig:risk-ablation}. While prior work \citep{PPOentropy} demonstrates that tuning the entropy coefficient $\epsilon$ can improve cross-play performance of IPPO in the Hanabi environment, we find that this strategy does not generalize to Overcooked or Tag. In these environments, entropy tuning alone fails to yield robust partner generalization, suggesting that entropy regularization is not a principled mechanism for addressing coordination under partner shifts.
\begin{figure}[h]
    \centering
    \begin{subfigure}{0.48\linewidth}
        \centering
        \includegraphics[width=\linewidth]{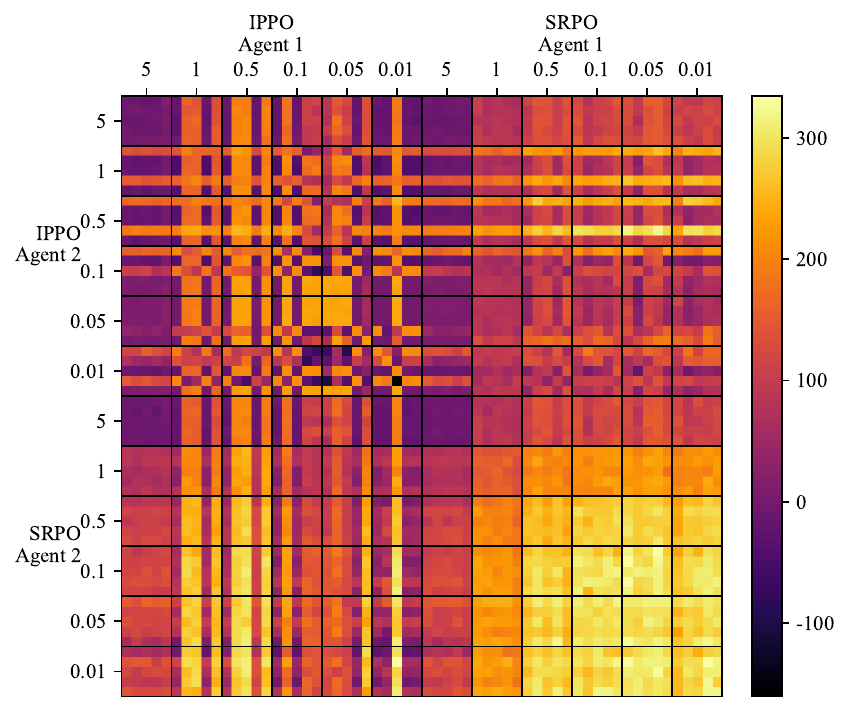}
        \caption{Overcooked.}
        \label{fig:entropy-ablation-overcooked}
    \end{subfigure}
    \hfill
    \begin{subfigure}{0.48\linewidth}
        \centering
        \includegraphics[width=\linewidth]{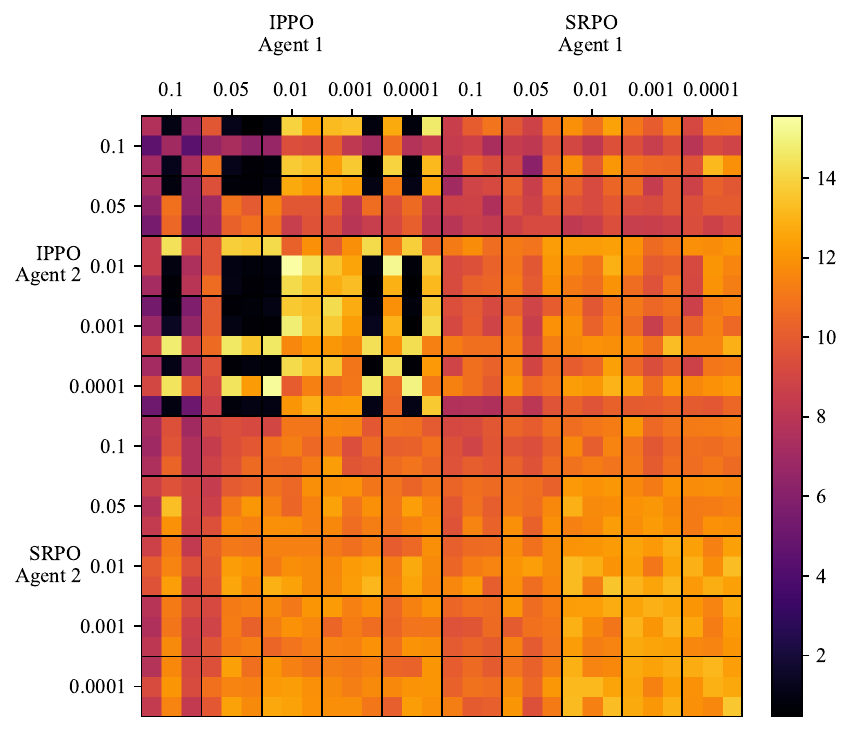}
        \caption{Tag.}
        \label{fig:entropy-ablation-tag-train}
    \end{subfigure}
    \caption{
    Ablation study on the entropy parameter $\epsilon$.
    }
    \label{fig:entropy-ablation}
\end{figure}

\begin{figure}[h]
    \centering
    \begin{subfigure}{0.31\linewidth}
        \centering
        \includegraphics[height=3.75cm]{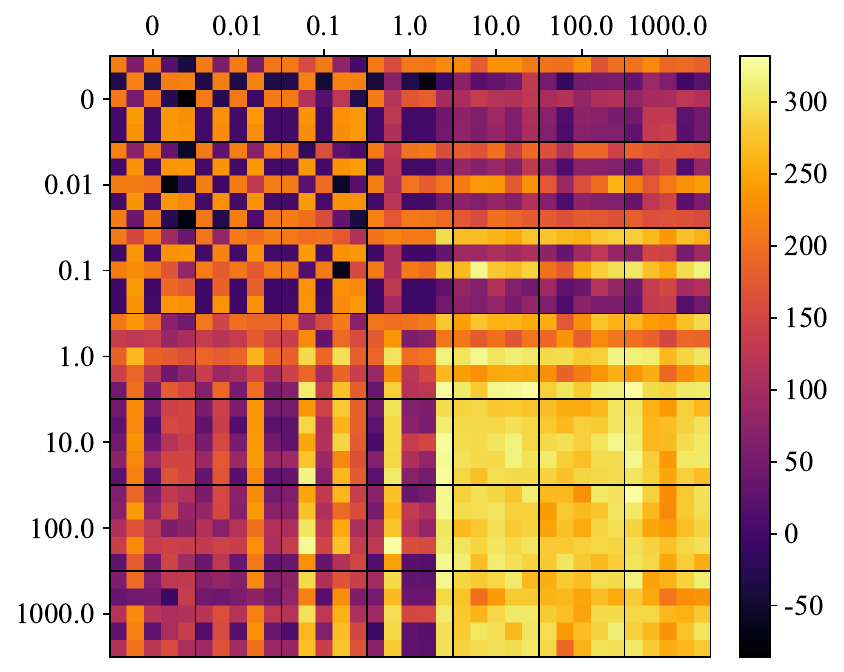}
        \caption{Overcooked.}
        \label{fig:risk-ablation-overcooked}
    \end{subfigure}
    \hfill
    \begin{subfigure}{0.31\linewidth}
        \centering
        \includegraphics[height=3.75cm]{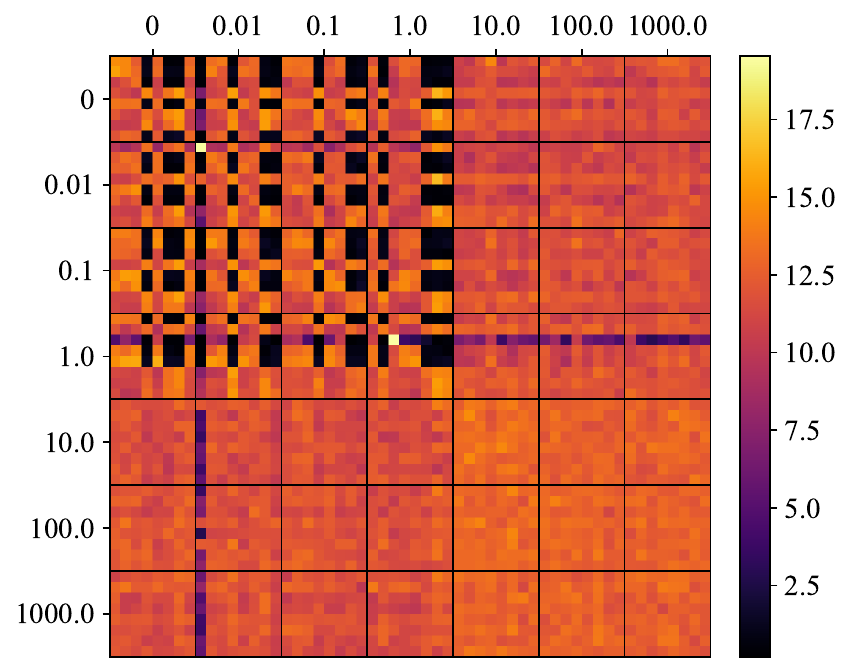}
        \caption{Tag.}
        \label{fig:risk-ablation-tag-train}
    \end{subfigure}
    \hfill
    \begin{subfigure}{0.31\linewidth}
        \centering
        \includegraphics[height=3.75cm]{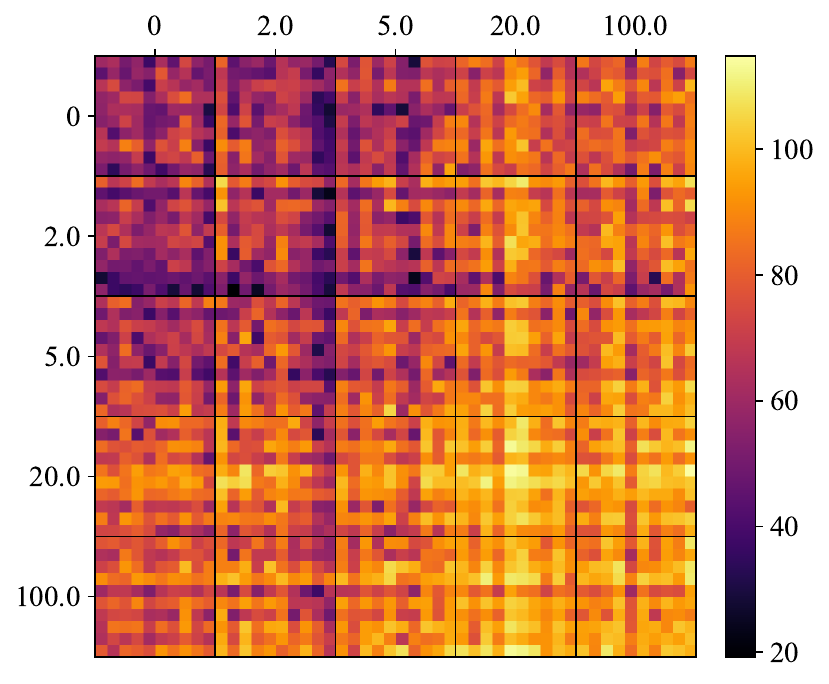}
        \caption{RWARE (10 agents).}
        \label{fig:risk-ablation-rware-train}
    \end{subfigure}
    \caption{
    Ablation study on the degree of risk aversion parameter $\tau$; $\tau=0$ refers to the risk-neutral case, i.e., IPPO. SRPO is robust to choices of $\tau$ and achieves consistent gains over IPPO across orders of magnitude of $\tau$.
    }
    \label{fig:risk-ablation}
\end{figure}

\section{Broader Impact}
This work advances collaborative MARL by training agents that generalize to unseen partners and contribute proportionally to shared objectives, with potential benefits for human--AI teaming and multi-agent foundation model systems. We do not release any new high-risk artifacts and do not foresee direct negative societal impacts beyond those generic to advances in MARL.
\section{Limitations}\label{app:limitations}
SRPO maintains an additional adversary policy $\phi_i$ per learning agent, which roughly doubles the per-iteration cost relative to IPPO in the fully decentralized form. However, this overhead can be eliminated when agents are exchangeable by sharing a single agent and a single adversary policy across players (as in our RWARE and large-population runs). The GSM8K multi-agent debate experiment (\cref{sec:llm-debate}) uses a single training seed per (model, algorithm) configuration due to the cost of fine-tuning Qwen-models, so its reported confidence intervals reflect only evaluation-rollout variability; we present it as a preliminary result complementing the multi-seed evidence ($10$--$30$ seeds) from Overcooked, Tag, RWARE, and Hanabi. Finally, our equilibrium-improvement and free-riding theorems apply respectively and provably to continuous quadratic aggregative games and to finite-action collaborative games with private costs; outside this scope (e.g., Tag) SRPO still improves cross-play robustness, but we do not claim improvements in training reward.

\end{document}